\documentclass{article}

\usepackage{iclr2023_conference,times}
\iclrfinalcopy

\usepackage[utf8]{inputenc} 
\usepackage[T1]{fontenc}    
\usepackage{hyperref}       
\usepackage{url}            
\usepackage{booktabs}       
\usepackage{amsfonts}       
\usepackage{nicefrac}       
\usepackage{microtype}      
\usepackage{xcolor}

\usepackage{amsmath,amsfonts,bm}

\def\eqref#1{equation~\ref{#1}}

\def\1{\bm{1}}

\def\vx{{\bm{x}}}
\def\vy{{\bm{y}}}
\def\vz{{\bm{z}}}

\DeclareMathAlphabet{\mathsfit}{\encodingdefault}{\sfdefault}{m}{sl}
\SetMathAlphabet{\mathsfit}{bold}{\encodingdefault}{\sfdefault}{bx}{n}

\newcommand{\R}{\mathbb{R}}

\newcommand{\sigmoid}{\sigma}

\usepackage{nameref}
\usepackage{ifthen}
\usepackage{braket}
\usepackage{amssymb,amsthm}
\usepackage{thm-restate}
\usepackage{mathtools}
\usepackage{graphicx}
\usepackage{scalerel}
\usepackage{stmaryrd}
\usepackage{centernot}
\usepackage{subcaption}
\usepackage{mathabx}
\usepackage{multibib}
\usepackage{enumerate}
\usepackage[ruled, vlined]{algorithm2e}
\usepackage{setspace}
\usepackage{dsfont}
\usepackage{tikz}{\tiny }
\usepackage{wrapfig}
\usepackage[stdsubgroups,nocfg]{nomencl}
\usetikzlibrary{decorations,arrows,shapes,automata,calc}
\usepackage{enumitem}

\newcites{AR}{Additional References}

\newif\ifarxiv
\arxivtrue

\newtheorem{theorem}{Theorem}[section]
\newtheorem{corollary}{Corollary}[theorem]
\newtheorem{lemma}[theorem]{Lemma}
\newtheorem{assumption}[theorem]{Assumption}
\theoremstyle{remark}
\newtheorem{remark}{Remark}
\newtheorem{example}{Example}

\newcommand{\FD}[1]{\todo[inline]{\textcolor{white}{\textbf{Florent: }#1}}}
\definecolor{OliveGreen}{HTML}{3C8031}
\definecolor{BrickRed}{HTML}{B6321C}

\newcommand{\smallparagraph}[1]{\smallskip\noindent\textbf{#1}}
\newcommand{\tabitem}{~~\llap{\textbullet}~~}
\makeatletter
\newcommand{\pushright}[1]{\ifmeasuring@#1\else\omit\hfill$\displaystyle#1$\fi\ignorespaces}
\newcommand{\pushleft}[1]{\ifmeasuring@#1\else\omit$\displaystyle#1$\hfill\fi\ignorespaces}
\makeatother

\newcommand{\tuple}[1]{\ensuremath{\left\langle #1 \right\rangle}}
\newcommand{\indexof}[2]{\ensuremath{#1_{\scriptscriptstyle #2}}}
\newcommand{\N}{\mathbb{N}}
\newcommand{\fun}[1]{\ensuremath{\mathopen{}\mathclose\bgroup\left(#1\aftergroup\egroup\right)}}
\newcommand{\condition}[1]{\ensuremath{\1_{#1}}}
\newcommand{\vect}[1]{\ensuremath{\bm{#1}}}

\newcommand{\images}[1]{\ensuremath{\mathrm{Im}\fun{#1}}}

\newcommand{\mdp}{\ensuremath{\mathcal{M}}}
\newcommand{\states}{\ensuremath{\mathcal{S}}}
\newcommand{\actions}{\ensuremath{\mathcal{A}}}

\newcommand{\probtransitions}{\ensuremath{\mathbf{P}}} 
\newcommand{\rewards}{\ensuremath{\mathcal{R}}}
\newcommand{\labels}{\ensuremath{\ell}}
\newcommand{\atomicprops}{\ensuremath{\mathbf{AP}}}
\newcommand{\sinit}{\ensuremath{s_{\mathit{I}}}}
\newcommand{\mdptuple}{\langle \states, \actions, \probtransitions, \rewards, \labels, \atomicprops, \sinit \rangle}

\newcommand{\state}{\ensuremath{s}}

\newcommand{\action}{\ensuremath{a}}

\newcommand{\reward}{\ensuremath{r}}
\newcommand{\labeling}{\ensuremath{l}}
\newcommand{\labelset}[1]{\ensuremath{\mathsf{#1}}}

\newcommand{\act}[1]{\ensuremath{\mathit{Act}\ifthenelse{\equal{#1}{}}{}{(#1)}}}

\newcommand{\inftrajectories}[1]{\ensuremath{\mathit{Traj}}}
\newcommand{\seq}[2]{\ensuremath{#1_{\scriptscriptstyle 0:#2}}}
\newcommand{\trajectory}{\tau}
\newcommand{\trajectorytuple}[3]{\ensuremath{\tuple{#1_{\scriptscriptstyle 0:#3}, #2_{\scriptscriptstyle 0: #3-1}}}}
\newcommand{\trace}{\ensuremath{\hat{\trajectory}}}

\newcommand{\traces}[1]{\ensuremath{\mathit{Traces}_{#1}}}
\newcommand{\policy}{\ensuremath{\pi}}

\newcommand{\mpolicies}[1]{\ensuremath{\Pi}}

\newcommand{\valuessymbol}[2]{\ensuremath{V_{#1}^{#2}}}
\newcommand{\values}[3]{\ensuremath{\valuessymbol{#1}{#2}\fun{#3}}}

\newcommand{\episodereturn}[1]{\ensuremath{\mathbf{R}_{#1}}}

\newcommand{\stationary}[1]{\ensuremath{\xi_{#1}}}
\newcommand{\bisimulation}{\ensuremath{\mathcal{B}}}

\newcommand{\encodersymbol}{\ensuremath{Q}}
\newcommand{\encoder}{\ensuremath{\encodersymbol_\encoderparameter}}

\newcommand{\decodersymbol}{\ensuremath{P}}
\newcommand{\decoder}{\ensuremath{\decodersymbol_\decoderparameter}}

\newcommand{\encoderparameter}{\ensuremath{\iota}}
\newcommand{\decoderparameter}{\ensuremath{\theta}}
\newcommand{\generative}{\ensuremath{\mathcal{G}}}
\newcommand{\wassersteinparameter}{\ensuremath{\omega}}

\newcommand{\discount}{\ensuremath{\gamma}}

\newcommand{\eventually}{\ensuremath{\lozenge}}
\newcommand{\ltlnext}{\ensuremath{\bigcirc}}

\newcommand{\until}[2]{\ensuremath{#1 \, \mathcal{U} \, #2}}

\newcommand{\Prob}{\ensuremath{\displaystyle \mathbb{P}}}
\newcommand{\measurableset}{\ensuremath{\mathcal{X}}}
\newcommand{\varmeasurableset}{\ensuremath{\mathcal{Y}}}

\newcommand{\sampledot}{\ensuremath{{\cdotp}}}
\newcommand{\expectedsymbol}[1]{\ensuremath{\mathop{\mathbb{E}}\ifthenelse{\equal{#1}{}}{}{_{#1}}}}
\newcommand{\expected}[2]{\ensuremath{\expectedsymbol{#1} \left[ #2 \right]}}

\newcommand{\divergencesymbol}{\ensuremath{D}}

\newcommand{\distributions}[1]{\ensuremath{\Delta\fun{#1}}}

\newcommand{\logistic}[2]{\ensuremath{\mathrm{Logistic}(#1, #2)}}

\newcommand{\temperature}{\ensuremath{\lambda}}

\newcommand{\wassersteinsymbol}[1]{\ensuremath{W}_{#1}}
\newcommand{\wassersteindist}[3]{\ensuremath{\wassersteinsymbol{#1}\left( #2, #3 \right)}}
\newcommand{\distance}{\ensuremath{d}}
\newcommand{\logic}{\ensuremath{\mathcal{L}}}
\newcommand{\tightoverset}[2]{%
  \tikz[baseline=(X.base),inner sep=0pt,outer sep=0pt]{%
    \node[inner sep=0pt,outer sep=0pt] (X) {$#2$}; 
    \node[yshift=1.2pt] at (X.north) {$#1$};
}}
\newcommand{\bidistance}{\ensuremath{\tightoverset{\,\scriptstyle\sim}{\distance}}}

\newcommand{\functionalexpr}{\ensuremath{\mathcal{F}}}
\newcommand{\couplings}[2]{\ensuremath{\Lambda(#1, #2)}}
\newcommand{\coupling}{\ensuremath{\lambda}}

\newcommand{\Lipschf}[1]{\ensuremath{\mathcal{F}_{#1}}}
\newcommand{\tracedistance}{\transitiondistance}
\newcommand{\transitiondistance}{\ensuremath{\ensuremath{\vec{\distance}}}}
\newcommand{\waemdp}{W$^2\!$AE-MDP$\,$}
\newcommand{\waemdps}{W$^2\!$AE-MDPs$\,$}
\newcommand{\ncritic}{\ensuremath{m}}

\newcommand{\overbar}[1]{\mkern 1.5mu\overline{\mkern-1.5mu#1\mkern-1.5mu}\mkern 1.5mu}
\newcommand{\overbarit}[1]{\,\overline{\!{#1}}}
\newcommand{\embed}{\ensuremath{\phi}}
\newcommand{\embeda}{\ensuremath{\psi}}
\newcommand{\actionencoder}{\ensuremath{\embed_{\encoderparameter}^{\scriptscriptstyle\actions}}}
\newcommand{\transitionencoder}{\ensuremath{q}}
\newcommand{\latentmdp}{\ensuremath{\overbarit{\mdp}}}
\newcommand{\latentprobtransitions}{\ensuremath{\overbar{\probtransitions}}}
\newcommand{\latentstates}{\ensuremath{\overbarit{\mathcal{\states}}}}
\newcommand{\latentrewards}{\ensuremath{\overbarit{\rewards}}}
\newcommand{\latentlabels}{\ensuremath{\overbarit{\labels}}}
\newcommand{\latentmdptuple}{\ensuremath{\tuple{\latentstates, \latentactions, \latentprobtransitions, \latentrewards, \latentlabels, \atomicprops, \zinit}}}
\newcommand{\latentstate}{\ensuremath{\overbarit{\state}}}
\newcommand{\zinit}{\ensuremath{\latentstate_I}}
\newcommand{\latentactions}{\ensuremath{\overbarit{\actions}}}
\newcommand{\latentaction}{\ensuremath{\overbarit{\action}}}
\newcommand{\latentvaluessymbol}[2]{\overbarit{\ensuremath{V}}_{#1}^{#2}}
\newcommand{\latentvalues}[3]{\ensuremath{\latentvaluessymbol{#1}{#2}\fun{#3}}}

\newcommand{\latentpolicy}{\ensuremath{\overbar{\policy}}}
\newcommand{\latentpolicies}{\ensuremath{\overbar{\Pi}}}
\newcommand{\latentmpolicies}{\ensuremath{\overbar{\Pi}}}
\newcommand{\localtransitionloss}[1]{L_{\probtransitions}^{#1}}
\newcommand{\localrewardloss}[1]{L_{\rewards}^{#1}}

\newcommand{\KR}[1]{\ensuremath{\ifthenelse{\equal{#1}{}}{K_{\latentrewards}}{K_{\latentrewards}^{#1}}}}
\newcommand{\KP}[1]{\ensuremath{\ifthenelse{\equal{#1}{}}{K_{\latentprobtransitions}}{K_{\latentprobtransitions}^{#1}}}}
\newcommand{\Rmax}[1]{\ensuremath{\ifthenelse{\equal{#1}{}}{|\latentrewards^\star|}{|\latentrewards_{#1}^{\star}|}}}
\newcommand{\stationarydecoder}{{\stationary{\decoderparameter}}}
\newcommand{\latentstationaryprior}{{\bar{\xi}_{\latentpolicy_{\decoderparameter}}}}
\newcommand{\norm}[1]{\ensuremath{\left\| #1 \right\|}}
\newcommand{\gradient}{\ensuremath{\nabla}}

\newcommand{\latentvariable}{\ensuremath{z}}
\newcommand{\originaltolatentstationary}[1]{\ensuremath{\mathcal{T}}}

\newcommand{\steadystateregularizer}[1]{\ensuremath{\mathcal{W}_{\stationary{#1}}}}
\newcommand{\steadystatenetwork}{\ensuremath{\varphi_{\wassersteinparameter}^{\stationary{}}}}
\newcommand{\transitionlossnetwork}{\ensuremath{\varphi_{\wassersteinparameter}^{\probtransitions}}}

\renewcommand{\nomgroup}[1]{%
\ifthenelse{\equal{#1}{M}}{%
      \item[\textbf{Markov Decision Processes}]}{%
        \ifthenelse{\equal{#1}{K}}{%
        \item[\textbf{P}]}{%
          \ifthenelse{\equal{#1}{P}}{%
          \item[\textbf{Probability / Measure Theory}]}{%
            \ifthenelse{\equal{#1}{L}}{%
            \item[\textbf{Latent Space Model}]}{%
              {%
              \ifthenelse{\equal{#1}{W}}{%
            \item[\textbf{Wasserstein Auto-encoded MDP}]}{%
              }}}}}}}
    \makenomenclature

\ifarxiv
\title{Wasserstein Auto-encoded MDPs\\\large Formal Verification of Efficiently Distilled RL Policies with Many-sided Guarantees}
\else
\title{Wasserstein Auto-encoded MDPs\\\small Formal Verification of Efficiently Distilled RL Policies with Many-sided Guarantees}
\fi

\author{%
    Florent Delgrange \\
    AI Lab, Vrije Universiteit Brussel (VUB) \\ University of Antwerp\\
    \texttt{florent.delgrange@ai.vub.ac.be}
    \And
    Ann Now\'e \\
    AI Lab, VUB
    \And
    Guillermo A. P\'erez \\
    University of Antwerp \\ Flanders Make
}

\begin{document}
\maketitle
\begin{abstract}
Although deep reinforcement learning (DRL) has many success stories, the large-scale deployment of policies learned through these advanced techniques in safety-critical scenarios is hindered by their lack of formal guarantees.
Variational Markov Decision Processes (VAE-MDPs) are discrete latent space models that provide a reliable framework for distilling formally verifiable controllers from any RL policy.
While the related guarantees address relevant practical aspects such as the satisfaction of performance and safety properties, the VAE approach suffers from several learning flaws (posterior collapse, slow learning speed, poor dynamics estimates), primarily due to the absence of abstraction and representation guarantees to support latent optimization.
We introduce the Wasserstein auto-encoded MDP (WAE-MDP), a latent space model that fixes those issues by minimizing a penalized form of the optimal transport between the behaviors of the agent executing the original policy and the distilled policy, for which the formal guarantees apply.
Our approach yields bisimulation guarantees while learning the distilled policy, allowing concrete optimization of the abstraction and representation model quality.
Our experiments show that, besides distilling policies up to 10 times faster, the latent model quality is indeed better in general.
Moreover, we present experiments from a simple time-to-failure verification algorithm on the latent space. The fact that our approach enables such simple verification techniques highlights its applicability.
\end{abstract}
\section{Introduction}
\emph{Reinforcement learning} (RL) is emerging as a solution of choice to address challenging real-word scenarios such as epidemic mitigation and prevention strategies 
\citep{DBLP:conf/pkdd/LibinMVPHLN20},
multi-energy management \citep{CEUSTERS2021117634}, or effective canal control  \citep{REN2021103049}.
RL enables learning high performance controllers by introducing general nonlinear function approximators (such as neural networks) to scale with high-dimensional and continuous state-action spaces.
This introduction, termed \emph{deep-RL}, causes the loss of the conventional convergence guarantees of RL \citep{DBLP:journals/ml/Tsitsiklis94} as well as those obtained in some continuous settings \citep{phdthesis:Nowe94}, and hinders their wide roll-out in critical settings.
This work \emph{enables} the \emph{formal verification} of \emph{any} such policies, learned by agents interacting with {unknown}, {continuous} environments modeled as \emph{Markov decision processes} (MDPs). 
Specifically, we learn a \emph{discrete} representation of the state-action space of the MDP, which yield both a (smaller, explicit) \emph{latent space model} and a distilled version of the RL policy, that are tractable for \emph{model checking} \citep{DBLP:BK08}.
The latter are supported by \emph{bisimulation guarantees}: intuitively, the agent behaves similarly in the original and latent models.
The strength of our approach is not simply that we verify that the RL agent meets a \emph{predefined} set of specifications, but rather provide an abstract model on which the user can reason and check \emph{any} desired agent property.

\emph{Variational MDPs} (VAE-MDPs, \citealt{DBLP:journals/corr/abs-2112-09655}) offer a valuable framework for doing so.
The distillation is provided with PAC-verifiable {bisimulation} bounds guaranteeing
that the agent behaves similarly
(i) in the original and latent model (\emph{abstraction quality});
(ii) from all original states embedded to the same discrete state (\emph{representation quality}).
Whilst the bounds offer a confidence metric that enables the verification
of performance and safety properties,
VAE-MDPs suffer from several learning flaws.
First, training a VAE-MDP relies on variational proxies to the bisimulation bounds, meaning there is no learning guarantee on the quality of the latent model via its optimization.
Second, \emph{variational autoencoders} (VAEs) \citep{DBLP:journals/corr/KingmaW13, DBLP:journals/jmlr/HoffmanBWP13}
are known to suffer from \emph{posterior collapse} (e.g., \citealt{DBLP:conf/icml/AlemiPFDS018})
resulting in a deterministic mapping to a unique latent state in VAE-MDPs.
Most of the training process focuses on handling this phenomenon and setting up the stage for the concrete distillation and abstraction, finally taking place in a second training phase.
This requires extra regularizers, setting up annealing schemes
and learning phases,
and defining prioritized replay buffers to store transitions.
Distillation through VAE-MDPs is thus a meticulous task, requiring a large step budget and tuning many hyperparameters.

Building upon \emph{Wasserstein} autoencoders 
\citep{DBLP:conf/iclr/TolstikhinBGS18} instead of VAEs, we introduce \emph{Wasserstein auto-encoded MDPs} (WAE-MDPs), which overcome those limitations.
Our WAE relies on the \emph{optimal transport} (OT) from trace distributions resulting from the execution of the RL policy in the real environment to that reconstructed from the latent model operating under the distilled policy.
%
In contrast to VAEs which rely on variational proxies, we derive a novel objective that directly incorporate the bisimulation bounds.
Furthermore, while VAEs learn stochastic mappings to the latent space which need be determinized or even entirely reconstructed from data at the deployment time to obtain the guarantees, our WAE has no such requirements, and learn \emph{all the necessary components to obtain the guarantees during learning}, and does not require such post-processing operations.

Those theoretical claims are reflected in our experiments: policies are distilled up to $10$ times faster through WAE- than VAE-MDPs and provide better abstraction quality and performance in general, without the need for setting up annealing schemes and training phases, nor prioritized buffer and extra regularizer.
Our distilled policies are able to recover (and sometimes even outperform) the original policy performance,
highlighting the representation quality offered by our new framework: the distillation is able to remove some non-robustness of the input RL policy.
Finally, we formally verified \emph{time-to-failure} properties (e.g., \citealt{DBLP:conf/focs/Pnueli77}) to emphasize the applicability of our approach.

\smallparagraph{Other Related Work.}~%
Complementary works approach safe RL via formal methods \citep{DBLP:conf/tacas/Junges0DTK16,DBLP:conf/aaai/AlshiekhBEKNT18,jansen_et_al:LIPIcs:2020:12815,DBLP:conf/atal/SimaoJS21}, aimed at formally ensuring safety \emph{during RL}, all of which require providing an abstract model of the safety aspects of the environment.
They also include the work of \citet{DBLP:conf/fmcad/AlamdariAHL20}, applying synthesis and model checking on
policies distilled from RL,
without quality guarantees.
Other frameworks share our goal of verifying deep-RL policies \citep{DBLP:conf/formats/Bacci020, DBLP:conf/ijcai/CarrJT20} but rely on a known environment model, among other assumptions (e.g., deterministic or discrete environment).
Finally, \emph{DeepSynth} \citep{DBLP:conf/aaai/HasanbeigJAMK21} allows learning a formal model from execution traces, with the different purpose of guiding the agent towards sparse and non-Markovian rewards.

%
On the latent space training side, WWAEs \citep{DBLP:journals/corr/abs-1902-09323/zhang19} reuse OT as latent regularizer discrepancy (in Gaussian closed form), whereas we derive two regularizers involving OT.
These two are, in contrast, optimized via the dual formulation of Wasserstein, as in \emph{Wassertein-GANs} \citep{DBLP:conf/icml/ArjovskyCB17}.
Similarly to \emph{VQ-VAEs} \citep{DBLP:conf/nips/OordVK17} and \emph{Latent Bernoulli AEs} \citep{DBLP:conf/icml/FajtlAMR20}, our latent space model learns discrete spaces via deterministic encoders, but relies on a smooth approximation instead of using the straight-through gradient estimator.

Works on \emph{representation learning} for RL \citep{DBLP:conf/icml/GeladaKBNB19,DBLP:conf/nips/CastroKPR21,DBLP:conf/iclr/0001MCGL21,DBLP:journals/corr/abs-2112-15303} consider bisimulation metrics to optimize the representation quality,
and aim at learning (continuous) representations which capture bisimulation, so that two states close in the representation are guaranteed to provide close and relevant information to optimize the performance of the controller.
In particular, as in our work, \emph{DeepMDPs} \citep{DBLP:conf/icml/GeladaKBNB19} are learned by optimizing \emph{local losses}, by assuming a deterministic MDP and without verifiable confidence measurement.

\section{Background}\label{sec:background}
%
In the following, we write $\distributions{\measurableset}$ for the set of measures over (complete, separable metric space) $\measurableset$.
\ifarxiv
The index of all the notations introduced along the paper is available at the end of the Appendix.
\fi

\textbf{Markov decision processes}~%
(MDPs) are tuples $\mdp = \mdptuple$ where
$\states$ is a set of \emph{states}; $\actions$, a set of \emph{actions}; $\probtransitions \colon \states \times \actions \to \distributions{\states}$, a \emph{probability transition function} that maps the current state and action to a \emph{distribution} over the next states; $\rewards \colon \states \times \actions \to \R$, a \emph{reward function}; $\labels \colon \states \to 2^\atomicprops$, a \emph{labeling function} over a set of atomic propositions $\atomicprops$; and $\sinit \in \states$, the \emph{initial state}.
If $|\actions| = 1$, $\mdp$ is a fully stochastic process called a \emph{Markov chain} (MC).
We write $\mdp_\state$ for the MDP obtained when replacing the initial state of $\mdp$ by $\state \in \states$.
An agent interacting in $\mdp$ produces \emph{trajectories}, i.e., sequences of states and actions $\trajectory = \trajectorytuple{\state}{\action}{T}$
where $\state_0 = \sinit$ and $\state_{t + 1} \sim \probtransitions\fun{\sampledot \mid \state_t, \action_t}$ for $t < T$.
The set of infinite trajectories of $\mdp$ is \inftrajectories{\mdp}.
%
We assume $\atomicprops$ and labels being respectively {one-hot} and binary encoded.
Given $\labelset{T} \subseteq \atomicprops$, we write $\state \models \labelset{T}$ if $\state$ is labeled with $\labelset{T}$, i.e., $\labels\fun{\state}\cap \labelset{T} \neq \emptyset$, and $\state \models \neg \labelset{T}$ for $\state \not\models \labelset{T}$.
We refer to MDPs with continuous state or action spaces as \emph{continuous MDPs}. 
In that case, we assume $\states$ and $\actions$ are complete separable metric spaces equipped with a Borel $\sigma$-algebra,
and $\labels^{-1}\fun{\labelset{T}}$ is Borel-measurable for any $\labelset{T} \subseteq \atomicprops$.

\smallparagraph{Policies and stationary distributions.}~A \emph{(memoryless) policy} $\policy \colon \states \to \distributions{\actions}$ 
prescribes which action to choose at each step of the interaction.
The set of memoryless policies of $\mdp$ is $\mpolicies{\mdp}$.
The MDP $\mdp$ and $\policy \in \mpolicies{\mdp}$
induce an MC $\mdp_\policy$
with unique probability measure $\Prob^\mdp_\policy$ on the Borel $\sigma$-algebra over measurable subsets $\varphi \subseteq \inftrajectories{\mdp}$~\citep{DBLP:books/wi/Puterman94}.
We drop the superscript when the context is clear.
%
Define $\stationary{\policy}^{t}\fun{\state' \mid \state} =  \Prob^{\mdp_\state}_\policy(\{{\seq{\state}{\infty}, \seq{\action}{\infty}}
\mid \state_t = \state'\})$ as the distribution giving the probability of being in each state of $\mdp_\state$ after $t$ steps.
$B\subseteq \states$ is a \emph{bottom strongly connected component} (BSCC) of $\mdp_\policy$ if (i) $B$ is a maximal subset satisfying $\stationary{\policy}^t\fun{\state'\mid \state} > 0$  for any  $\state, \state' \in B$ and some $t \geq 0$, and (ii)  
$\expectedsymbol{\action \sim \policy\fun{\sampledot \mid \state}} \probtransitions\fun{B \mid \state, \action}=1$ for all $\state \in \states$.
The unique stationary distribution of $B$ is
$\stationary{\policy} \in \distributions{B}$.
We write $\state, \action \sim \stationary{\policy}$ for sampling $\state$ from $\stationary{\policy}$ then $\action$ from $\policy$.
An MDP $\mdp$ is \emph{ergodic} if for all $\policy \in \mpolicies{\mdp}$, the state space of $\mdp_\policy$ consists of a unique aperiodic BSCC %
with $\stationary{\policy} = \lim_{t \to \infty} \stationary{\policy}^t\fun{\sampledot \mid \state}$ for all $\state \in \states$.

\smallparagraph{Value objectives.}~%
Given $\policy \in \mpolicies{\mdp}$, the \emph{value} of a state $\state \in \states$ is the expected value of a random variable obtained by running $\policy$ from $\state$.
For a discount factor $\discount \in \mathopen[0, 1\mathclose]$, we consider the following objectives.
(i) \emph{Discounted return}:~we write $\values{\policy}{}{\state} = \expectedsymbol{\policy}^{\mdp_{\state}}\left[ \sum_{t = 0}^{\infty} \discount^t \rewards\fun{\state_t, \action_t} \right]$ for the expected discounted rewards accumulated along trajectories.
	The typical goal of an RL agent is to learn a policy $\policy^{\star}$ that maximizes $\values{\policy^{\star}}{}{\sinit}$ through interactions with the (unknown) MDP; 
(ii) \emph{Reachability}:~let $\labelset{C}, \labelset{T} \subseteq \atomicprops$, the \emph{(constrained) reachability} 
	event is $\until{\labelset{C}}{\labelset{T}} =
	\{\, \seq{\state}{\infty}, \seq{\action}{\infty} \,|\, \exists i \in \N, \forall j < i, \state_j \models {\labelset{C}} \wedge \state_i \models {\labelset{T}}\, \} \subseteq \inftrajectories{\mdp}$.
	We write $\values{\policy}{\varphi}{\state} = \expectedsymbol{\policy}^{\mdp_{\state}}\left[ \discount^{t^{\star}} \condition{\tuple{\seq{\state}{\infty}, \seq{\action}{\infty}} \, \in \, \varphi} \right]$ for the \emph{discounted probability of satisfying} $\varphi = \until{\labelset{C}}{\labelset{T}}$,
	where $t^{\star}$ is the length of the shortest trajectory prefix that allows satisfying $\varphi$.
	Intuitively, 
	this denotes the discounted return of remaining in a region of the MDP where states are labeled with $\labelset{C}$, until visiting \emph{for the first time} a \emph{goal state} labeled with $\labelset{T}$, and the return is the binary reward signal capturing this event.
	\emph{Safety} w.r.t. failure states $\labelset{C}$ can be expressed as the safety-constrained reachability to a destination $\labelset{T}$ through $\until{\neg \labelset{C}}{\labelset{T}}$.
	Notice that $\values{\policy}{\varphi}{\state} = \Prob_{\policy}^{\mdp_{\state}}\fun{\varphi}$ when $\discount = 1$.

\smallparagraph{Latent MDP.}~Given the original (continuous, possibly unknown) environment model $\mdp$, a \emph{latent space model} is another (smaller, explicit) MDP $\latentmdp = \latentmdptuple$ with state-action space linked to the original one via state and action \emph{embedding functions}:  $\embed\colon \states \to \latentstates$ and $\embeda\colon \latentstates \times \latentactions \to \actions$.
We refer to $\tuple{\latentmdp, \embed, \embeda}$ as a \emph{latent space model} of $\mdp$ and $\latentmdp$ as its \emph{latent MDP}.
Our goal is to learn $\tuple{\latentmdp, \embed, \embeda}$ by optimizing an \emph{equivalence criterion} between the two models.
We assume that $\distance_{\latentstates}$ is a metric on $\latentstates$, and
write  $\latentmpolicies$ for the set of policies of $\latentmdp$
and ${\latentvaluessymbol{\latentpolicy}{}}$ for the values of running $\latentpolicy \in \latentmpolicies$ in $\latentmdp$. 
\begin{remark}[Latent flow] \label{rmk:latent-policy-execution}
The latent policy $\latentpolicy$ can be seen as a policy in $\mdp$ (cf. Fig.~\ref{subfig:latent-flow-guarantees}): states passed to $\latentpolicy$ are first embedded with $\embed$ to the latent space, then the actions produced by $\latentpolicy$ are executed via $\embeda$ in the original environment.
Let $\state \in \states$, we write 
$\latentaction \sim \latentpolicy\fun{\sampledot \mid \state}$ for 
$\latentpolicy\fun{\sampledot \mid \embed\fun{\state}}$, then
the reward and next state are respectively given by $\rewards\fun{\state, \latentaction} = \rewards\fun{\state, \embeda\fun{\embed\fun{\state}, \latentaction}}$ and $ \state' \sim \probtransitions\fun{\sampledot \mid \state, \latentaction} = \probtransitions\fun{\sampledot \mid \state, \embeda\fun{\embed\fun{\state}, \latentaction}}$.
\end{remark}

\begin{figure}
\begin{subfigure}{.375\textwidth}
    \centering
    \includegraphics[width=.95\textwidth]{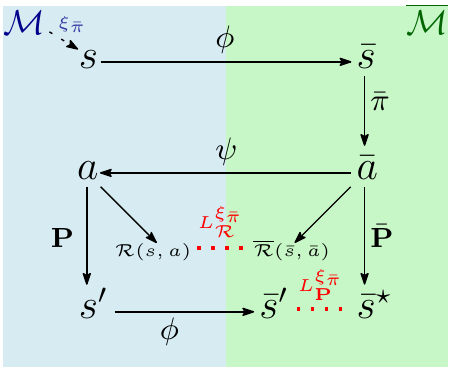}
    \caption{Execution of the latent policy $\latentpolicy$ in the original and latent MDPs, and local losses.}
    \label{subfig:latent-flow-guarantees}
\end{subfigure}
\hspace{.025\textwidth}
\begin{subfigure}{.575\textwidth}
    \centering
    \includegraphics[width=.925\textwidth]{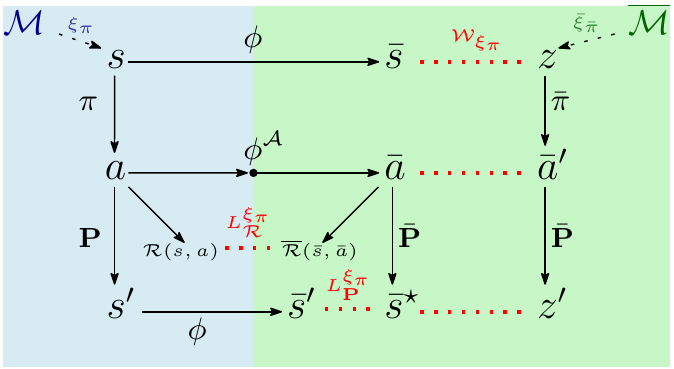}
    \caption{Parallel execution of the original RL policy $\policy$ in the original and latent MDPs, local losses, and steady-state regularizer.}
    \label{subfig:latent-fow-distillation}
\end{subfigure}
\caption{Latent flows: arrows represent (stochastic) mappings, the original (resp. latent) state-action space is spread along the blue (resp. green) area, and distances are depicted in red.
Distilling $\policy$ into $\latentpolicy$ via flow~(b) by minimizing 
$\steadystateregularizer{\policy}$ 
allows closing the gap between flows~(a) and~(b).}
\label{fig:latent-flow}
\end{figure}

\textbf{Local losses} allow quantifying the distance between the original and latent reward/transition functions \emph{in the local setting}, i.e., under a given state-action distribution $\stationary{} \in \distributions{\states \times \latentactions}$:
\begin{align*}
    \localrewardloss{\stationary{}} &= \expectedsymbol{\state,  \latentaction \sim \stationary{}}\left| \rewards\fun{\state, \latentaction} - \latentrewards\fun{\embed\fun{\state}, \latentaction} \right|, 
    &
    \localtransitionloss{\stationary{}} &= \expectedsymbol{\state, \latentaction \sim \stationary{}} 
    D\fun{\embed\probtransitions\fun{\sampledot \mid \state, \latentaction}, \latentprobtransitions\fun{\sampledot \mid \embed\fun{\state}, \latentaction}}
\end{align*}
where $\embed\probtransitions\fun{\sampledot \mid \state, \latentaction}$ is
the distribution 
of drawing
$\state' \sim \probtransitions\fun{\sampledot \mid \state, \latentaction}$ then
embedding $\latentstate' = \embed\fun{\state'}$, and $D$ is a discrepancy measure. %
Fig~\ref{subfig:latent-flow-guarantees} depicts the losses when states and actions are drawn from a stationary distribution $\stationary{\latentpolicy}$
resulting from running $\latentpolicy \in \latentmpolicies$ in $\mdp$.
In this work, we focus on the case where $D$ is the \emph{Wasserstein distance} $\wassersteinsymbol{\distance_{\latentstates}}$:
given two distributions $P, Q$ over a measurable set $\measurableset$ 
equipped with a metric $\distance$,
$\wassersteinsymbol{\distance}$ is the solution of the \emph{optimal transport} (OT) from $P$ to $Q$, i.e.,
the minimum cost of changing $P$ into $Q$ \citep{Villani2009}: 
$
\wassersteindist{\distance}{P}{Q} = \inf_{\coupling \in \couplings{P}{Q}} \expectedsymbol{x, y \sim \coupling} \distance\fun{x, y},
$
$\couplings{P}{Q}$ being the set of all \emph{couplings} of $P$ and $Q$. 
The \emph{Kantorovich duality} yields
$
    \wassersteindist{\distance}{P}{Q} = \sup_{f \in \Lipschf{\distance}} \expectedsymbol{x \sim P} f\fun{x}  - \expectedsymbol{x \sim Q} f\fun{y}
    \label{eq:wasserstein-dual}
$
where $\Lipschf{\distance}$ is the set of 1-Lipschiz functions.
Local losses are related to a well-established \emph{behavioral} equivalence between transition systems, called \emph{bisimulation}.

\smallparagraph{Bisimulation.}
A \emph{bisimulation} $\bisimulation$ 
on $\mdp$ is a behavioral equivalence between states $\state_1, \state_2 \in \states$ 
so that, $\state_1 \, \bisimulation \, \state_2$ iff
(i) $\probtransitions(T \mid \state_1, \action) = \probtransitions(T \mid \state_2, \action)$,
(ii) $\labels(\state_1) = \labels(\state_2)$,
and (iii) $\rewards(\state_1, \action) = \rewards(\state_2, \action)$
for each action $\action \in \actions$ and (Borel measurable) equivalence class $T \in \states / \bisimulation$.
Properties of bisimulation include trajectory and value equivalence \citep{DBLP:conf/popl/LarsenS89,DBLP:journals/ai/GivanDG03}.
Requirements (ii) and (iii) can be respectively relaxed depending on whether we focus only on behaviors formalized through $\atomicprops$ or rewards.
The relation can be extended to compare two MDPs (e.g., $\mdp$ and $\latentmdp$) by considering the disjoint union of their state space.
We denote the largest bisimulation relation by $\sim$.

Characterized by a logical family of functional expressions
derived from a logic $\logic$,
\emph{bisimulation pseudometrics}~\citep{DBLP:journals/tcs/DesharnaisGJP04} generalize the notion of bisimilariy. 
More specifically, given a policy $\policy \in \mpolicies{\mdp}$, we consider a family $\functionalexpr$ of real-valued functions parameterized by a discount factor $\discount$ and defining the semantics of $\logic$ in $\mdp_\policy$.
Such functional expressions allow to formalize discounted properties such as reachability, safety, as well as general $\omega$-regular specifications
\citep{DBLP:conf/fsttcs/ChatterjeeAMR08}
and may include rewards as well \citep{DBLP:conf/birthday/FernsPK14}.
The pseudometric $\bidistance_{\policy} $
is defined as \emph{the largest behavioral difference} $\bidistance_\policy\fun{\state_1, \state_2} = \sup_{f \in \functionalexpr} \left| f\fun{\state_1} - f\fun{\state_2} \right|$,
and \emph{its kernel is bisimilarity}: $\bidistance_{\policy}\fun{\state_1, \state_2} = 0$ iff $\state_1 \sim \state_2$.
In particular, \emph{value functions are Lipschitz-continuous w.r.t.~$\bidistance_{\policy}$}:
$| \values{\policy}{\scalebox{.95}{$\cdot$}}{\state_1} - \values{\policy}{\scalebox{.95}{$\cdot$}}{\state_2}| \leq 
K \bidistance_{\policy}\fun{\state_1, \state_2}$,
where $K$ is $\nicefrac{1}{\fun{1 - \discount}}$ if rewards are included in $\functionalexpr$ and $1$ otherwise.
To ensure the upcoming bisimulation guarantees, we make the following assumptions:
\begin{assumption}\label{assumption:vae-mdp}
MDP $\mdp$ is ergodic, $\images{\rewards}$ is a bounded space scaled in $\left[\nicefrac{-1}{2}, \nicefrac{1}{2} \right]$, and
the embedding function preserves the labels, i.e., $\embed\fun{\state} = \latentstate \implies \labels\fun{\state} = \latentlabels\fun{\latentstate}$ for $\state \in \states$, $\latentstate \in \latentstates$. 
\end{assumption}
Note that the ergodicity assumption is compliant with episodic RL and a wide range of continuous learning tasks (see \citealt{DBLP:conf/nips/Huang20,DBLP:journals/corr/abs-2112-09655} for detailed discussions on this setting).

\smallparagraph{Bisimulation bounds \citep{DBLP:journals/corr/abs-2112-09655}.}~%
$\mdp$ being set over continuous spaces with possibly unknown dynamics, evaluating $\bidistance$ can turn out to be particularly arduous, if not intractable.
A solution is to evaluate the original and latent model bisimilarity via local losses:
fix
$\latentpolicy \in \latentmpolicies$, assume $\latentmdp$ is discrete, then given the induced stationary distribution $\stationary{\latentpolicy}$ in $\mdp$, let $\state_1, \state_2 \in \states$ with $\embed\fun{\state_1} = \embed\fun{\state_2}$:
\begin{align}
\expectedsymbol{\state \sim \stationary{\latentpolicy}} \bidistance_{{\latentpolicy}}\fun{\state, \embed\fun{\state}} &\leq 
\frac{\localrewardloss{\stationary{\latentpolicy}} + \discount \localtransitionloss{\stationary{\latentpolicy}}}{1 - \discount}, &
\bidistance_{\latentpolicy}\fun{\state_1, \state_2} &\leq \Big( \frac{\localrewardloss{\stationary{\latentpolicy}} +\discount \localtransitionloss{\stationary{\latentpolicy}}}{1 - \discount}\Big) \fun{\stationary{\latentpolicy}^{-1}\fun{\state_1} + \stationary{\latentpolicy}^{-1}\fun{\state_2}}.
\label{eq:bidistance-bound}
\end{align}
The two inequalities guarantee respectively the \emph{quality of the abstraction} and \emph{representation}: when local losses are small, (i) states and their embedding are bisimilarly close in average, and (ii) all states sharing the same discrete representation are bisimilarly close.
The local losses and related bounds can be efficiently PAC-estimated. 
Our goal is to learn a latent model where the behaviors of the agent executing $\latentpolicy$ can be formally verified, and the bounds offer a confidence metric allowing to lift the guarantees obtained this way back to the original model $\mdp$, when the latter operates under $\latentpolicy$. We show in the following how to learn a latent space model by optimizing the aforementioned bounds, and distill policies $\policy \in \mpolicies{\mdp}$ obtained via \emph{any} RL technique to a latent policy $\latentpolicy \in \latentmpolicies$.
\section{Wasserstein Auto-encoded MDPs}
Fix $ \latentmdp_{\decoderparameter} = \tuple{\latentstates, \latentactions, \latentprobtransitions_{\decoderparameter}, \latentrewards_{\decoderparameter}, \latentlabels, \atomicprops, \zinit}$ and $\tuple{\latentmdp_{\decoderparameter}, \embed_{\encoderparameter}, \embeda_{\decoderparameter}}$ as a latent space model of $\mdp$ parameterized by $\encoderparameter$ and $\decoderparameter$.
Our method relies on learning a \emph{behavioral model}
$\stationary{\decoderparameter}$
of $\mdp$ from which we can retrieve the latent space model 
and distill $\pi$.
This can be achieved via the minimization of a suitable discrepancy between $\stationarydecoder$ and
$\mdp_\policy$.
VAE-MDPs optimize a lower bound on the likelihood of the dynamics of $\mdp_\policy$ using the \emph{Kullback-Leibler divergence},
yielding (i) $\latentmdp_{\decoderparameter}$, (ii) a distillation  $\latentpolicy_{\decoderparameter}$ of $\policy$, and (iii) $\embed_{\encoderparameter}$ and $\embeda_{\decoderparameter}$.
Local losses are not directly minimized, but rather variational proxies that do not offer theoretical guarantees during the learning process.
To control the local losses minimization and exploit their theoretical guarantees, 
we present a novel autoencoder that incorporates them in its objective, derived from the OT. 
%
Proofs of the claims made in this Section are provided in Appendix~\ref{appendix:wae-mdp}.
\subsection{The Objective Function}
Assume that $\states$, $\actions$, and $\images{\rewards}$ are respectively equipped with metrics $\distance_{\states}$, $\distance_{\actions}$, and $\distance_{\rewards}$,
we define the  \emph{raw transition distance metric} $\transitiondistance$ as the component-wise sum of distances between states, actions, and rewards occurring of along transitions:
$
    \tracedistance\fun{\tuple{\state_1, \action_1, \reward_1, \state'_1}, \tuple{\state_2, \action_2, \reward_2, \state'_2}} = \distance_\states\fun{\state_1, \state_2} + \distance_{\actions}\fun{\action_1, \action_2} + \distance_{\rewards}\fun{\reward_1, \reward_2} + \distance_{\states}\fun{\state_1', \state'_2}.\notag
$
Given Assumption~\ref{assumption:vae-mdp}, we consider the OT between \emph{local} distributions,
where traces are drawn from episodic RL processes or infinite interactions
(we show in Appendix~\ref{appendix:discrepancy-measure} that considering the OT between trace-based distributions in the limit amounts to reasoning about stationary distributions).
Our goal is to minimize
$\wassersteindist{\transitiondistance}{\stationary{\policy}}{\stationarydecoder}$ so that
\begin{equation}
    \stationarydecoder\fun{\state, \action, \reward, \state'} =  \int_{\latentstates \times \latentactions \times \latentstates} \decoder\fun{\state, \action, \reward, \state' \mid \latentstate, \latentaction, \latentstate'} \, d\latentstationaryprior\fun{\latentstate, \latentaction, \latentstate'}, \label{eq:stationary-decoder}
\end{equation}
where $\decoder$ is a transition decoder and $\latentstationaryprior$ denotes the stationary distribution of the latent model $\latentmdp_{\decoderparameter}$.
%
As proved by \citet{Bousquet2017FromOT}, this model allows to derive a simpler form of the OT: 
instead of finding the optimal coupling
of (i) the stationary distribution $\stationary{\policy}$ of $\mdp_\policy$ 
and (ii) the behavioral model $\stationary{\decoderparameter}$, in the primal definition of $\wassersteindist{\tracedistance}{\stationary{\policy}}{\stationarydecoder}$,
it is sufficient to find an encoder $\transitionencoder$ whose marginal is given by $Q\fun{\latentstate, \latentaction, \latentstate'} = \expectedsymbol{\state, \action, \state' \sim \stationary{\policy}} \transitionencoder\fun{\latentstate, \latentaction, \latentstate' \mid \state, \action, \state'}$ 
and identical to $\stationary{\policy}$. 
This is summarized in the following Theorem, yielding a particular case of \emph{Wasserstein-autoencoder} \cite{DBLP:conf/iclr/TolstikhinBGS18}:
\begin{theorem}
Let $\stationarydecoder$ and $\decoder$ be respectively a behavioral model and transition decoder as defined in Eq.~\ref{eq:stationary-decoder},
$\generative_{\decoderparameter}\colon \latentstates \to \states$ be a state-wise decoder, and 
$\embeda_{\decoderparameter}$ be an action embedding function.
Assume $\decoder$ is deterministic with Dirac function
$G_{\decoderparameter}\fun{\latentstate, \latentaction, \latentstate'} = \tuple{\generative_{\decoderparameter}\fun{\latentstate}, \embeda_{\decoderparameter}\fun{\latentstate, \latentaction}, {\latentrewards_{\decoderparameter}\fun{\latentstate, \latentaction}}, \generative_{\decoderparameter}\fun{\latentstate'}}$, then
\begin{equation*}
    \wassersteindist{\tracedistance}{\stationary{\policy}}{\stationarydecoder}
    = \inf_{\transitionencoder: \, Q = \latentstationaryprior} \, \expectedsymbol{\state, \action, \reward, \state' \sim \stationary{\policy}} \, \expectedsymbol{\latentstate, \latentaction, \latentstate' \sim \transitionencoder\fun{\sampledot \mid \state, \action, \state'}} 
    \tracedistance\fun{\tuple{\state, \action, \reward, \state'}, G_{\decoderparameter}\fun{\latentstate, \latentaction, \latentstate'}}.
\end{equation*}
\end{theorem}
%
Henceforth, fix
$\embed_\encoderparameter \colon \states \to {\latentstates}$
and $\embed_{\encoderparameter}^{\scriptscriptstyle\actions} \colon \latentstates \times \actions \to \distributions{\latentactions}$ as parameterized
state
and action encoders with
$\embed_\encoderparameter\fun{\latentstate, \latentaction, \latentstate' \mid \state, \action, \state'} = \condition{\embed_{\encoderparameter}\fun{\state}=\latentstate} \cdot \embed_{\encoderparameter}^{\scriptscriptstyle \actions}\fun{\latentaction \mid \latentstate, \action} \cdot \condition{\embed_{\encoderparameter}\fun{\state'} =\latentstate'}$%
, and define the marginal encoder as $Q_\encoderparameter = \expectedsymbol{\state, \action, \state' \sim \stationary{\policy}} \embed_{\encoderparameter}\fun{\cdot \mid \state, \action, \state'}$.
%
Training the model components can be achieved via the objective:
\begin{align*}
    \min_{\encoderparameter, \decoderparameter} \, \expectedsymbol{\state, \action, \reward, \state' \sim \stationary{\policy}} \, \expectedsymbol{\latentstate, \latentaction, \latentstate' \sim \embed_{\encoderparameter}\fun{\sampledot \mid \state, \action, \state'}}  \tracedistance\fun{\tuple{\state, \action, \reward, \state'}, G_{\decoderparameter}\fun{\latentstate, \latentaction, \latentstate'}} + \beta \cdot \divergencesymbol\fun{\encoder, \latentstationaryprior},
\end{align*}
where $\divergencesymbol$ is an arbitrary discrepancy metric and $\beta > 0$ a hyperparameter.
Intuitively, the encoder $\embed_{\encoderparameter}$ can be learned by enforcing its marginal distribution $\encoder$ to match $\latentstationaryprior$ through this discrepancy.
\begin{remark}%
If $\mdp$ has a discrete action space, then learning $\latentactions$ is not necessary. We can set $\latentactions = \actions$ using identity functions for the action encoder and decoder
(details in Appendix~\ref{appendix:discrete-action-space}).
\end{remark}
\begin{algorithm}
\caption{Wasserstein$^2$ Auto-Encoded MDP}\label{alg:wwae-mdp}
\DontPrintSemicolon
\KwIn{batch size $N$, max. step $T$, no. of regularizer updates $\ncritic$, penalty coefficient $\delta > 0$}
\SetKwComment{Comment}{$\triangleright$\ }{}
\SetCommentSty{textnormal}
\LinesNotNumbered 
\SetKwBlock{Begin}{function}{end function}
\For{$t = 1$ to $T$}{
    \For{$i = 1$ to $N$}{
        Sample a transition ${\state_i, \action_i, \reward_i, \state^{\prime}_i}$ from the original environment via $\stationary{\policy}$\;
        Embed the transition into the latent space by drawing ${\latentstate_{i}, \latentaction_{i}, \latentstate^{\prime}_{i}}$ from $\embed_{\encoderparameter}({\sampledot \mid \state_i, \action_i, \state^{\prime}_{i}})$\;
        Make the latent space model transition to the next latent state:  $\latentstate^{\star}_{i} \sim \latentprobtransitions_{\decoderparameter}({\sampledot \mid \latentstate_{i}, \latentaction_{i}})$\;
        Sample a latent transition from $\latentstationaryprior$:
        $\latentvariable_i \sim \latentstationaryprior$, $\latentaction'_i \sim \latentpolicy_{\decoderparameter}\fun{\sampledot \mid \latentvariable_i}$, and $\latentvariable_i^{\prime} \sim \latentprobtransitions_{\decoderparameter}\fun{\sampledot \mid \latentvariable_i, \latentaction^{\prime}_i}$\;
    }
    $
    	\mathcal{W} \gets \textstyle \sum_{i = 1}^{N} \steadystatenetwork({\latentstate_{i}, \latentaction_{i}, \latentstate^{\star}_{i}})
    	- \steadystatenetwork({\latentvariable_{i}, \latentaction^{\prime}_{i}, \latentvariable^{\prime}_{i}}) +
    	\transitionlossnetwork({\state_{i}, \action_{i}, \latentstate_{i}, \latentaction_{i}, \latentstate^{\prime}_{i}}) -  \transitionlossnetwork({\state_{i}, \action_{i}, \latentstate_{i}, \latentaction_{i}, \latentstate^{\star}_{i}})$\;
    	$P \gets \textstyle \sum_{i = 1}^{N} \textsc{Gp}\big({\steadystatenetwork, \tuple{\latentstate_{i}, \latentaction_{i}, \latentstate^{\star}_{i}}, \tuple{\latentvariable_{i}, \latentaction^{\prime}_{i}, \latentvariable^{\prime}_i}}\big) + \textsc{Gp}\big({\vx \mapsto \transitionlossnetwork\fun{\state_i, \action_i, \latentstate_{i}, \latentaction_{i}, \vx}, \latentstate^{\prime}_{i}, \latentstate^{\star}_{i}}\big)
    $\;
    Update the Lipschitz networks parameters $\wassersteinparameter$ by ascending $\nicefrac{1}{N} \cdot \fun{\beta  \,\mathcal{W} - \delta \, P}$\;
    \If{$t \bmod \ncritic = 0$
    }{
     $\mathcal{L} \gets \sum_{i = 1}^{N} \distance_{\states}\fun{\state_i,  \generative_{\decoderparameter}\fun{\latentstate_{i}}} + \distance_{\actions}\fun{\action_i, \embeda_{\decoderparameter}\fun{\latentstate_{i}, \latentaction_{i}}} + \distance_{\rewards}\fun{\reward_i,  \latentrewards_{\decoderparameter}\fun{\latentstate_{i}, \latentaction_{i}}} + \distance_{\states}\fun{\state^{\prime}_i, \generative_{\decoderparameter}\fun{\latentstate^{\prime}_{i}}}$\;
     Update the latent space model parameters $\tuple{\encoderparameter, \decoderparameter}$ by descending $\nicefrac{1}{N}\cdot \fun{\mathcal{L} + {\beta} \, \mathcal{W}}$\;
    }
}
\Begin($\textsc{Gp}\fun{\varphi_{\omega},\vect{x}, \vect{y}}$ \hfill $\triangleright\ $ \textbf{Gradient penalty} for $\varphi_{\wassersteinparameter} \colon \R^n \to \R$ and $\vect{x}, \vect{y} \in \R^n$){
\ifarxiv
$\epsilon \sim \mathit{U}\fun{0, 1}$ \Comment*[r]{random noise}
$\tilde{\vect{x}} \gets \epsilon \vect{x} + (1 - \epsilon) \vect{y}$ \Comment*[r]{straight lines between $\vect{x}$ and $\vect{y}$}
\fi
\Return{$\fun{\norm{\gradient_{\tilde{\vect{x}}} \, \varphi_{\wassersteinparameter}\fun{\tilde{\vect{x}}}} - 1}^2$}
}
\end{algorithm}
When $\policy$ is executed in $\mdp$, 
observe that
its \emph{parallel execution} in $\latentmdp_{\decoderparameter}$ 
is enabled by the action encoder $\actionencoder$:
%
given an original state $\state \in \states$, $\policy$ first prescribes the action $\action \sim \policy\fun{\sampledot \mid \state}$, which is then embedded in the latent space via $\latentaction \sim \actionencoder\fun{\sampledot \mid \embed_{\encoderparameter}\fun{\state}, \action}$ (cf. Fig.~\ref{subfig:latent-fow-distillation}).
%
This parallel execution, along with setting  $\divergencesymbol$ to $\wassersteinsymbol{\transitiondistance}$, yield an upper bound on the latent regularization, compliant with the bisimulation bounds. 
A two-fold regularizer is obtained thereby, defining the foundations of our objective function:%
\begin{restatable}{lemma}{regularizerlemma}\label{lem:regularizer-upper-bound}
Define $\originaltolatentstationary{}\fun{\latentstate, \latentaction, \latentstate'} = \expectedsymbol{\state, \action \sim \stationary{\policy}}[\condition{\embed_{\encoderparameter}(\state)=\latentstate} \cdot \embed_{\encoderparameter}^{\scriptscriptstyle \actions}(\latentaction \mid \latentstate, \action) \cdot \latentprobtransitions_{\decoderparameter}(\latentstate' \mid \latentstate, \latentaction)]$ as the distribution of drawing state-action pairs from interacting with $\mdp$, embedding them to the latent spaces, and finally letting them transition to their successor state in $\latentmdp_{\decoderparameter}$. Then,
$
    \wassersteindist{\transitiondistance}{\encoder}{ \latentstationaryprior}
    \leq \wassersteindist{\transitiondistance}{\latentstationaryprior}{\originaltolatentstationary{}} + \localtransitionloss{\stationary{\policy}}.
$
\end{restatable}
We therefore define the \waemdp (\emph{Wasserstein-Wasserstein auto-encoded MDP}) objective as:
\begin{equation*}
    \min_{\encoderparameter, \decoderparameter} \! \! \! \expectedsymbol{}_{\substack{\state, \action, \state' \sim \stationary{\policy} \\ \latentstate, \latentaction, \latentstate' \sim \embed_{\encoderparameter}({\sampledot \mid \state, \action, \state'})}} \! \! \! \!
    \left[{\distance_{\states}\fun{\state, \generative_{\decoderparameter}\fun{\latentstate}} + 
    \distance_{\actions}\fun{\action, \embeda_{\decoderparameter}\fun{\latentstate, \latentaction}} + 
    \distance_{\states}\fun{\state', \generative_{\decoderparameter}\fun{\latentstate'}}}\right] + \localrewardloss{\stationary{\policy}} + \beta \cdot ({ \steadystateregularizer{\policy} + \localtransitionloss{\stationary{\policy}}}),
\end{equation*}
%
%
where $\steadystateregularizer{\policy} = \wassersteindist{\transitiondistance}{\originaltolatentstationary{}}{\latentstationaryprior}$ and
$\localtransitionloss{\stationary{\policy}}$ are respectively called \emph{steady-state} and \emph{transition} regularizers.
The former allows to quantify the distance between the stationary distributions respectively induced by $\policy$ in $\mdp$ and $\latentpolicy_{\decoderparameter}$ in $\latentmdp_{\decoderparameter}$, further enabling the distillation. 
The latter allows to learn the latent dynamics.
%
%
Note that $\localrewardloss{\stationary{\policy}}$ and $\localtransitionloss{\stationary{\policy}}$ --- set over $\stationary{\policy}$ instead of $\stationary{\latentpolicy_{\decoderparameter}}$ --- are not sufficient to ensure the bisimulation bounds (Eq.~\ref{eq:bidistance-bound}): running $\policy$ in $\latentmdp_{\decoderparameter}$ depends on the parallel execution of $\policy$ in the original model, which does not permit its (conventional) verification.  
Breaking this dependency is enabled by learning the distillation $\latentpolicy_{\decoderparameter}$ through $\steadystateregularizer{\policy}$, as shown in Fig.~\ref{subfig:latent-fow-distillation}:
%
minimizing
$\steadystateregularizer{\policy}$
allows to make $\stationary{\policy}$ and $\latentstationaryprior$ closer together, further bridging the gap of the discrepancy between $\policy$ and $\latentpolicy_{\decoderparameter}$.
At any time, recovering the local losses along with the linked bisimulation bounds in the objective function of the \waemdp is allowed by considering the latent policy resulting from this distillation:
\begin{restatable}{theorem}{latentexecutionobjective}\label{thm:latent-execution-objective}
Assume that traces are generated by running a latent policy $\latentpolicy \in \latentpolicies$ in the original environment and let $\distance_{\rewards}$ be the usual Euclidean distance, then the \waemdp objective is
\begin{equation*}
    \min_{\encoderparameter, \decoderparameter} \expected{\state, \state' \sim \stationary{\latentpolicy}}{\distance_{\states}\fun{\state, \generative_{\decoderparameter}\fun{\embed_{\encoderparameter}\fun{\state}}} + \distance_{\states}\fun{\state', \generative_{\decoderparameter}\fun{\embed_{\encoderparameter}\fun{\state'}}}} + \localrewardloss{\stationary{\latentpolicy}} + \beta \cdot ( \steadystateregularizer{\latentpolicy} + \localtransitionloss{\stationary{\latentpolicy}}).
\end{equation*}
\end{restatable}
\smallparagraph{Optimizing the regularizers}~%
is enabled by the dual form of the OT: we introduce two parameterized networks, $\steadystatenetwork$ and $\transitionlossnetwork$, constrained to be $1$-Lipschitz and trained to attain the supremum of the dual:
\[
\steadystateregularizer{\policy}\fun{\wassersteinparameter} = \max_{\wassersteinparameter 
} \; \expectedsymbol{\state, \action \sim \stationary{\policy}}\expectedsymbol{\latentaction \sim \actionencoder\fun{\sampledot \mid \embed_{\encoderparameter}\fun{\state}, \action}}\expectedsymbol{\latentstate^{\star} \sim \latentprobtransitions_{\decoderparameter}\fun{\sampledot \mid \embed_{\encoderparameter}\fun{\state}, \latentaction}} \steadystatenetwork\fun{\embed_{\encoderparameter}\fun{\state}, \latentaction, \latentstate^{\star}}
- \expectedsymbol{\latentvariable, \latentaction', \latentvariable' \sim \latentstationaryprior}\steadystatenetwork\fun{\latentvariable, \latentaction', \latentvariable'}
\]
\[
\localtransitionloss{\stationary{\policy}}\fun{\wassersteinparameter} = \max_{\wassersteinparameter
} \expectedsymbol{\state, \action, \state' \sim \stationary{\policy}} \expectedsymbol{\latentstate, \latentaction, \latentstate' \sim \embed_{\encoderparameter}\fun{\sampledot \mid \state, \action, \state'}}\Big[{\transitionlossnetwork\fun{\state, \action, \latentstate, \latentaction, \latentstate'} - \! \expectedsymbol{\latentstate^{\star} \sim \latentprobtransitions_{\decoderparameter}\fun{\sampledot \mid \latentstate, \latentaction}} \transitionlossnetwork\fun{\state, \action, \latentstate, \latentaction, \latentstate^{\star}}}\Big]
\]
Details to derive this tractable form of $\localtransitionloss{\stationary{\policy}}\fun{\wassersteinparameter}$ are in Appendix~\ref{appendix:tractable-transition-regularizer}.
The networks are constrained via
the gradient penalty approach of 
\cite{DBLP:conf/nips/GulrajaniAADC17}, leveraging that any differentiable function is $1$-Lipschitz iff it has gradients with norm at most $1$ everywhere
(we show in Appendix~\ref{appendix:latent-metric} this is still valid for relaxations of discrete spaces).
%
The final learning process is presented in Algorithm~\ref{alg:wwae-mdp}. 

\begin{figure}
    \centering
    \includegraphics[width=\textwidth]{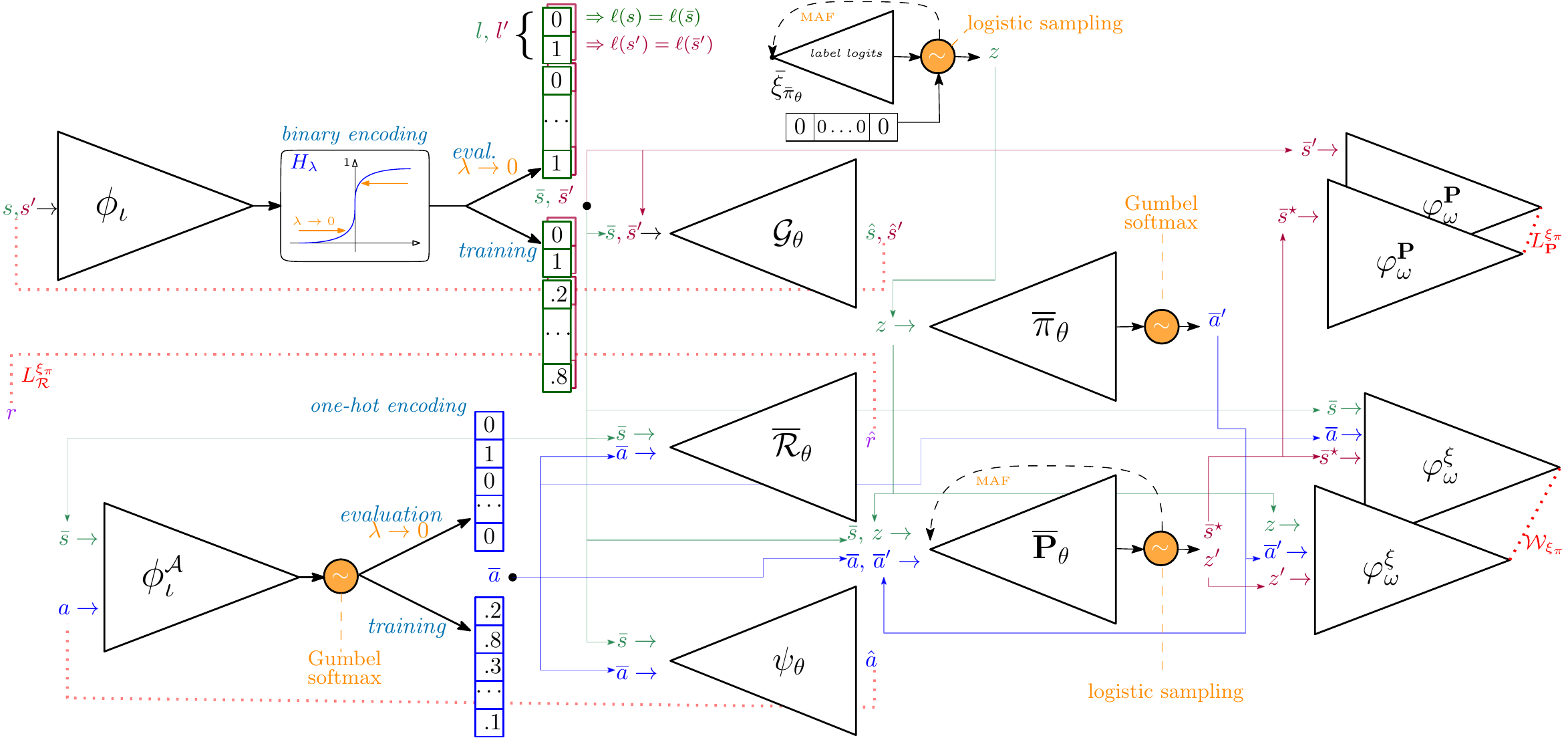}
    \caption{\waemdp architecture.
    Distances are depicted by red dotted lines.}
    \label{fig:wae-architecture}
\end{figure}
\subsection{Discrete Latent Spaces}\label{sec:discrete-latent-spaces}
To enable the verification of latent models supported by the bisimulation guarantees of Eq.~\ref{eq:bidistance-bound}, we focus on the special case of \emph{discrete latent space models}.
Our approach relies on continuous relaxation of discrete random variables, regulated by some \emph{temperature} parameter(s) $\temperature$: discrete random variables are retrieved as $\temperature \to 0$, which amounts to applying a rounding operator.
For training, we use the temperature-controlled relaxations to differentiate the objective and let the gradient flow through the network.
When we deploy the latent policy in the environment and formally check the latent model, the zero-temperature limit is used.
An overview of the approach is depticted in Fig.~\ref{fig:wae-architecture}.

\smallparagraph{State encoder.}~%
We work with a \emph{binary representation} of the latent states.
First, this induces compact networks, able to deal with a large discrete space via a tractable number of parameter variables.
But most importantly, this 
ensures that Assumption~\ref{assumption:vae-mdp} is satisfied:
let $n = \log_2 | \latentstates|$,
we reserve $\left| \atomicprops \right|$ bits in $\latentstates$ and each time $\state \in \states$ is passed to $\embed_{\encoderparameter}$, $n- \left|\atomicprops\right|$ bits are produced and concatenated with $\labels\fun{\state}$, ensuring a perfect reconstruction of the labels and further bisimulation bounds.
To produce Bernoulli variables, $\embed_{\encoderparameter}$ deterministically maps $\state$ to a latent code $\vz$, passed to the Heaviside $H\fun{\vz} = \condition{\vz > 0}$.
We train $\embed_{\encoderparameter}$ by using the smooth approximation $H_{\temperature}\fun{\vz} = \sigma\fun{\nicefrac{2\vz}{\temperature}}$, satisfying $H = \lim_{\temperature \to 0} H_{\temperature}$.

\smallparagraph{Latent distributions.}~%
Besides the discontinuity of their latent image space, a major challenge of optimizing over discrete distributions is \emph{sampling}, required to be a differentiable operation.
We circumvent this by using \emph{concrete distributions} \citep{DBLP:conf/iclr/JangGP17,DBLP:conf/iclr/MaddisonMT17}:
the idea is to sample reparameterizable random variables from $\temperature$-parameterized distributions, and applying a differentiable, nonlinear operator in downstream. 
We use the \emph{Gumbel softmax trick} to sample from distributions over (one-hot encoded) latent actions ($\actionencoder$, $\latentpolicy_{\decoderparameter}$).
For binary distributions ($\latentprobtransitions_{\decoderparameter}$, $\latentstationaryprior$), each relaxed Bernoulli with logit $\alpha$ is retrieved by drawing a logistic random variable located in $\nicefrac{\alpha}{\temperature}$ and scaled to $\nicefrac{1}{\temperature}$, then applying a sigmoid in downstream.
We emphasize that this trick alone (as used by \citealt{DBLP:conf/icml/CorneilGB18,DBLP:journals/corr/abs-2112-09655}) is not sufficient: it yields independent Bernoullis, being too restrictive in general, which prevents from learning sound transition dynamics (cf. Example~\ref{ex:independent-ber-not-sufficient}). 
%
\begin{wrapfigure}{r}{0.4\textwidth}
\begin{tikzpicture}[->,shorten >=1pt,auto,node distance=2.5cm,bend angle=45, scale=0.6, font=\small]
\tikzstyle{state}=[draw,circle,text centered,minimum size=7mm,text width=4mm]
\tikzstyle{act}=[fill,circle,inner sep=1pt,minimum size=1.5pt, node distance=1cm]    
\tikzstyle{empty}=[text centered, text width=15mm]
\node[state] (0) at (0, 0)     {$\latentstate_0$};
\node[state] (1) at (2.5, .8)   {$\latentstate_1$};
\node[empty] (l1) at (1.2, 1.4)  {$\{\textit{\color{OliveGreen}goal}\}$};
\node[state] (2) at (2.5, -.8)  {$\latentstate_2$};
\node[empty] (l2) at (1., -1.3) {$\{\textit{\color{BrickRed}unsafe}\}$};
\node[state] (3) at (-2.5, 0)    {$\latentstate_3$};
\node[empty] (l2) at (-2.5, 1)   {$\{\textit{\color{BrickRed}unsafe}\}$};
\draw[->] (0) -- node[above]{\scriptsize $\nicefrac{1}{2}$} (1);
\draw[->] (0) -- node[below]{\scriptsize $\nicefrac{1}{4}$} (2);
\draw[->] (0) -- node[above]{\scriptsize $\nicefrac{1}{4}$} (3);
\draw[->] (1) edge[out=335,in=25,loop] node[left]{\scriptsize $1$} (1);
\draw[->] (2) edge[out=335,in=25,loop] node[left]{\scriptsize $1$} (2);
\draw[->] (3) edge[out=215,in=155,loop] node[right]{\scriptsize $1$} (3);
\end{tikzpicture}
\caption{Markov Chain with four states%
\ifarxiv{.}
\fi}
\label{fig:markov-chain-independent-bernoulli}
\end{wrapfigure}%
\begin{example}\label{ex:independent-ber-not-sufficient}
\ifarxiv
Let $\latentmdp$ be the discrete MC of Fig.~\ref{fig:markov-chain-independent-bernoulli}
(the labels of $\latentmdp$ are drawn next to their state).
\fi
In one-hot, $\atomicprops = \{\textit{goal}\!: \tuple{{\color{OliveGreen}1}, 0}, \textit{unsafe}\!: \tuple{0, {\color{BrickRed}1}}\}$.
We assume that $3$ bits are used for the (binary) state space, with $\latentstates = \{\latentstate_0 \!: \tuple{{0, 0}, 0}, \latentstate_1\!: \tuple{{\color{OliveGreen}1}, 0, 0}, \latentstate_2\!: \tuple{0, {\color{BrickRed}1}, 0}, \latentstate_3 \!: \tuple{0, {\color{BrickRed}1}, 1}\}$ (the two first bits are reserved for the labels).
Considering each bit as being independent is not sufficient to learn $\latentprobtransitions$: 
the optimal estimation $\latentprobtransitions_{\decoderparameter^{\star}}\fun{\sampledot \mid \latentstate_0}$ is in that case represented by the independent Bernoulli vector $\mathbf{b} = \tuple{\nicefrac{1}{2}, \nicefrac{1}{2}, \nicefrac{1}{4}}$, giving the probability to go from $\latentstate_0$ to each bit \emph{independently}.
This yields a poor estimation of the actual transition function:
$
    \latentprobtransitions_{\decoderparameter^{\star}}\fun{\latentstate_0 \mid \latentstate_0} = (1 - \mathbf{b}_1) \cdot (1 - \mathbf{b}_2) \cdot (1 - \mathbf{b}_3) = 
    \latentprobtransitions_{\decoderparameter^{\star}}\fun{\latentstate_1 \mid \latentstate_0} = \mathbf{b}_1 \cdot (1 - \mathbf{b}_2) \cdot (1 - \mathbf{b}_3) = 
    \latentprobtransitions_{\decoderparameter^{\star}}\fun{\latentstate_2 \mid \latentstate_0} = (1 - \mathbf{b}_1) \cdot \mathbf{b}_2 \cdot (1 - \mathbf{b}_3) = 
    \nicefrac{3}{16}, \,
    \latentprobtransitions_{\decoderparameter^{\star}}\fun{\latentstate_3 \mid \latentstate_0} = (1 - \mathbf{b}_1) \cdot \mathbf{b}_2 \cdot \mathbf{b}_3 = \nicefrac{1}{16}
$.%
\end{example}%
We consider instead relaxed multivariate Bernoulli distributions by decomposing $P \in \distributions{\latentstates}$
as a product of conditionals:
$
P\fun{\latentstate} = \prod_{i=1}^{n} P\fun{\latentstate_{i} \mid \latentstate_{1 \colon i-1}}
$
where $\latentstate_i$ is the $i^\text{th}$ entry (bit) of $\latentstate$.
We learn such distributions by introducing
a \emph{masked autoregressive flow} (MAF, \citealt{DBLP:conf/nips/PapamakariosMP17}) for relaxed Bernoullis via the recursion:
$\latentstate_i = \sigma\fun{\nicefrac{l_i + \alpha_i}{\temperature}}$, \text{where} $l_i \sim \logistic{0}{1}$, $\alpha_i = f_i\fun{\latentstate_{1\colon i - 1}}$, \text{and} $f$ is a MADE \citep{DBLP:conf/icml/GermainGML15}, a feedforward network implementing the conditional output dependency on the inputs via a mask that only keeps the necessary connections to enforce the conditional property.
We use this MAF to model $\latentprobtransitions_{\decoderparameter}$ and the dynamics 
related to
the labels in $\latentstationaryprior$.
We fix the logits of the remaining $n - \left| \atomicprops \right|$ bits to $0$ to allow for a fairly distributed latent space.
\section{Experiments}\label{sec:experiments}
\begin{figure*}
\begin{subfigure}{\textwidth}
    \centering
    \includegraphics[width=\textwidth]{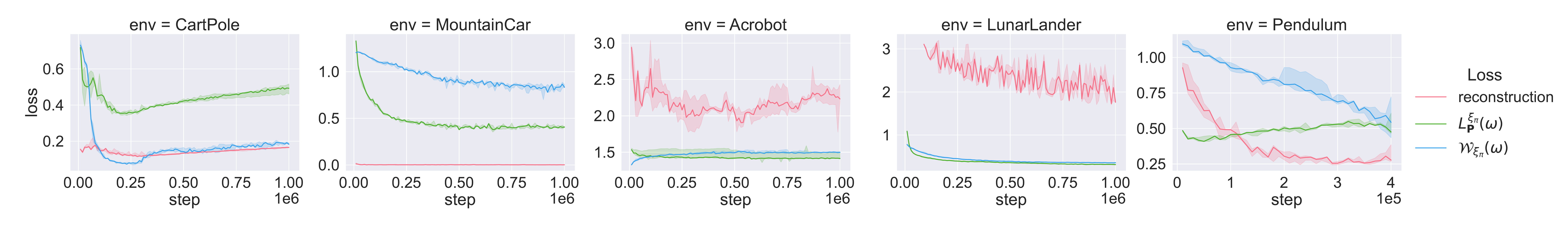}
    \caption{\waemdp objective: reconstruction loss, transition and steady-state regularizers}
    \label{subfig:objective}
\end{subfigure}
\begin{subfigure}{\textwidth}
    \centering
    \includegraphics[width=\textwidth]{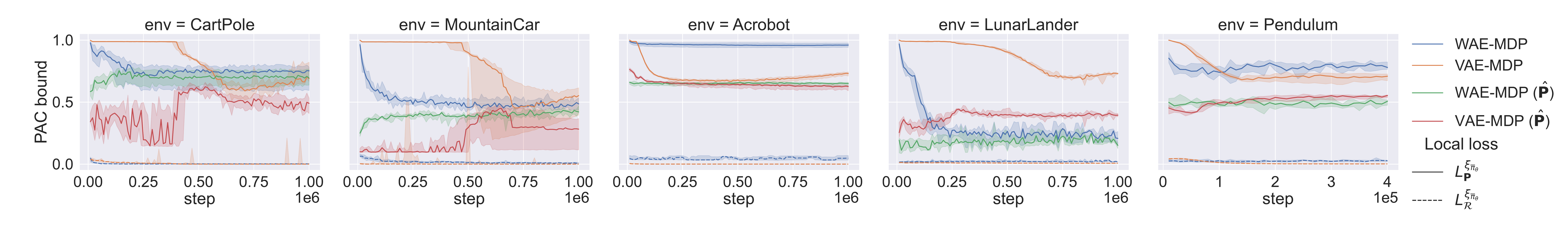}
    \caption{PAC local losses approximation for an error of at most $10^{-2}$ and probability confidence $0.955$}
    \label{subfig:pac-losses}
\end{subfigure}
\begin{subfigure}{\textwidth}
    \centering
    \includegraphics[width=\textwidth]{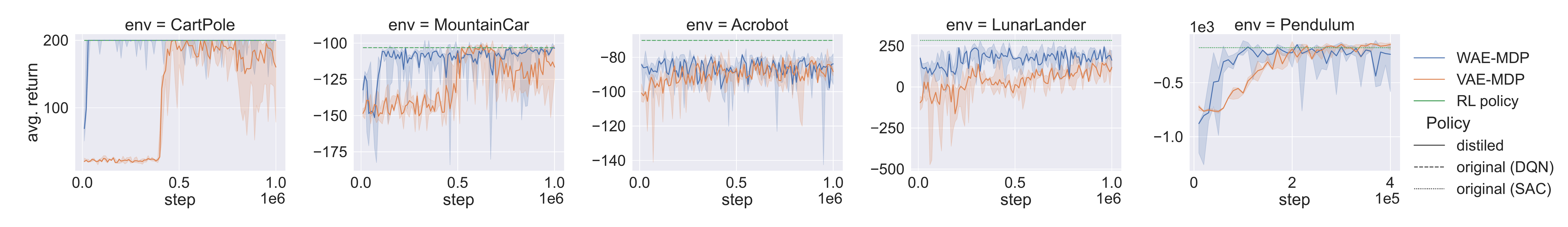}
    \caption{Episode return obtained when executing the distilled policy in the original MDP 
    (averaged over $30$ episodes)}
    \label{subfig:policy-performance}
\end{subfigure}
    \caption{
    For each environment, we trained five different instances of the models with different random seeds: the solid line is the median and the shaded interval the interquartile range.}
    \label{fig:eval-plots}
\end{figure*}
We evaluate the quality of latent space models learned and policies distilled through \waemdps.
To do so, we first trained deep-RL policies (DQN, \citealt{DBLP:journals/nature/MnihKSRVBGRFOPB15} on discrete, and SAC, \citealt{DBLP:conf/icml/HaarnojaZAL18} on continuous action spaces) for various OpenAI benchmarks \citep{DBLP:journals/corr/BrockmanCPSSTZ16}, which we then distill via our approach (Figure~\ref{fig:eval-plots}).
We thus evaluate
    (a) the \waemdp training metrics, 
    (b) the abstraction and representation quality via \emph{PAC local losses upper bounds} \citep{DBLP:journals/corr/abs-2112-09655},
    and (c) the distilled policy performance when deployed in the original environment. 
%
The confidence metrics and performance are compared with those of VAE-MDPs.
Finally, we formally verify properties in the latent model. 
The exact setting to reproduce our results is in Appendix~\ref{appendix:experiments}. 


\smallparagraph{Learning metrics.}~The objective (Fig.~\ref{subfig:objective}) is a weighted sum of the reconstruction loss and the two Wasserstein regularizers.
The choice of $\beta$ defines the optimization direction.
\ifarxiv
Posterior collapse is not observed, naturally avoided in WAEs \citep{DBLP:conf/iclr/TolstikhinBGS18},
which reflects
that the latent space is consistently distributed (see Appendix~\ref{appendix:posterior-collapse} for a discussion and a concrete illustration of collapsing issues occurring in VAE-MDPs).
\fi
Optimizing the objective (Fig.~\ref{subfig:objective}) effectively allows minimizing the local losses (Fig.~\ref{subfig:pac-losses}) and recovering the performance of the original policy (Fig.~\ref{subfig:policy-performance}).

\smallparagraph{Local losses.}~%
%
%
For V- and WAEs, we formally evaluate PAC upper bounds on $\localrewardloss{\stationary{\latentpolicy_{\decoderparameter}}}$ and $\localtransitionloss{\stationary{\latentpolicy_{\decoderparameter}}}$ via the algorithm of \cite{DBLP:journals/corr/abs-2112-09655} (Fig~\ref{subfig:pac-losses}). 
The lower the local losses, the closer $\mdp$ and $\latentmdp_{\decoderparameter}$ are in terms of behaviors induced by $\latentpolicy_{\decoderparameter}$ (cf. Eq.~\ref{eq:bidistance-bound}).
In VAEs, the losses are evaluated on a transition function $\hat{\probtransitions}$ obtained via frequency estimation of the latent transition dynamics \citep{DBLP:journals/corr/abs-2112-09655},
by reconstructing the transition model a posteriori and collecting data to estimate the transition probabilities (e.g., \citealt{DBLP:conf/cav/BazilleGJS20,DBLP:conf/icml/CorneilGB18}).
We thus also report the metrics for $\hat{\probtransitions}$.
Our bounds quickly converge to close values in general for $\latentprobtransitions_{\decoderparameter}$  and $\hat{\probtransitions}$, whereas for VAEs, the convergence is slow and unstable, with $\hat{\probtransitions}$ offering better bounds.
We emphasize that WAEs do not require this additional reconstruction step to obtain losses that can be leveraged to assess the quality of the model, in contrast to VAEs, where learning $\latentprobtransitions_{\decoderparameter}$ was performed via overly restrictive distributions, leading to poor estimation in general (cf. Ex.~\ref{ex:independent-ber-not-sufficient}).
Finally, \emph{when the distilled policies offer comparable performance} (Fig.~\ref{subfig:policy-performance}), our bounds are either close to or better than those of VAEs.

\smallparagraph{Distillation.}~%
The bisimulation guarantees (Eq.~\ref{eq:bidistance-bound}) are only valid for $\latentpolicy_{\decoderparameter}$, the policy under which formal properties can be verified.
It is crucial that $\latentpolicy_{\decoderparameter}$ achieves performance close to $\policy$, the original one, when deployed in the RL environment. 
We evaluate the performance of $\latentpolicy_{\decoderparameter}$ via the undiscounted episode return $\episodereturn{\latentpolicy_{\decoderparameter}}$ obtained by running $\latentpolicy_{\decoderparameter}$ in the original model $\mdp$.
We observe that $\episodereturn{\latentpolicy_{\decoderparameter}}$ approaches faster the original performance $\episodereturn{\policy}$ for
W- than VAEs:
WAEs converge in a few steps
for all environments, whereas the full learning budget 
is sometimes necessary with VAEs.
The success in recovering the original performance emphasizes the representation quality guarantees (Eq.~\ref{eq:bidistance-bound}) induced by WAEs:
when local losses are minimized, all original states that are embedded to the same representation are bisimilarly close. 
Distilling the policy over the new representation, albeit discrete and hence coarser, still achieves effective performance since
$\embed_{\encoderparameter}$
keeps only what is important to preserve behaviors, and thus values. 
Furthermore, the distillation can remove some non-robustness obtained during RL:
$\latentpolicy_{\decoderparameter}$ prescribes the same actions for bisimilarly close states, whereas this is not necessarily the case for $\policy$.
\begin{table}
\centering
\caption{Formal Verification of distilled policies. Values are computed for $\discount = 0.99$ (lower is better).}
\label{table:evaluation}
\resizebox{0.9\columnwidth}{!}{%
\begin{tabular}{@{}llllllllllll@{}}
\toprule
Environment &
  step ($10^5$)&
  $\states$ &
  $\actions$ &
  $\left|\latentstates\right|$ &
  $\left|\latentactions\right|$ &
  $\localrewardloss{\stationary{\latentpolicy_{\decoderparameter}}}$ (PAC) &
  $\localtransitionloss{\stationary{\latentpolicy_{\decoderparameter}}}$ (PAC) &
  $\norm{V_{\latentpolicy_{\decoderparameter}}}$ &
  $\latentvalues{\latentpolicy_{\decoderparameter}}{\varphi}{\zinit}$ \\ \midrule
{CartPole} &
  $1.2$ &
  $\subseteq \R^4$ &
  $\set{1, 2}$ &
  $512$ &
  $2$ &
  $0.00499653$ &
  $0.399636$ &
  $3.71213$ &
  $0.0316655$ \\
{MountainCar} &
  $2.32$ &
  $\subseteq \R^2$ &
  $\set{1, 2}$ &
  $1024$ &
  $2$ &
  $0.0141763$ &
  $0.382323$ &
  $2.83714$ &
  $0$ \\
{Acrobot} &
  $4.3 $ &
  $\subseteq \R^6$ &
  $\set{1, 2, 3}$ &
  $8192$ &
  $3$ &
  $0.0347698$ &
  $0.649478$ &
  $2.22006$ &
  $0.0021911$ \\
  {LunarLander} &
  $3.2 $ &
  $\subseteq \R^8$ &
  $\mathopen[-1, 1\mathclose]^2$ &
  $16384$ &
  $3$ &
  $0.0207205$ &
  $0.131357$ &
  $0.0372883$ &
  $0.0702039$ \\
{Pendulum} &
  $3.7 $ &
  $\subseteq \R^3$ &
  $\mathopen[ -2, 2 \mathclose]$ &
  $8192$ &
  $3$ &
  $0.0266745$ &
  $0.539508$ &
  $4.33006$ &
  $0.0348492$ \\
 \bottomrule
\end{tabular}%
}
\end{table}

\smallparagraph{Formal verification.}~%
To formally verify $\latentmdp_\decoderparameter$, we implemented a \emph{value iteration} (VI) engine, handling the neural network encoding of the latent space for discounted properties, which is one of the most popular algorithms for checking property probabilities in MDPs (e.g., \citealt{DBLP:BK08, stt:Hensel2021, doi:10.1146/annurev-control-042820-010947}).
We verify \emph{time-to-failure} properties $\varphi$, often used to check the failure rate of a system \citep{DBLP:conf/focs/Pnueli77} by measuring whether the agent fails \emph{before the end of the episode}. 
Although simple, such properties highlight the applicability of our approach on reachability events, which are building blocks to verify MDPs
\ifarxiv
(\citealt{DBLP:BK08}; cf. Appendix~\ref{rmk:reachability} for a discussion).
\fi
In particular, we checked whether the agent reaches an unsafe position or angle ({CartPole}, {LunarLander}), does not reach its goal position (MountainCar, Acrobot), and does not reach and stay in a safe region of the system (Pendulum).
%
Results are in Table~\ref{table:evaluation}: for each environment, we select the distilled policy which gives the best trade-off between performance (episode return) and abstraction quality (local losses). 
As extra confidence metric, we report the value difference 
$\norm{V_{\latentpolicy_{\decoderparameter}}} = |\values{\latentpolicy_{\decoderparameter}}{}{\sinit} - \latentvalues{\latentpolicy_{\decoderparameter}}{}{\zinit}|$ obtained by executing $\latentpolicy_{\decoderparameter}$ in $\mdp$ and $\latentmdp_{\decoderparameter}$ 
($\values{\latentpolicy_{\decoderparameter}}{}{\sampledot}$ is averaged while $\latentvalues{\latentpolicy_{\decoderparameter}}{}{\sampledot}$ is formally computed).
\section{Conclusion}\label{sec:conclusion}

We presented  WAE-MDPs, a framework for learning formally verifiable distillations of RL policies with bisimulation guarantees.
The latter, along with the learned abstraction of the unknown continuous environment to a discrete model, 
enables the verification.
Our method overcomes the limitations of VAE-MDPs and our results show that it outperforms the latter in terms of learning speed, model quality, and performance, in addition to being supported by stronger learning guarantees.
As mentioned by \citet{DBLP:journals/corr/abs-2112-09655}, distillation failure reveals the lack of robustness of original RL policies.
In particular, we found that distilling highly noise-sensitive RL policies (such as robotics simulations, e.g., \citealt{todorov2012mujoco}) is laborious, even though the result remains formally verifiable.

We demonstrated the feasibility of our approach through the verification of reachability objectives, which are building blocks for stochastic model-checking \citep{DBLP:BK08}.
Besides the scope of this work, the verification of general discounted $\omega$-regular properties is theoretically allowed in our model via the rechability to components of standard constructions based on automata products (e.g., \citealt{
DBLP:conf/cav/BaierK0K0W16,DBLP:conf/cav/SickertEJK16}), and discounted games algorithms \citep{DBLP:conf/fsttcs/ChatterjeeAMR08}.
Beyond distillation, our results, supported by Thm.~\ref{thm:latent-execution-objective}, suggest that our WAE-MDP can be used as a \emph{general latent space learner} for RL, further opening possibilities to combine RL and formal methods \emph{online} when no formal model is a priori known, and  address this way safety in  RL with guarantees. 

\subsubsection*{Reproducibility Statement}
We referenced in the main text the Appendix parts presenting the proofs or additional details of every claim, Assumption, Lemma, and Theorem occurring in the paper.
In addition, Appendix~\ref{appendix:experiments} is dedicated to the presentation of the setup, hyperparameters, and other extra details required for reproducing the results of Section~\ref{sec:experiments}.
We provide the source code of the implementation of our approach in Supplementary material \footnote{available at \url{https://github.com/florentdelgrange/wae_mdp}}, and we also provide the models saved during training that we used for model checking (i.e., reproducing the results of Table~\ref{table:evaluation}).
Additionally, we present in a notebook ({\small\texttt{evaluation.html}}) videos demonstrating how our distilled policies behave in each environment, and code snippets showing how we formally verified the policies. 

\subsubsection*{Acknowledgments}
This research received funding from the Flemish Government (AI Research Program) and was supported by the DESCARTES iBOF project. G.A. Perez is also supported by the Belgian FWO “SAILor” project (G030020N).
We thank Raphael Avalos for his valuable feedback during the preparation of this manuscript.

\bibliography{references}

\begin{thebibliography}{11}
\providecommand{\natexlab}[1]{#1}
\providecommand{\url}[1]{\texttt{#1}}
\expandafter\ifx\csname urlstyle\endcsname\relax
  \providecommand{\doi}[1]{doi: #1}\else
  \providecommand{\doi}{doi: \begingroup \urlstyle{rm}\Url}\fi

\bibitem[Abadi et~al.(2015)Abadi, Agarwal, Barham, Brevdo, Chen, Citro,
  Corrado, Davis, Dean, Devin, Ghemawat, Goodfellow, Harp, Irving, Isard, Jia,
  Jozefowicz, Kaiser, Kudlur, Levenberg, Man\'{e}, Monga, Moore, Murray, Olah,
  Schuster, Shlens, Steiner, Sutskever, Talwar, Tucker, Vanhoucke, Vasudevan,
  Vi\'{e}gas, Vinyals, Warden, Wattenberg, Wicke, Yu, and
  Zheng]{tensorflow2015-whitepaper}
Mart\'{\i}n Abadi, Ashish Agarwal, Paul Barham, Eugene Brevdo, Zhifeng Chen,
  Craig Citro, Greg~S. Corrado, Andy Davis, Jeffrey Dean, Matthieu Devin,
  Sanjay Ghemawat, Ian Goodfellow, Andrew Harp, Geoffrey Irving, Michael Isard,
  Yangqing Jia, Rafal Jozefowicz, Lukasz Kaiser, Manjunath Kudlur, Josh
  Levenberg, Dandelion Man\'{e}, Rajat Monga, Sherry Moore, Derek Murray, Chris
  Olah, Mike Schuster, Jonathon Shlens, Benoit Steiner, Ilya Sutskever, Kunal
  Talwar, Paul Tucker, Vincent Vanhoucke, Vijay Vasudevan, Fernanda Vi\'{e}gas,
  Oriol Vinyals, Pete Warden, Martin Wattenberg, Martin Wicke, Yuan Yu, and
  Xiaoqiang Zheng.
\newblock {TensorFlow}: Large-scale machine learning on heterogeneous systems,
  2015.
\newblock URL \url{https://www.tensorflow.org/}.
\newblock Software available from tensorflow.org.

\bibitem[Alemi et~al.(2018)Alemi, Poole, Fischer, Dillon, Saurous, and
  Murphy]{DBLP:conf/icml/AlemiPFDS018}
Alexander~A. Alemi, Ben Poole, Ian Fischer, Joshua~V. Dillon, Rif~A. Saurous,
  and Kevin Murphy.
\newblock Fixing a broken {ELBO}.
\newblock In Jennifer~G. Dy and Andreas Krause (eds.), \emph{Proceedings of the
  35th International Conference on Machine Learning, {ICML} 2018,
  Stockholmsm{\"{a}}ssan, Stockholm, Sweden, July 10-15, 2018}, volume~80 of
  \emph{Proceedings of Machine Learning Research}, pp.\  159--168. {PMLR},
  2018.
\newblock URL \url{http://proceedings.mlr.press/v80/alemi18a.html}.

\bibitem[Dillon et~al.(2017)Dillon, Langmore, Tran, Brevdo, Vasudevan, Moore,
  Patton, Alemi, Hoffman, and Saurous]{dillon2017tensorflow}
Joshua~V. Dillon, Ian Langmore, Dustin Tran, Eugene Brevdo, Srinivas Vasudevan,
  Dave Moore, Brian Patton, Alex Alemi, Matt Hoffman, and Rif~A. Saurous.
\newblock Tensorflow distributions, 2017.

\bibitem[Dong et~al.(2020)Dong, Seybold, Murphy, and
  Bui]{DBLP:conf/icml/DongS0B20}
Zhe Dong, Bryan~A. Seybold, Kevin Murphy, and Hung~H. Bui.
\newblock Collapsed amortized variational inference for switching nonlinear
  dynamical systems.
\newblock In \emph{Proceedings of the 37th International Conference on Machine
  Learning, {ICML} 2020, 13-18 July 2020, Virtual Event}, volume 119 of
  \emph{Proceedings of Machine Learning Research}, pp.\  2638--2647. {PMLR},
  2020.
\newblock URL \url{http://proceedings.mlr.press/v119/dong20e.html}.

\bibitem[Guadarrama et~al.(2018)Guadarrama, Korattikara, Ramirez, Castro,
  Holly, Fishman, Wang, Gonina, Wu, Kokiopoulou, Sbaiz, Smith, Bartók, Berent,
  Harris, Vanhoucke, and Brevdo]{TFAgents}
Sergio Guadarrama, Anoop Korattikara, Oscar Ramirez, Pablo Castro, Ethan Holly,
  Sam Fishman, Ke~Wang, Ekaterina Gonina, Neal Wu, Efi Kokiopoulou, Luciano
  Sbaiz, Jamie Smith, Gábor Bartók, Jesse Berent, Chris Harris, Vincent
  Vanhoucke, and Eugene Brevdo.
\newblock {TF-Agents}: A library for reinforcement learning in tensorflow.
\newblock \url{https://github.com/tensorflow/agents}, 2018.
\newblock URL \url{https://github.com/tensorflow/agents}.
\newblock [Online; accessed 25-June-2019].

\bibitem[He et~al.(2019)He, Spokoyny, Neubig, and
  Berg{-}Kirkpatrick]{DBLP:conf/iclr/HeSNB19}
Junxian He, Daniel Spokoyny, Graham Neubig, and Taylor Berg{-}Kirkpatrick.
\newblock Lagging inference networks and posterior collapse in variational
  autoencoders.
\newblock In \emph{7th International Conference on Learning Representations,
  {ICLR} 2019, New Orleans, LA, USA, May 6-9, 2019}. OpenReview.net, 2019.
\newblock URL \url{https://openreview.net/forum?id=rylDfnCqF7}.

\bibitem[Hoffman et~al.(2013)Hoffman, Blei, Wang, and
  Paisley]{DBLP:journals/jmlr/HoffmanBWP13}
Matthew~D. Hoffman, David~M. Blei, Chong Wang, and John~W. Paisley.
\newblock Stochastic variational inference.
\newblock \emph{J. Mach. Learn. Res.}, 14\penalty0 (1):\penalty0 1303--1347,
  2013.
\newblock URL \url{http://dl.acm.org/citation.cfm?id=2502622}.

\bibitem[Kingma \& Ba(2015)Kingma and Ba]{DBLP:journals/corr/KingmaB14}
Diederik~P. Kingma and Jimmy Ba.
\newblock Adam: {A} method for stochastic optimization.
\newblock In Yoshua Bengio and Yann LeCun (eds.), \emph{3rd International
  Conference on Learning Representations, {ICLR} 2015, San Diego, CA, USA, May
  7-9, 2015, Conference Track Proceedings}, 2015.
\newblock URL \url{http://arxiv.org/abs/1412.6980}.

\bibitem[Kingma \& Welling(2014)Kingma and
  Welling]{DBLP:journals/corr/KingmaW13}
Diederik~P. Kingma and Max Welling.
\newblock Auto-encoding variational bayes.
\newblock In Yoshua Bengio and Yann LeCun (eds.), \emph{2nd International
  Conference on Learning Representations, {ICLR} 2014, Banff, AB, Canada, April
  14-16, 2014, Conference Track Proceedings}, 2014.
\newblock URL \url{http://arxiv.org/abs/1312.6114}.

\bibitem[Kulkarni(1995)]{10.5555/280952}
Vidyadhar~G. Kulkarni.
\newblock \emph{Modeling and Analysis of Stochastic Systems}.
\newblock Chapman \& Hall, Ltd., GBR, 1995.
\newblock ISBN 0412049910.

\bibitem[Tolstikhin et~al.(2018)Tolstikhin, Bousquet, Gelly, and
  Sch{\"{o}}lkopf]{DBLP:conf/iclr/TolstikhinBGS18}
Ilya~O. Tolstikhin, Olivier Bousquet, Sylvain Gelly, and Bernhard
  Sch{\"{o}}lkopf.
\newblock Wasserstein auto-encoders.
\newblock In \emph{6th International Conference on Learning Representations,
  {ICLR} 2018, Vancouver, BC, Canada, April 30 - May 3, 2018, Conference Track
  Proceedings}. OpenReview.net, 2018.
\newblock URL \url{https://openreview.net/forum?id=HkL7n1-0b}.

\end{thebibliography}


\begin{thebibliography}{54}
\providecommand{\natexlab}[1]{#1}
\providecommand{\url}[1]{\texttt{#1}}
\expandafter\ifx\csname urlstyle\endcsname\relax
  \providecommand{\doi}[1]{doi: #1}\else
  \providecommand{\doi}{doi: \begingroup \urlstyle{rm}\Url}\fi

\bibitem[Alamdari et~al.(2020)Alamdari, Avni, Henzinger, and
  Lukina]{DBLP:conf/fmcad/AlamdariAHL20}
Parand~Alizadeh Alamdari, Guy Avni, Thomas~A. Henzinger, and Anna Lukina.
\newblock Formal methods with a touch of magic.
\newblock In \emph{2020 Formal Methods in Computer Aided Design, {FMCAD} 2020,
  Haifa, Israel, September 21-24, 2020}, pp.\  138--147. {IEEE}, 2020.
\newblock \doi{10.34727/2020/isbn.978-3-85448-042-6\_21}.
\newblock URL \url{https://doi.org/10.34727/2020/isbn.978-3-85448-042-6\_21}.

\bibitem[Alemi et~al.(2018)Alemi, Poole, Fischer, Dillon, Saurous, and
  Murphy]{DBLP:conf/icml/AlemiPFDS018}
Alexander~A. Alemi, Ben Poole, Ian Fischer, Joshua~V. Dillon, Rif~A. Saurous,
  and Kevin Murphy.
\newblock Fixing a broken {ELBO}.
\newblock In Jennifer~G. Dy and Andreas Krause (eds.), \emph{Proceedings of the
  35th International Conference on Machine Learning, {ICML} 2018,
  Stockholmsm{\"{a}}ssan, Stockholm, Sweden, July 10-15, 2018}, volume~80 of
  \emph{Proceedings of Machine Learning Research}, pp.\  159--168. {PMLR},
  2018.
\newblock URL \url{http://proceedings.mlr.press/v80/alemi18a.html}.

\bibitem[Alshiekh et~al.(2018)Alshiekh, Bloem, Ehlers, K{\"{o}}nighofer,
  Niekum, and Topcu]{DBLP:conf/aaai/AlshiekhBEKNT18}
Mohammed Alshiekh, Roderick Bloem, R{\"{u}}diger Ehlers, Bettina
  K{\"{o}}nighofer, Scott Niekum, and Ufuk Topcu.
\newblock Safe reinforcement learning via shielding.
\newblock In Sheila~A. McIlraith and Kilian~Q. Weinberger (eds.),
  \emph{Proceedings of the Thirty-Second {AAAI} Conference on Artificial
  Intelligence, (AAAI-18), the 30th innovative Applications of Artificial
  Intelligence (IAAI-18), and the 8th {AAAI} Symposium on Educational Advances
  in Artificial Intelligence (EAAI-18), New Orleans, Louisiana, USA, February
  2-7, 2018}, pp.\  2669--2678. {AAAI} Press, 2018.
\newblock URL
  \url{https://www.aaai.org/ocs/index.php/AAAI/AAAI18/paper/view/17211}.

\bibitem[Arjovsky et~al.(2017)Arjovsky, Chintala, and
  Bottou]{DBLP:conf/icml/ArjovskyCB17}
Mart{\'{\i}}n Arjovsky, Soumith Chintala, and L{\'{e}}on Bottou.
\newblock Wasserstein generative adversarial networks.
\newblock In Doina Precup and Yee~Whye Teh (eds.), \emph{Proceedings of the
  34th International Conference on Machine Learning, {ICML} 2017, Sydney, NSW,
  Australia, 6-11 August 2017}, volume~70 of \emph{Proceedings of Machine
  Learning Research}, pp.\  214--223. {PMLR}, 2017.
\newblock URL \url{http://proceedings.mlr.press/v70/arjovsky17a.html}.

\bibitem[Bacci \& Parker(2020)Bacci and Parker]{DBLP:conf/formats/Bacci020}
Edoardo Bacci and David Parker.
\newblock Probabilistic guarantees for safe deep reinforcement learning.
\newblock In Nathalie Bertrand and Nils Jansen (eds.), \emph{Formal Modeling
  and Analysis of Timed Systems - 18th International Conference, {FORMATS}
  2020, Vienna, Austria, September 1-3, 2020, Proceedings}, volume 12288 of
  \emph{LNCS}, pp.\  231--248. Springer, 2020.
\newblock \doi{10.1007/978-3-030-57628-8\_14}.
\newblock URL \url{https://doi.org/10.1007/978-3-030-57628-8\_14}.

\bibitem[Baier \& Katoen(2008)Baier and Katoen]{DBLP:BK08}
Christel Baier and Joost{-}Pieter Katoen.
\newblock \emph{Principles of model checking}.
\newblock {MIT} Press, 2008.
\newblock ISBN 978-0-262-02649-9.

\bibitem[Baier et~al.(2016)Baier, Kiefer, Klein, Kl{\"{u}}ppelholz,
  M{\"{u}}ller, and Worrell]{DBLP:conf/cav/BaierK0K0W16}
Christel Baier, Stefan Kiefer, Joachim Klein, Sascha Kl{\"{u}}ppelholz, David
  M{\"{u}}ller, and James Worrell.
\newblock Markov chains and unambiguous b{\"{u}}chi automata.
\newblock In Swarat Chaudhuri and Azadeh Farzan (eds.), \emph{Computer Aided
  Verification - 28th International Conference, {CAV} 2016, Toronto, ON,
  Canada, July 17-23, 2016, Proceedings, Part {I}}, volume 9779 of
  \emph{Lecture Notes in Computer Science}, pp.\  23--42. Springer, 2016.
\newblock \doi{10.1007/978-3-319-41528-4\_2}.
\newblock URL \url{https://doi.org/10.1007/978-3-319-41528-4\_2}.

\bibitem[Bazille et~al.(2020)Bazille, Genest, J{\'{e}}gourel, and
  Sun]{DBLP:conf/cav/BazilleGJS20}
Hugo Bazille, Blaise Genest, Cyrille J{\'{e}}gourel, and Jun Sun.
\newblock Global {PAC} bounds for learning discrete time markov chains.
\newblock In Shuvendu~K. Lahiri and Chao Wang (eds.), \emph{Computer Aided
  Verification - 32nd International Conference, {CAV} 2020, Los Angeles, CA,
  USA, July 21-24, 2020, Proceedings, Part {II}}, volume 12225 of \emph{Lecture
  Notes in Computer Science}, pp.\  304--326. Springer, 2020.
\newblock \doi{10.1007/978-3-030-53291-8\_17}.
\newblock URL \url{https://doi.org/10.1007/978-3-030-53291-8\_17}.

\bibitem[Bousquet et~al.(2017)Bousquet, Gelly, Tolstikhin, Simon-Gabriel, and
  Sch{\"o}lkopf]{Bousquet2017FromOT}
O.~Bousquet, S.~Gelly, I.~Tolstikhin, Carl-Johann Simon-Gabriel, and
  B.~Sch{\"o}lkopf.
\newblock From optimal transport to generative modeling: the vegan cookbook.
\newblock \emph{arXiv: Machine Learning}, 2017.

\bibitem[Brockman et~al.(2016)Brockman, Cheung, Pettersson, Schneider,
  Schulman, Tang, and Zaremba]{DBLP:journals/corr/BrockmanCPSSTZ16}
Greg Brockman, Vicki Cheung, Ludwig Pettersson, Jonas Schneider, John Schulman,
  Jie Tang, and Wojciech Zaremba.
\newblock Openai gym.
\newblock \emph{CoRR}, abs/1606.01540, 2016.
\newblock URL \url{http://arxiv.org/abs/1606.01540}.

\bibitem[Carr et~al.(2020)Carr, Jansen, and Topcu]{DBLP:conf/ijcai/CarrJT20}
Steven Carr, Nils Jansen, and Ufuk Topcu.
\newblock Verifiable rnn-based policies for pomdps under temporal logic
  constraints.
\newblock In Christian Bessiere (ed.), \emph{Proceedings of the Twenty-Ninth
  International Joint Conference on Artificial Intelligence, {IJCAI} 2020},
  pp.\  4121--4127. ijcai.org, 2020.
\newblock \doi{10.24963/ijcai.2020/570}.
\newblock URL \url{https://doi.org/10.24963/ijcai.2020/570}.

\bibitem[Castro et~al.(2021)Castro, Kastner, Panangaden, and
  Rowland]{DBLP:conf/nips/CastroKPR21}
Pablo~Samuel Castro, Tyler Kastner, Prakash Panangaden, and Mark Rowland.
\newblock Mico: Improved representations via sampling-based state similarity
  for markov decision processes.
\newblock In Marc'Aurelio Ranzato, Alina Beygelzimer, Yann~N. Dauphin, Percy
  Liang, and Jennifer~Wortman Vaughan (eds.), \emph{Advances in Neural
  Information Processing Systems 34: Annual Conference on Neural Information
  Processing Systems 2021, NeurIPS 2021, December 6-14, 2021, virtual}, pp.\
  30113--30126, 2021.
\newblock URL
  \url{https://proceedings.neurips.cc/paper/2021/hash/fd06b8ea02fe5b1c2496fe1700e9d16c-Abstract.html}.

\bibitem[Ceusters et~al.(2021)Ceusters, Rodríguez, García, Franke, Deconinck,
  Helsen, Nowé, Messagie, and Camargo]{CEUSTERS2021117634}
Glenn Ceusters, Román~Cantú Rodríguez, Alberte~Bouso García, Rüdiger
  Franke, Geert Deconinck, Lieve Helsen, Ann Nowé, Maarten Messagie, and
  Luis~Ramirez Camargo.
\newblock Model-predictive control and reinforcement learning in multi-energy
  system case studies.
\newblock \emph{Applied Energy}, 303:\penalty0 117634, 2021.
\newblock ISSN 0306-2619.
\newblock \doi{https://doi.org/10.1016/j.apenergy.2021.117634}.
\newblock URL
  \url{https://www.sciencedirect.com/science/article/pii/S0306261921010011}.

\bibitem[Chatterjee et~al.(2010)Chatterjee, de~Alfaro, Majumdar, and
  Raman]{DBLP:conf/fsttcs/ChatterjeeAMR08}
Krishnendu Chatterjee, Luca de~Alfaro, Rupak Majumdar, and Vishwanath Raman.
\newblock Algorithms for game metrics (full version).
\newblock \emph{Log. Methods Comput. Sci.}, 6\penalty0 (3), 2010.
\newblock URL \url{http://arxiv.org/abs/0809.4326}.

\bibitem[Corneil et~al.(2018)Corneil, Gerstner, and
  Brea]{DBLP:conf/icml/CorneilGB18}
Dane~S. Corneil, Wulfram Gerstner, and Johanni Brea.
\newblock Efficient modelbased deep reinforcement learning with variational
  state tabulation.
\newblock In Jennifer~G. Dy and Andreas Krause (eds.), \emph{Proceedings of the
  35th International Conference on Machine Learning, {ICML} 2018,
  Stockholmsm{\"{a}}ssan, Stockholm, Sweden, July 10-15, 2018}, volume~80 of
  \emph{Proceedings of Machine Learning Research}, pp.\  1057--1066. {PMLR},
  2018.
\newblock URL \url{http://proceedings.mlr.press/v80/corneil18a.html}.

\bibitem[Delgrange et~al.(2022)Delgrange, Nowé, and
  Pérez]{DBLP:journals/corr/abs-2112-09655}
Florent Delgrange, Ann Nowé, and Guillermo~A. Pérez.
\newblock Distillation of rl policies with formal guarantees via variational
  abstraction of markov decision processes.
\newblock \emph{Proceedings of the AAAI Conference on Artificial Intelligence},
  36\penalty0 (6):\penalty0 6497--6505, Jun. 2022.
\newblock \doi{10.1609/aaai.v36i6.20602}.
\newblock URL \url{https://ojs.aaai.org/index.php/AAAI/article/view/20602}.

\bibitem[Desharnais et~al.(2004)Desharnais, Gupta, Jagadeesan, and
  Panangaden]{DBLP:journals/tcs/DesharnaisGJP04}
Jos{\'{e}}e Desharnais, Vineet Gupta, Radha Jagadeesan, and Prakash Panangaden.
\newblock Metrics for labelled markov processes.
\newblock \emph{Theor. Comput. Sci.}, 318\penalty0 (3):\penalty0 323--354,
  2004.
\newblock \doi{10.1016/j.tcs.2003.09.013}.
\newblock URL \url{https://doi.org/10.1016/j.tcs.2003.09.013}.

\bibitem[Fajtl et~al.(2020)Fajtl, Argyriou, Monekosso, and
  Remagnino]{DBLP:conf/icml/FajtlAMR20}
Jiri Fajtl, Vasileios Argyriou, Dorothy Monekosso, and Paolo Remagnino.
\newblock Latent bernoulli autoencoder.
\newblock In \emph{Proceedings of the 37th International Conference on Machine
  Learning, {ICML} 2020, 13-18 July 2020, Virtual Event}, volume 119 of
  \emph{Proceedings of Machine Learning Research}, pp.\  2964--2974. {PMLR},
  2020.
\newblock URL \url{http://proceedings.mlr.press/v119/fajtl20a.html}.

\bibitem[Ferns et~al.(2014)Ferns, Precup, and
  Knight]{DBLP:conf/birthday/FernsPK14}
Norm Ferns, Doina Precup, and Sophia Knight.
\newblock Bisimulation for markov decision processes through families of
  functional expressions.
\newblock In Franck van Breugel, Elham Kashefi, Catuscia Palamidessi, and Jan
  Rutten (eds.), \emph{Horizons of the Mind. {A} Tribute to Prakash Panangaden
  - Essays Dedicated to Prakash Panangaden on the Occasion of His 60th
  Birthday}, volume 8464 of \emph{LNCS}, pp.\  319--342. Springer, 2014.
\newblock \doi{10.1007/978-3-319-06880-0\_17}.
\newblock URL \url{https://doi.org/10.1007/978-3-319-06880-0\_17}.

\bibitem[Gelada et~al.(2019)Gelada, Kumar, Buckman, Nachum, and
  Bellemare]{DBLP:conf/icml/GeladaKBNB19}
Carles Gelada, Saurabh Kumar, Jacob Buckman, Ofir Nachum, and Marc~G.
  Bellemare.
\newblock Deepmdp: Learning continuous latent space models for representation
  learning.
\newblock In Kamalika Chaudhuri and Ruslan Salakhutdinov (eds.),
  \emph{Proceedings of the 36th International Conference on Machine Learning,
  {ICML} 2019, 9-15 June 2019, Long Beach, California, {USA}}, volume~97 of
  \emph{Proceedings of Machine Learning Research}, pp.\  2170--2179. {PMLR},
  2019.
\newblock URL \url{http://proceedings.mlr.press/v97/gelada19a.html}.

\bibitem[Germain et~al.(2015)Germain, Gregor, Murray, and
  Larochelle]{DBLP:conf/icml/GermainGML15}
Mathieu Germain, Karol Gregor, Iain Murray, and Hugo Larochelle.
\newblock {MADE:} masked autoencoder for distribution estimation.
\newblock In Francis~R. Bach and David~M. Blei (eds.), \emph{Proceedings of the
  32nd International Conference on Machine Learning, {ICML} 2015, Lille,
  France, 6-11 July 2015}, volume~37 of \emph{{JMLR} Workshop and Conference
  Proceedings}, pp.\  881--889. JMLR.org, 2015.
\newblock URL \url{http://proceedings.mlr.press/v37/germain15.html}.

\bibitem[Givan et~al.(2003)Givan, Dean, and Greig]{DBLP:journals/ai/GivanDG03}
Robert Givan, Thomas~L. Dean, and Matthew Greig.
\newblock Equivalence notions and model minimization in markov decision
  processes.
\newblock \emph{Artif. Intell.}, 147\penalty0 (1-2):\penalty0 163--223, 2003.
\newblock \doi{10.1016/S0004-3702(02)00376-4}.
\newblock URL \url{https://doi.org/10.1016/S0004-3702(02)00376-4}.

\bibitem[Gulrajani et~al.(2017)Gulrajani, Ahmed, Arjovsky, Dumoulin, and
  Courville]{DBLP:conf/nips/GulrajaniAADC17}
Ishaan Gulrajani, Faruk Ahmed, Mart{\'{\i}}n Arjovsky, Vincent Dumoulin, and
  Aaron~C. Courville.
\newblock Improved training of wasserstein gans.
\newblock In Isabelle Guyon, Ulrike von Luxburg, Samy Bengio, Hanna~M. Wallach,
  Rob Fergus, S.~V.~N. Vishwanathan, and Roman Garnett (eds.), \emph{Advances
  in Neural Information Processing Systems 30: Annual Conference on Neural
  Information Processing Systems 2017, December 4-9, 2017, Long Beach, CA,
  {USA}}, pp.\  5767--5777, 2017.
\newblock URL
  \url{https://proceedings.neurips.cc/paper/2017/hash/892c3b1c6dccd52936e27cbd0ff683d6-Abstract.html}.

\bibitem[Haarnoja et~al.(2018)Haarnoja, Zhou, Abbeel, and
  Levine]{DBLP:conf/icml/HaarnojaZAL18}
Tuomas Haarnoja, Aurick Zhou, Pieter Abbeel, and Sergey Levine.
\newblock Soft actor-critic: Off-policy maximum entropy deep reinforcement
  learning with a stochastic actor.
\newblock In Jennifer~G. Dy and Andreas Krause (eds.), \emph{Proceedings of the
  35th International Conference on Machine Learning, {ICML} 2018,
  Stockholmsm{\"{a}}ssan, Stockholm, Sweden, July 10-15, 2018}, volume~80 of
  \emph{Proceedings of Machine Learning Research}, pp.\  1856--1865. {PMLR},
  2018.
\newblock URL \url{http://proceedings.mlr.press/v80/haarnoja18b.html}.

\bibitem[Hasanbeig et~al.(2021)Hasanbeig, Jeppu, Abate, Melham, and
  Kroening]{DBLP:conf/aaai/HasanbeigJAMK21}
Mohammadhosein Hasanbeig, Natasha~Yogananda Jeppu, Alessandro Abate, Tom
  Melham, and Daniel Kroening.
\newblock Deepsynth: Automata synthesis for automatic task segmentation in deep
  reinforcement learning.
\newblock In \emph{Thirty-Fifth {AAAI} Conference on Artificial Intelligence,
  {AAAI} 2021, Thirty-Third Conference on Innovative Applications of Artificial
  Intelligence, {IAAI} 2021, The Eleventh Symposium on Educational Advances in
  Artificial Intelligence, {EAAI} 2021, Virtual Event, February 2-9, 2021},
  pp.\  7647--7656. {AAAI} Press, 2021.
\newblock URL \url{https://ojs.aaai.org/index.php/AAAI/article/view/16935}.

\bibitem[Hensel et~al.(2021)Hensel, Junges, Katoen, Quatmann, and
  Volk]{stt:Hensel2021}
Christian Hensel, Sebastian Junges, Joost-Pieter Katoen, Tim Quatmann, and
  Matthias Volk.
\newblock The probabilistic model checker storm.
\newblock \emph{International Journal on Software Tools for Technology
  Transfer}, 2021.
\newblock ISSN 1433-2787.
\newblock \doi{10.1007/s10009-021-00633-z}.
\newblock URL \url{https://doi.org/10.1007/s10009-021-00633-z}.

\bibitem[Hoffman et~al.(2013)Hoffman, Blei, Wang, and
  Paisley]{DBLP:journals/jmlr/HoffmanBWP13}
Matthew~D. Hoffman, David~M. Blei, Chong Wang, and John~W. Paisley.
\newblock Stochastic variational inference.
\newblock \emph{J. Mach. Learn. Res.}, 14\penalty0 (1):\penalty0 1303--1347,
  2013.
\newblock URL \url{http://dl.acm.org/citation.cfm?id=2502622}.

\bibitem[Huang(2020)]{DBLP:conf/nips/Huang20}
Bojun Huang.
\newblock Steady state analysis of episodic reinforcement learning.
\newblock In Hugo Larochelle, Marc'Aurelio Ranzato, Raia Hadsell,
  Maria{-}Florina Balcan, and Hsuan{-}Tien Lin (eds.), \emph{Advances in Neural
  Information Processing Systems 33: Annual Conference on Neural Information
  Processing Systems 2020, NeurIPS 2020, December 6-12, 2020, virtual}, 2020.
\newblock URL
  \url{https://proceedings.neurips.cc/paper/2020/hash/69bfa2aa2b7b139ff581a806abf0a886-Abstract.html}.

\bibitem[Jang et~al.(2017)Jang, Gu, and Poole]{DBLP:conf/iclr/JangGP17}
Eric Jang, Shixiang Gu, and Ben Poole.
\newblock Categorical reparameterization with gumbel-softmax.
\newblock In \emph{5th International Conference on Learning Representations,
  {ICLR} 2017, Toulon, France, April 24-26, 2017, Conference Track
  Proceedings}. OpenReview.net, 2017.
\newblock URL \url{https://openreview.net/forum?id=rkE3y85ee}.

\bibitem[Jansen et~al.(2020)Jansen, K{\"o}nighofer, Junges, Serban, and
  Bloem]{jansen_et_al:LIPIcs:2020:12815}
Nils Jansen, Bettina K{\"o}nighofer, Sebastian Junges, Alex Serban, and
  Roderick Bloem.
\newblock {Safe Reinforcement Learning Using Probabilistic Shields (Invited
  Paper)}.
\newblock In Igor Konnov and Laura Kov{\'a}cs (eds.), \emph{31st International
  Conference on Concurrency Theory (CONCUR 2020)}, volume 171 of \emph{Leibniz
  International Proceedings in Informatics (LIPIcs)}, pp.\  3:1--3:16,
  Dagstuhl, Germany, 2020. Schloss Dagstuhl--Leibniz-Zentrum f{\"u}r
  Informatik.
\newblock ISBN 978-3-95977-160-3.
\newblock \doi{10.4230/LIPIcs.CONCUR.2020.3}.
\newblock URL \url{https://drops.dagstuhl.de/opus/volltexte/2020/12815}.

\bibitem[Junges et~al.(2016)Junges, Jansen, Dehnert, Topcu, and
  Katoen]{DBLP:conf/tacas/Junges0DTK16}
Sebastian Junges, Nils Jansen, Christian Dehnert, Ufuk Topcu, and
  Joost{-}Pieter Katoen.
\newblock Safety-constrained reinforcement learning for mdps.
\newblock In Marsha Chechik and Jean{-}Fran{\c{c}}ois Raskin (eds.),
  \emph{Tools and Algorithms for the Construction and Analysis of Systems -
  22nd International Conference, {TACAS} 2016, Eindhoven, The Netherlands,
  April 2-8, 2016, Proceedings}, volume 9636 of \emph{LNCS}, pp.\  130--146.
  Springer, 2016.
\newblock \doi{10.1007/978-3-662-49674-9\_8}.
\newblock URL \url{https://doi.org/10.1007/978-3-662-49674-9\_8}.

\bibitem[Kingma \& Welling(2014)Kingma and
  Welling]{DBLP:journals/corr/KingmaW13}
Diederik~P. Kingma and Max Welling.
\newblock Auto-encoding variational bayes.
\newblock In Yoshua Bengio and Yann LeCun (eds.), \emph{2nd International
  Conference on Learning Representations, {ICLR} 2014, Banff, AB, Canada, April
  14-16, 2014, Conference Track Proceedings}, 2014.
\newblock URL \url{http://arxiv.org/abs/1312.6114}.

\bibitem[Kwiatkowska et~al.(2022)Kwiatkowska, Norman, and
  Parker]{doi:10.1146/annurev-control-042820-010947}
Marta Kwiatkowska, Gethin Norman, and David Parker.
\newblock Probabilistic model checking and autonomy.
\newblock \emph{Annual Review of Control, Robotics, and Autonomous Systems},
  5\penalty0 (1):\penalty0 385--410, 2022.
\newblock \doi{10.1146/annurev-control-042820-010947}.
\newblock URL \url{https://doi.org/10.1146/annurev-control-042820-010947}.

\bibitem[Larsen \& Skou(1989)Larsen and Skou]{DBLP:conf/popl/LarsenS89}
Kim~Guldstrand Larsen and Arne Skou.
\newblock Bisimulation through probabilistic testing.
\newblock In \emph{Conference Record of the Sixteenth Annual {ACM} Symposium on
  Principles of Programming Languages, Austin, Texas, USA, January 11-13,
  1989}, pp.\  344--352. {ACM} Press, 1989.
\newblock \doi{10.1145/75277.75307}.
\newblock URL \url{https://doi.org/10.1145/75277.75307}.

\bibitem[Libin et~al.(2020)Libin, Moonens, Verstraeten, Perez{-}Sanjines, Hens,
  Lemey, and Now{\'{e}}]{DBLP:conf/pkdd/LibinMVPHLN20}
Pieter J.~K. Libin, Arno Moonens, Timothy Verstraeten, Fabian Perez{-}Sanjines,
  Niel Hens, Philippe Lemey, and Ann Now{\'{e}}.
\newblock Deep reinforcement learning for large-scale epidemic control.
\newblock In Yuxiao Dong, Georgiana Ifrim, Dunja Mladenic, Craig Saunders, and
  Sofie~Van Hoecke (eds.), \emph{Machine Learning and Knowledge Discovery in
  Databases. Applied Data Science and Demo Track - European Conference, {ECML}
  {PKDD} 2020, Ghent, Belgium, September 14-18, 2020, Proceedings, Part {V}},
  volume 12461 of \emph{Lecture Notes in Computer Science}, pp.\  155--170.
  Springer, 2020.
\newblock \doi{10.1007/978-3-030-67670-4\_10}.
\newblock URL \url{https://doi.org/10.1007/978-3-030-67670-4\_10}.

\bibitem[Littman et~al.(2017)Littman, Topcu, Fu, Jr., Wen, and
  MacGlashan]{DBLP:journals/corr/LittmanTFIWM17}
Michael~L. Littman, Ufuk Topcu, Jie Fu, Charles Lee~Isbell Jr., Min Wen, and
  James MacGlashan.
\newblock Environment-independent task specifications via {GLTL}.
\newblock \emph{CoRR}, abs/1704.04341, 2017.
\newblock URL \url{http://arxiv.org/abs/1704.04341}.

\bibitem[Maddison et~al.(2017)Maddison, Mnih, and
  Teh]{DBLP:conf/iclr/MaddisonMT17}
Chris~J. Maddison, Andriy Mnih, and Yee~Whye Teh.
\newblock The concrete distribution: {A} continuous relaxation of discrete
  random variables.
\newblock In \emph{5th International Conference on Learning Representations,
  {ICLR} 2017, Toulon, France, April 24-26, 2017, Conference Track
  Proceedings}. OpenReview.net, 2017.
\newblock URL \url{https://openreview.net/forum?id=S1jE5L5gl}.

\bibitem[Mnih et~al.(2015)Mnih, Kavukcuoglu, Silver, Rusu, Veness, Bellemare,
  Graves, Riedmiller, Fidjeland, Ostrovski, Petersen, Beattie, Sadik,
  Antonoglou, King, Kumaran, Wierstra, Legg, and
  Hassabis]{DBLP:journals/nature/MnihKSRVBGRFOPB15}
Volodymyr Mnih, Koray Kavukcuoglu, David Silver, Andrei~A. Rusu, Joel Veness,
  Marc~G. Bellemare, Alex Graves, Martin~A. Riedmiller, Andreas Fidjeland,
  Georg Ostrovski, Stig Petersen, Charles Beattie, Amir Sadik, Ioannis
  Antonoglou, Helen King, Dharshan Kumaran, Daan Wierstra, Shane Legg, and
  Demis Hassabis.
\newblock Human-level control through deep reinforcement learning.
\newblock \emph{Nat.}, 518\penalty0 (7540):\penalty0 529--533, 2015.
\newblock \doi{10.1038/nature14236}.
\newblock URL \url{https://doi.org/10.1038/nature14236}.

\bibitem[Nowe(1994)]{phdthesis:Nowe94}
Ann Nowe.
\newblock \emph{Synthesis of ``safe'' fuzzy controllers based on reinforcement
  learning}.
\newblock PhD thesis, Vrije Universiteit Brussel, 1994.

\bibitem[Papamakarios et~al.(2017)Papamakarios, Murray, and
  Pavlakou]{DBLP:conf/nips/PapamakariosMP17}
George Papamakarios, Iain Murray, and Theo Pavlakou.
\newblock Masked autoregressive flow for density estimation.
\newblock In Isabelle Guyon, Ulrike von Luxburg, Samy Bengio, Hanna~M. Wallach,
  Rob Fergus, S.~V.~N. Vishwanathan, and Roman Garnett (eds.), \emph{Advances
  in Neural Information Processing Systems 30: Annual Conference on Neural
  Information Processing Systems 2017, December 4-9, 2017, Long Beach, CA,
  {USA}}, pp.\  2338--2347, 2017.
\newblock URL
  \url{https://proceedings.neurips.cc/paper/2017/hash/6c1da886822c67822bcf3679d04369fa-Abstract.html}.

\bibitem[Pnueli(1977)]{DBLP:conf/focs/Pnueli77}
Amir Pnueli.
\newblock The temporal logic of programs.
\newblock In \emph{18th Annual Symposium on Foundations of Computer Science,
  Providence, Rhode Island, USA, 31 October - 1 November 1977}, pp.\  46--57.
  {IEEE} Computer Society, 1977.
\newblock \doi{10.1109/SFCS.1977.32}.
\newblock URL \url{https://doi.org/10.1109/SFCS.1977.32}.

\bibitem[Puterman(1994)]{DBLP:books/wi/Puterman94}
Martin~L. Puterman.
\newblock \emph{Markov Decision Processes: Discrete Stochastic Dynamic
  Programming}.
\newblock Wiley Series in Probability and Statistics. Wiley, 1994.
\newblock ISBN 978-0-47161977-2.
\newblock \doi{10.1002/9780470316887}.
\newblock URL \url{https://doi.org/10.1002/9780470316887}.

\bibitem[Ren et~al.(2021)Ren, Niu, Cui, Ouyang, and Liu]{REN2021103049}
Tao Ren, Jianwei Niu, Jiahe Cui, Zhenchao Ouyang, and Xuefeng Liu.
\newblock An application of multi-objective reinforcement learning for
  efficient model-free control of canals deployed with iot networks.
\newblock \emph{Journal of Network and Computer Applications}, 182:\penalty0
  103049, 2021.
\newblock ISSN 1084-8045.
\newblock \doi{https://doi.org/10.1016/j.jnca.2021.103049}.
\newblock URL
  \url{https://www.sciencedirect.com/science/article/pii/S1084804521000734}.

\bibitem[Sickert et~al.(2016)Sickert, Esparza, Jaax, and
  Kret{\'{\i}}nsk{\'{y}}]{DBLP:conf/cav/SickertEJK16}
Salomon Sickert, Javier Esparza, Stefan Jaax, and Jan Kret{\'{\i}}nsk{\'{y}}.
\newblock Limit-deterministic b{\"{u}}chi automata for linear temporal logic.
\newblock In Swarat Chaudhuri and Azadeh Farzan (eds.), \emph{Computer Aided
  Verification - 28th International Conference, {CAV} 2016, Toronto, ON,
  Canada, July 17-23, 2016, Proceedings, Part {II}}, volume 9780 of
  \emph{Lecture Notes in Computer Science}, pp.\  312--332. Springer, 2016.
\newblock \doi{10.1007/978-3-319-41540-6\_17}.
\newblock URL \url{https://doi.org/10.1007/978-3-319-41540-6\_17}.

\bibitem[Sim{\~{a}}o et~al.(2021)Sim{\~{a}}o, Jansen, and
  Spaan]{DBLP:conf/atal/SimaoJS21}
Thiago~D. Sim{\~{a}}o, Nils Jansen, and Matthijs T.~J. Spaan.
\newblock Alwayssafe: Reinforcement learning without safety constraint
  violations during training.
\newblock In Frank Dignum, Alessio Lomuscio, Ulle Endriss, and Ann Now{\'{e}}
  (eds.), \emph{{AAMAS} '21: 20th International Conference on Autonomous Agents
  and Multiagent Systems, Virtual Event, United Kingdom, May 3-7, 2021}, pp.\
  1226--1235. {ACM}, 2021.
\newblock URL \url{https://dl.acm.org/doi/10.5555/3463952.3464094}.

\bibitem[Todorov et~al.(2012)Todorov, Erez, and Tassa]{todorov2012mujoco}
Emanuel Todorov, Tom Erez, and Yuval Tassa.
\newblock Mujoco: A physics engine for model-based control.
\newblock In \emph{2012 IEEE/RSJ International Conference on Intelligent Robots
  and Systems}, pp.\  5026--5033. IEEE, 2012.

\bibitem[Tolstikhin et~al.(2018)Tolstikhin, Bousquet, Gelly, and
  Sch{\"{o}}lkopf]{DBLP:conf/iclr/TolstikhinBGS18}
Ilya~O. Tolstikhin, Olivier Bousquet, Sylvain Gelly, and Bernhard
  Sch{\"{o}}lkopf.
\newblock Wasserstein auto-encoders.
\newblock In \emph{6th International Conference on Learning Representations,
  {ICLR} 2018, Vancouver, BC, Canada, April 30 - May 3, 2018, Conference Track
  Proceedings}. OpenReview.net, 2018.
\newblock URL \url{https://openreview.net/forum?id=HkL7n1-0b}.

\bibitem[Tsitsiklis(1994)]{DBLP:journals/ml/Tsitsiklis94}
John~N. Tsitsiklis.
\newblock Asynchronous stochastic approximation and q-learning.
\newblock \emph{Mach. Learn.}, 16\penalty0 (3):\penalty0 185--202, 1994.
\newblock \doi{10.1007/BF00993306}.
\newblock URL \url{https://doi.org/10.1007/BF00993306}.

\bibitem[van~den Oord et~al.(2017)van~den Oord, Vinyals, and
  Kavukcuoglu]{DBLP:conf/nips/OordVK17}
A{\"{a}}ron van~den Oord, Oriol Vinyals, and Koray Kavukcuoglu.
\newblock Neural discrete representation learning.
\newblock In Isabelle Guyon, Ulrike von Luxburg, Samy Bengio, Hanna~M. Wallach,
  Rob Fergus, S.~V.~N. Vishwanathan, and Roman Garnett (eds.), \emph{Advances
  in Neural Information Processing Systems 30: Annual Conference on Neural
  Information Processing Systems 2017, 4-9 December 2017, Long Beach, CA,
  {USA}}, pp.\  6306--6315, 2017.
\newblock URL
  \url{http://papers.nips.cc/paper/7210-neural-discrete-representation-learning}.

\bibitem[Villani(2009)]{Villani2009}
C{\'e}dric Villani.
\newblock \emph{Optimal Transport: Old and New}.
\newblock Springer Berlin Heidelberg, Berlin, Heidelberg, 2009.
\newblock ISBN 978-3-540-71050-9.
\newblock \doi{10.1007/978-3-540-71050-9_6}.
\newblock URL \url{https://doi.org/10.1007/978-3-540-71050-9_6}.

\bibitem[Wells et~al.(2020)Wells, Lahijanian, Kavraki, and
  Vardi]{DBLP:journals/corr/abs-2009-10883}
Andrew~M. Wells, Morteza Lahijanian, Lydia~E. Kavraki, and Moshe~Y. Vardi.
\newblock Ltlf synthesis on probabilistic systems.
\newblock In Jean{-}Fran{\c{c}}ois Raskin and Davide Bresolin (eds.),
  \emph{Proceedings 11th International Symposium on Games, Automata, Logics,
  and Formal Verification, GandALF 2020, Brussels, Belgium, September 21-22,
  2020}, volume 326 of \emph{{EPTCS}}, pp.\  166--181, 2020.
\newblock \doi{10.4204/EPTCS.326.11}.
\newblock URL \url{https://doi.org/10.4204/EPTCS.326.11}.

\bibitem[Zang et~al.(2022)Zang, Li, and
  Wang]{DBLP:journals/corr/abs-2112-15303}
Hongyu Zang, Xin Li, and Mingzhong Wang.
\newblock Simsr: Simple distance-based state representations for deep
  reinforcement learning.
\newblock \emph{Proceedings of the AAAI Conference on Artificial Intelligence},
  36\penalty0 (8):\penalty0 8997--9005, Jun. 2022.
\newblock \doi{10.1609/aaai.v36i8.20883}.
\newblock URL \url{https://ojs.aaai.org/index.php/AAAI/article/view/20883}.

\bibitem[Zhang et~al.(2021)Zhang, McAllister, Calandra, Gal, and
  Levine]{DBLP:conf/iclr/0001MCGL21}
Amy Zhang, Rowan~Thomas McAllister, Roberto Calandra, Yarin Gal, and Sergey
  Levine.
\newblock Learning invariant representations for reinforcement learning without
  reconstruction.
\newblock In \emph{9th International Conference on Learning Representations,
  {ICLR} 2021, Virtual Event, Austria, May 3-7, 2021}. OpenReview.net, 2021.
\newblock URL \url{https://openreview.net/forum?id=-2FCwDKRREu}.

\bibitem[Zhang et~al.(2019)Zhang, Gao, Jiao, Liu, Wang, and
  Yang]{DBLP:journals/corr/abs-1902-09323/zhang19}
Shunkang Zhang, Yuan Gao, Yuling Jiao, Jin Liu, Yang Wang, and Can Yang.
\newblock Wasserstein-wasserstein auto-encoders.
\newblock \emph{CoRR}, abs/1902.09323, 2019.
\newblock URL \url{http://arxiv.org/abs/1902.09323}.

\end{thebibliography}
\bibliographystyle{iclr2023_conference}

\newpage
\appendix
\section*{Appendix}
\section{Theoretical Details on WAE-MDPs} \label{appendix:wae-mdp}
\subsection{The Discrepancy Measure}\label{appendix:discrepancy-measure}
We show that reasoning about discrepancy measures between stationary distributions is sound in the context of infinite interaction and episodic RL processes.
Let $\decoder$ be a parameterized behavioral model that generate finite traces from the original environment (i.e., finite sequences of state, actions, and rewards of the form $\tuple{\seq{\state}{T}, \seq{\action}{T - 1}, \seq{\reward}{T - 1}}$), our goal is to find the best parameter $\decoderparameter$ which offers the most accurate reconstruction of the original traces issued from the original model $\mdp$ operating under $\policy$. 
We demonstrate that, in the limit, considering the OT between trace-based distributions is equivalent to considering the OT between the stationary distribution of $\mdp_{\policy}$ and the one of the behavioral model.
Let us first formally recall the definition of the metric on the \emph{transitions} of the MDP.

\smallparagraph{Raw transition distance.}~%
Assume that $\states$, $\actions$, and $\images{\rewards}$ are respectively equipped with metric $\distance_{\states}$, $\distance_{\actions}$, and $\distance_{\rewards}$,
let us define the \emph{raw transition distance metric} over \emph{transitions} of $\mdp$, i.e., tuples of the form $\tuple{\state, \action, \reward, \state'}$, as $\transitiondistance \colon \states \times \actions \times \images{\rewards} \times \states$,
\begin{equation}
    \tracedistance\fun{\tuple{\state_1, \action_1, \reward_1, \state'_1}, \tuple{\state_2, \action_2, \reward_2, \state'_2}} = \distance_\states\fun{\state_1, \state_2} + \distance_{\actions}\fun{\action_1, \action_2} + \distance_{\rewards}\fun{\reward_1, \reward_2} + \distance_{\states}\fun{\state_1', \state'_2}.\notag
\end{equation}
In a nutshell, $\transitiondistance$ consists of the sum of the distance of all the transition components.
Note that it is a well defined distance metric since the sum of distances preserves the identity of indiscernible, symmetry, and triangle inequality. 

\smallparagraph{Trace-based distributions.}~
The raw distance $\tracedistance$ allows to reason about \emph{transitions}, we thus consider the distribution over \emph{transitions which occur along traces of length $T$} to compare the dynamics of the original and behavioral models:
\begin{align*}
    \mathcal{D}_\policy\left[ T \right]\fun{\state, \action, \reward, \state'} &= \frac{1}{T} \sum_{t = 1}^{T} \stationary{\policy}^t\fun{\state \mid \sinit} \cdot \policy\fun{\action \mid \state} \cdot \probtransitions\fun{\state' \mid \state, \action} \cdot \condition{\reward=\rewards\fun{\state, \action}}, \text{ and} \\
    \mathcal{P}_\decoderparameter[T]\fun{\state, \action, \reward, \state'} &= \frac{1}{T} \sum_{t = 1}^{T} \expectedsymbol{\seq{\state}{t}, \seq{\action}{t - 1}, \seq{\reward}{t - 1} \sim \decoder[t]}{\condition{\tuple{\indexof{\state}{t - 1}, \indexof{\action}{t - 1}\indexof{\reward}{t - 1}, \indexof{\state}{t}}= \tuple{\state, \action, \reward, \state'}}},
\end{align*} 
where $\decoder[T]$ denotes the distribution over traces of length $T$, generated from $\decoder$.
Intuitively, $\nicefrac{1}{T} \cdot \sum_{t = 1}^{T} \stationary{\policy}^t\fun{\state \mid \sinit}$
can be seen as the fraction of the time spent in $\state$
along traces of length $T$, starting from the initial state \citeAR{10.5555/280952}.
Therefore, drawing $\tuple{\state, \action, \reward, \state'} \sim \mathcal{D}_\policy\left[ T \right]$ trivially follows: 
it is equivalent to drawing $\state$ from $\nicefrac{1}{T} \cdot \sum_{t = 1}^{T} \stationary{\policy}^t\fun{\cdot \mid \sinit}$, then respectively $\action$ and $\state'$ from $\policy\fun{\cdot \mid \state}$ and $\probtransitions\fun{\cdot \mid \state, \action}$, to finally obtain $\reward = \rewards\fun{\state, \action}$. 
Given $T \in \N$, our objective is to minimize the Wasserstein distance between those distributions:
$\wassersteindist{\tracedistance}{\mathcal{D}_{\policy}[T]}{\mathcal{P}_{\decoderparameter}[T]}$.
%
The following Lemma enables optimizing the Wasserstein distance between the original MDP and the behavioral model 
when traces are drawn from episodic RL processes or infinite interactions \citep{DBLP:conf/nips/Huang20}.

\begin{lemma}\label{lemma:wasserstein-transition-limit}
Assume the existence of a stationary behavioral model $\stationarydecoder = \lim_{T \to \infty} \mathcal{P}_{\decoderparameter}[T]$, then
\begin{equation*}
    \lim_{T \to \infty} \wassersteindist{\tracedistance}{\mathcal{D}_{\policy}[T]}{\mathcal{P}_{\decoderparameter}[T]} = \wassersteindist{\tracedistance}{\stationary{\policy}}{\stationarydecoder}.
\end{equation*}
\end{lemma}
\begin{proof}
First, note that $\nicefrac{1}{T} \cdot \sum_{t = 1}^T \stationary{\policy}^t\fun{\cdot \mid \sinit}$ weakly converges to $ \stationary{\policy}$ as $T$  goes to $\infty$ \citeAR{10.5555/280952}. The result follows then from \citep[Corollary~6.9]{Villani2009}.
\end{proof}

\subsection{Dealing with Discrete Actions}\label{appendix:discrete-action-space}
When the policy $\policy$ executed in $\mdp$ already produces discrete actions, learning a latent action space is, in many cases, not necessary.
We thus make the following assumptions:
%
%
\begin{assumption} \label{assumption:action-encoder}
Let $\policy \colon \states \to \distributions{\actions^{\star}}$ be the policy executed in $\mdp$ and assume that $\actions^{\star}$ is a (tractable) finite set.
Then, we take $\latentactions=\actions^{\star}$ and $\actionencoder$ as the 
identity function,
i.e., $\actionencoder \colon \latentstates \times \actions^{\star} \to \actions^{\star}, \, \tuple{\latentstate, \action^{\star}} \mapsto \action^{\star}$.
\end{assumption}
\begin{assumption} \label{assumption:action-decoder}
Assume that the action space of the original environment $\mdp$ is a (tractable) finite set.
Then, we take $\embeda_{\decoderparameter}$ as the identity function, i.e., $\embeda_{\decoderparameter} = \actionencoder$.
\end{assumption}
Concretely, the premise of Assumption~\ref{assumption:action-encoder} typically occurs when $\policy$ is a latent policy (see Rem.~\ref{rmk:latent-policy-execution}) \emph{or} when $\mdp$ has already a discrete action space.
In the latter case, Assumption~\ref{assumption:action-encoder} and~\ref{assumption:action-decoder} amount to setting $\latentactions = \actions$ and ignoring the action encoder and embedding function.
Note that if a discrete action space is too large, or if the user explicitly aims for a coarser space, then the former is not considered as tractable, these assumptions do not hold, and the action space is abstracted to a smaller set of discrete actions.

\subsection{Proof of Lemma~\ref{lem:regularizer-upper-bound}}
\smallparagraph{Notation.}~%
From now on, we write $\embed_{\encoderparameter}\fun{\latentstate, \latentaction \mid \state, \action} = \condition{\embed_\encoderparameter\fun{\state}=\latentstate} \cdot \actionencoder\fun{\latentaction \mid \latentstate, \action}$.


\regularizerlemma*
\begin{proof}
Wasserstein is compliant with the triangular inequality \citep{Villani2009}, which gives us:
\begin{align*}
    \wassersteindist{\transitiondistance}{\encoder}{ \latentstationaryprior}
    \leq \wassersteindist{\transitiondistance}{\encoder}{\originaltolatentstationary{}} + \wassersteindist{\distance_{\latentstates}}{\originaltolatentstationary{}}{ \latentstationaryprior},
\end{align*}
where
\begin{align}
&\wassersteindist{\transitiondistance}{\originaltolatentstationary{}}{ \latentstationaryprior} \tag{note that $\wassersteinsymbol{\tracedistance}$ is reflexive \citep{Villani2009}} \\
=& \sup_{f \in \Lipschf{\transitiondistance}} \expectedsymbol{\state, \action \sim \stationary{\policy}}\expectedsymbol{\latentstate, \latentaction \sim \embed_{\encoderparameter}\fun{\sampledot \mid \state, \action}}\expectedsymbol{\latentstate' \sim \latentprobtransitions_{\decoderparameter}\fun{\sampledot \mid \latentstate, \latentaction}} f\fun{\latentstate, \latentaction, \latentstate'}
- \expectedsymbol{\latentstate \sim \latentstationaryprior}\expectedsymbol{\latentaction \sim \latentpolicy_{\decoderparameter}\fun{\sampledot \mid \latentstate}}\expectedsymbol{\latentstate' \sim \latentprobtransitions_{\decoderparameter}\fun{\sampledot \mid \latentstate, \latentaction}} f\fun{\latentstate, \latentaction, \latentstate'}\text{, and} \notag \\[.5em]
& \wassersteindist{\transitiondistance}{\encoder}{\originaltolatentstationary{}} \notag \\
=&
    \sup_{f \in \Lipschf{\transitiondistance}} \expectedsymbol{\state, \action, \state' \sim \stationary{\policy}}\expectedsymbol{\latentstate, \latentaction, \latentstate' \sim \embed_{\encoderparameter}\fun{\sampledot \mid \state, \action, \state'}} f\fun{\latentstate, \latentaction, \latentstate'} - \expectedsymbol{\state, \action \sim \stationary{\policy}}\expectedsymbol{\latentstate, \latentaction \sim \embed_{\encoderparameter}\fun{\sampledot \mid \state, \action}}\expectedsymbol{\latentstate' \sim \latentprobtransitions_{\decoderparameter}\fun{\sampledot \mid \latentstate, \latentaction}} f\fun{\latentstate, \latentaction, \latentstate'} \label{eq:proof-lemma-wdist-triangular-inequality-1} \\
    \leq& \expectedsymbol{\state, \action \sim \stationary{\policy}}\expectedsymbol{\latentstate, \latentaction \sim \embed_{\encoderparameter}\fun{\sampledot \mid \state, \action}} \; \sup_{f \in \Lipschf{\transitiondistance}} \expectedsymbol{\state' \sim \probtransitions\fun{\sampledot \mid \state, \action}} f\fun{\latentstate, \latentaction, \embed_{\encoderparameter}\fun{\state'}} - \expectedsymbol{\latentstate' \sim \latentprobtransitions_{\decoderparameter}\fun{\sampledot \mid \latentstate, \latentaction}} f\fun{\latentstate, \latentaction, \latentstate'} \label{eq:proof-lemma-wdist-triangular-inequality-2} \\
    =& \expectedsymbol{\state, \action \sim \stationary{\policy}}\expectedsymbol{\latentaction \sim \embed_{\encoderparameter}^{\scriptscriptstyle \actions}\fun{\sampledot \mid \embed_{\encoderparameter}\fun{\state}, \action}} \; \sup_{f \in \Lipschf{\distance_{\latentstates}}} \expectedsymbol{\latentstate' \sim \embed_{\encoderparameter}\probtransitions\fun{\sampledot \mid \state, \action}} f\fun{\latentstate'} - \expectedsymbol{\latentstate' \sim \latentprobtransitions_{\decoderparameter}\fun{\sampledot \mid \embed_{\encoderparameter}\fun{\state}, \latentaction}} f\fun{\latentstate'} \label{eq:proof-lemma-wdist-triangular-inequality-3} \\
    =& \expectedsymbol{\state, \action \sim \stationary{\policy}}\expectedsymbol{\latentaction \sim \embed_{\encoderparameter}^{\scriptscriptstyle \actions}\fun{\sampledot \mid \embed_{\encoderparameter}\fun{\state}, \action}} \wassersteindist{\distance_{\latentstates}}{\embed_{\encoderparameter}\probtransitions\fun{\sampledot \mid \state, \action}}{\latentprobtransitions_{\decoderparameter}\fun{\sampledot \mid \embed_{\encoderparameter}\fun{\state}, \latentaction}}. \notag
\end{align}
We pass from Eq.~\ref{eq:proof-lemma-wdist-triangular-inequality-1} to Eq.~\ref{eq:proof-lemma-wdist-triangular-inequality-2} by the Jensen's inequality.
To see how we pass from Eq.~\ref{eq:proof-lemma-wdist-triangular-inequality-2} to Eq.~\ref{eq:proof-lemma-wdist-triangular-inequality-3}, notice
that
\begin{align*}
    \Lipschf{\transitiondistance} &= \left\{f \colon f\fun{\latentstate_1, \latentaction_1, \latentstate_1'} - f\fun{\latentstate_2, \latentaction_2, \latentstate_2'} \leq \transitiondistance\fun{\tuple{\latentstate_1, \latentaction_1, \latentstate_1'}, \tuple{\latentstate_2, \latentaction_2, \latentstate_2'}}\right\} \\
    \Lipschf{\transitiondistance} &= \set{f \colon f\fun{\latentstate_1, \latentaction_1, \latentstate_1'} - f\fun{\latentstate_2, \latentaction_2, \latentstate_2'} \leq \distance_{\latentstates}\fun{\latentstate_1, \latentstate_2} + \distance_{\latentactions}\fun{\latentaction_1, \latentaction_2} + \distance_{\latentstates}\fun{\latentstate_1', \latentstate_2'}}
\end{align*}
Observe now that $\latentstate$ and $\latentaction$ are fixed in the supremum computation of Eq.~\ref{eq:proof-lemma-wdist-triangular-inequality-2}: all functions $f$ considered and taken from $\Lipschf{\transitiondistance}\,$ are of the form $f\fun{\latentstate, \latentaction, \sampledot}$.
It is thus sufficient to consider the supremum over functions from the following subset of $\Lipschf{\transitiondistance}\,$:
\begin{align*}
    &\set{f \colon f\fun{\latentstate, \latentaction, \latentstate_1'} - f\fun{\latentstate, \latentaction, \latentstate_2'} \leq \distance_{\latentstates}\fun{\latentstate, \latentstate} + \distance_{\latentactions}\fun{\latentaction, \latentaction} +  \distance_{\latentstates}\fun{\latentstate_1', \latentstate_2'}} \\ \tag{for $\latentstate$, $\latentaction$ drawn from $\embed_{\encoderparameter}$} \\
    =&\set{f \colon f\fun{\latentstate, \latentaction, \latentstate_1'} - f\fun{\latentstate, \latentaction, \latentstate_2'} \leq \distance_{\latentstates}\fun{\latentstate_1', \latentstate_2'}}\\
    =&\set{f \colon f\fun{\latentstate_1'} - f\fun{\latentstate_2'} \leq \distance_{\latentstates}\fun{\latentstate_1', \latentstate_2'}} \\
    =& \Lipschf{\distance_{\latentstates}}.
\end{align*}
Given a state $\state \in \states$ in the original model, the (parallel) execution of $\policy$ in $\latentmdp_{\decoderparameter}$ is enabled through $\policy\fun{\action, \latentaction \mid \state} = \policy\fun{\action \mid \state} \cdot \actionencoder\fun{\latentaction \mid \embed_{\encoderparameter}\fun{\state}, \action}$ (cf. Fig.~\ref{subfig:latent-fow-distillation}).
The local transition loss resulting from this interaction is:
\begin{align*}
\localtransitionloss{\stationary{\policy}} &= \expectedsymbol{\state, \tuple{\action, \latentaction} \sim \stationary{\policy}} \wassersteindist{\distance_{\latentstates}}{\embed_{\encoderparameter}\probtransitions\fun{\sampledot \mid \state, \action}}{\latentprobtransitions\fun{\sampledot \mid \embed_{\encoderparameter}\fun{\state}, \latentaction}}\\
&= \expectedsymbol{\state, \action \sim \stationary{\policy}}\expectedsymbol{\latentaction \sim \embed_{\encoderparameter}^{\scriptscriptstyle \actions}\fun{\sampledot \mid \embed_{\encoderparameter}\fun{\state}, \action}} \wassersteindist{\distance_{\latentstates}}{\embed_{\encoderparameter}\probtransitions\fun{\sampledot \mid \state, \action}}{\latentprobtransitions_{\decoderparameter}\fun{\sampledot \mid \embed_{\encoderparameter}\fun{\state}, \latentaction}},
\end{align*}
which finally yields the result.
\end{proof}

\subsection{Proof of Theorem~\ref{thm:latent-execution-objective}}
Before proving Theorem~\ref{thm:latent-execution-objective}, let us introduce the following Lemma, that explicitly demonstrates the link between the transition regularizer of the \waemdp objective and the local transition loss required to obtain the guarantees related to the bisimulation bounds of Eq.~\ref{eq:bidistance-bound}.

\begin{lemma} \label{lem:regularizer-to-local-transition-loss}
Assume that traces are generated by running $\latentpolicy \in \latentpolicies$ in the original environment, then
\begin{align*}
     \expectedsymbol{\state, \action^\star \sim \stationary{\latentpolicy}}\expectedsymbol{\latentaction \sim \embed_{\encoderparameter}^{\actions}\fun{\sampledot \mid \embed_{\encoderparameter}\fun{\state}, \action^{\star}}} \wassersteindist{\distance_{\latentstates}}{\embed_{\encoderparameter}\probtransitions\fun{\sampledot \mid \state, \action^{\star}}}{\latentprobtransitions_{\decoderparameter}\fun{\sampledot \mid \embed_{\encoderparameter}\fun{\state}, \latentaction}} = \localtransitionloss{\stationary{\latentpolicy}}.
\end{align*}
\end{lemma}
\begin{proof}
Since the latent policy $\latentpolicy$ generates latent actions, Assumption~\ref{assumption:action-encoder} holds, which means:
\begin{align*}
     & \expectedsymbol{\state, \action^{\star} \sim \stationary{\latentpolicy}}\expectedsymbol{\latentaction \sim \embed_{\encoderparameter}^{\actions}\fun{\sampledot \mid \embed_{\encoderparameter}\fun{\state}, \action^{\star}}} \wassersteindist{\distance_{\latentstates}}{\embed_{\encoderparameter}\probtransitions\fun{\sampledot \mid \state, \action^{\star}}}{\latentprobtransitions_{\decoderparameter}\fun{\sampledot \mid \embed_{\encoderparameter}\fun{\state}, \latentaction}}\\
     =& \expectedsymbol{\state, \latentaction \sim \stationary{\latentpolicy}} \wassersteindist{\distance_{\latentstates}}{\embed_{\encoderparameter}\probtransitions\fun{\sampledot \mid \state, \latentaction}}{\latentprobtransitions_{\decoderparameter}\fun{\sampledot \mid \embed_{\encoderparameter}\fun{\state}, \latentaction}}\\
     =& \, \localtransitionloss{\stationary{\latentpolicy}}.
\end{align*}
\end{proof}

\latentexecutionobjective*

\begin{proof}
We distinguish two cases: (i) the case where the original and latent models share the same discrete action space, i.e., $\actions = \latentactions$, and (ii) the case where the two have a different action space (e.g., when the original action space is continuous), i.e., $\actions \neq \latentactions$.
In both cases, the local losses term follows by definition of $\localrewardloss{\stationary{\latentpolicy}}$ and Lemma~\ref{lem:regularizer-to-local-transition-loss}.
When $\distance_{\rewards}$ is the Euclidean distance (or even the $L_1$ distance since rewards are scalar values), the expected reward distance occurring in the expected trace-distance term $\tracedistance$ in the \waemdp objective directly translates to the local loss $\localrewardloss{\stationary{\latentpolicy}}$.
Concerning the local transition loss, in case~(i), the result naturally follows from  Assumption~\ref{assumption:action-encoder} and~\ref{assumption:action-decoder}.
In case~(ii), only Assumption~\ref{assumption:action-encoder} holds, meaning the action encoder term of the \waemdp objective is ignored, but not the action embedding term appearing in $G_{\decoderparameter}$.
Given $\state \sim \stationary{\latentpolicy}$, recall that executing $\latentpolicy$ in $\mdp$ amounts to embedding the produced latent actions $\latentaction \sim \latentpolicy\fun{\sampledot \mid \embed_{\encoderparameter}\fun{\state}}$ back to the original environment via $\action = \embeda_{\decoderparameter}\fun{\embed_{\encoderparameter}\fun{\state}, \latentaction}$ (cf. Rem.~\ref{rmk:latent-policy-execution} and Fig.~\ref{subfig:latent-flow-guarantees}).
Therefore, the projection of $\tracedistance\fun{\tuple{\state, \action, \reward, \state'}, G_{\decoderparameter}\fun{\embed_{\encoderparameter}\fun{\state}, \latentaction, \embed_{\encoderparameter}\fun{\state'}}}$ on the action space $\actions$ is $\distance_{\actions}\fun{\embeda_{\decoderparameter}\fun{\embed_{\encoderparameter}\fun{\state}, \latentaction}, \embeda_{\decoderparameter}\fun{\embed_{\encoderparameter}\fun{\state}, \latentaction}} = 0$, for $\reward = \rewards\fun{\state, \action}$ and $\state' \sim \probtransitions\fun{\sampledot \mid \state, \action}$.
\end{proof}

\subsection{Optimizing the Transition Regularizer}\label{appendix:tractable-transition-regularizer}
In the following, we detail how we derive a tractable form of 
our transition regularizer $\localtransitionloss{\stationary{\policy}}\fun{\wassersteinparameter}$.
Optimizing the ground Kantorovich-Rubinstein duality is enabled via the introduction of a parameterized, $1$-Lipschitz network 
$\transitionlossnetwork$, that need to be trained to attain the supremum of the dual:
\begin{align*}
\localtransitionloss{\stationary{\policy}}\fun{\wassersteinparameter} &= \expectedsymbol{\state, \action \sim \stationary{\policy}}\expectedsymbol{\latentstate, \latentaction \sim \embed_{\encoderparameter}\fun{\sampledot \mid \state, \action}} \; \max_{\wassersteinparameter \colon \transitionlossnetwork \in \Lipschf{\distance_{\latentstates}}} \; \expectedsymbol{\latentstate' \sim \embed_{\encoderparameter}\probtransitions\fun{\sampledot \mid \state, \action}} \transitionlossnetwork\fun{\latentstate'} - \expectedsymbol{\latentstate' \sim \latentprobtransitions_{\decoderparameter}\fun{\sampledot \mid \latentstate, \latentaction}} \transitionlossnetwork\fun{\latentstate'}. 
\end{align*}
Under this form, optimizing $\localtransitionloss{\stationary{\policy}}\fun{\wassersteinparameter}$ is intractable due to the expectation over the maximum. 
The following Lemma allows us rewriting $\localtransitionloss{\stationary{\policy}}$ to make the optimization tractable through Monte Carlo estimation.
\begin{lemma}
Let $\measurableset, \varmeasurableset$ be two measurable sets, $\stationary{} \in \distributions{\measurableset}$, $P \colon \measurableset \to \distributions{\varmeasurableset}, Q \colon \measurableset \to \distributions{\varmeasurableset}$, and $\distance \colon \varmeasurableset \times \varmeasurableset \to \mathopen[ 0, + \infty \mathclose[$ be a metric on $\varmeasurableset$. Then,
\begin{align*}
    &\expectedsymbol{x \sim \stationary{}}{\wassersteindist{\distance}{P\fun{\sampledot \mid x}}{Q\fun{\sampledot \mid x}}} = \sup_{\varphi \colon \measurableset \to \Lipschf{\distance}} \expected{x \sim \stationary{}}{\expectedsymbol{y_1 \sim P\fun{\sampledot \mid x}} \varphi\fun{x}\fun{y_1} - \expectedsymbol{y_2 \sim Q\fun{\sampledot \mid x}} \varphi\fun{x}\fun{y_2}}
\end{align*}
\begin{proof}
Our objective is to show that 
\begin{align}
    &\expected{x \sim \stationary{}}{\sup_{f \in \Lipschf{\distance}} \expectedsymbol{y_1 \sim P\fun{\sampledot \mid x}}\varphi\fun{y_1}\fun{x} - \expectedsymbol{y_2 \sim Q\fun{\sampledot \mid x}} \varphi\fun{y_2}\fun{x}} \label{eq:expected-cond-wasserstein}\\
    =& \sup_{\varphi \colon \measurableset \to \Lipschf{\distance}} \expected{x \sim \stationary{}}{\expectedsymbol{y_1 \sim P\fun{\sampledot \mid x}} \varphi\fun{x}\fun{y_1} - \expectedsymbol{y_2 \sim Q\fun{\sampledot \mid x}} \varphi\fun{x}\fun{y_2}} \label{eq:sup-expected-cond-wasserstein}
\end{align}
We start with (\ref{eq:expected-cond-wasserstein})~$\leq$~(\ref{eq:sup-expected-cond-wasserstein}). Construct $\varphi^\star \colon \measurableset \to \Lipschf{\distance}$ by setting for all $x \in \measurableset$
\begin{equation*}
    \varphi^\star\fun{x} = \arg\sup_{f \in \Lipschf{\distance}} \expectedsymbol{y_1 \sim P\fun{\sampledot \mid x}} f\fun{y_1} - \expectedsymbol{y_2 \sim Q\fun{\sampledot \mid x}} f\fun{y_2}.
\end{equation*}
This gives us
\begin{align*}
    &\expected{x \sim \stationary{}}{\sup_{f \in \Lipschf{\distance}} \expectedsymbol{y_1 \sim P\fun{\sampledot \mid x}} f\fun{y_1} - \expectedsymbol{y_2 \sim Q\fun{\sampledot \mid x}} f\fun{y_2}} \\
    =& \expected{x \sim \stationary{}}{\expectedsymbol{y_1 \sim P\fun{\sampledot \mid x}} \varphi^{\star}\fun{x}\fun{y_1} - \expectedsymbol{y_2 \sim Q\fun{\sampledot \mid x}} \varphi^{\star}\fun{x}\fun{y_2}} \\
    \leq& \sup_{\varphi \colon \measurableset \to \Lipschf{\distance}} \expected{x \sim \stationary{}}{\expectedsymbol{y_1 \sim P\fun{\sampledot \mid x}} \varphi\fun{x}\fun{y_1} - \expectedsymbol{y_2 \sim Q\fun{\sampledot \mid x}} \varphi\fun{x}\fun{y_2}}.
\end{align*}
It remains to show that (\ref{eq:expected-cond-wasserstein})~$\geq$~(\ref{eq:sup-expected-cond-wasserstein}).
Take
\[
\varphi^{\star} = \arg\sup_{\varphi \colon \measurableset \to \Lipschf{\distance}} \expected{x \sim \stationary{}}{\expectedsymbol{y_1 \sim P\fun{\sampledot \mid x}} \varphi\fun{x}\fun{y_1} - \expectedsymbol{y_2 \sim Q\fun{\sampledot \mid x}} \varphi\fun{x}\fun{y_2}}.
\]
Then, for all $x \in \measurableset$, we have $\varphi^{\star}\fun{x} \in \Lipschf{\distance}$ which means:
\begin{align*}
    & \expectedsymbol{y_1 \sim P\fun{\sampledot \mid x}} \varphi^{\star}\fun{x}\fun{y_1} - \expectedsymbol{y_2 \sim Q\fun{\sampledot \mid x}} \varphi^{\star}\fun{x}\fun{y_2} \\
    \leq & \sup_{f \in \Lipschf{\distance}} \expectedsymbol{y_1 \sim P\fun{\sampledot \mid x}} f\fun{y_1} - \expectedsymbol{y_2 \sim Q\fun{\sampledot \mid x}} f\fun{y_2}
\end{align*}
This finally yields
\begin{align*}
    &\expected{x \sim \stationary{}}{\expectedsymbol{y_1 \sim P\fun{\sampledot \mid x}} \varphi^{\star}\fun{x}\fun{y_1} - \expectedsymbol{y_2 \sim Q\fun{\sampledot \mid x}} \varphi^{\star}\fun{x}\fun{y_2}} \\
    \leq& \expected{x \sim \stationary{}}{\sup_{f \in \Lipschf{\distance}} \expectedsymbol{y_1 \sim P\fun{\sampledot \mid x}} f\fun{y_1} - \expectedsymbol{y_2 \sim Q\fun{\sampledot \mid x}} f\fun{y_2}}.
\end{align*}
\end{proof}
\end{lemma}
\begin{corollary}\label{cor:min-sup-regularizer}
Let  $\stationary{\policy}$ be a stationary distribution of $\mdp_\policy$ and $\measurableset = \states \times \actions \times \latentstates \times \latentactions$, then
\begin{align*}
\localtransitionloss{\stationary{\policy}} = \sup_{\varphi \colon \measurableset \to \scriptscriptstyle \Lipschf{\distance_{\latentstates}}} \expectedsymbol{\state, \action, \state' \sim \stationary{\policy}} \expected{\latentstate, \latentaction \sim \embed_{\encoderparameter}\fun{\sampledot \mid \state, \action}}{\varphi\fun{\state, \action, \latentstate, \latentaction}\fun{\embed_{\encoderparameter}\fun{\state'}} - \expectedsymbol{\latentstate' \sim \latentprobtransitions_{\decoderparameter}\fun{\sampledot \mid \latentstate, \action}}{\varphi\fun{\state, \action, \latentstate, \latentaction}}\fun{\latentstate'}}
\end{align*}
\end{corollary}
Consequently, we rewrite $\localtransitionloss{\stationary{\policy}}\fun{\wassersteinparameter}$ as a tractable maximization:
\begin{equation*}
    \localtransitionloss{\stationary{\policy}}\fun{\wassersteinparameter} = \max_{\wassersteinparameter \colon \transitionlossnetwork \in \Lipschf{\distance_{\scriptscriptstyle \latentstates}}} \; \expectedsymbol{\state, \action, \state' \sim \stationary{\policy}} \expected{\latentstate, \latentaction \sim \embed_{\encoderparameter}\fun{\sampledot \mid \state, \action}}{\transitionlossnetwork\fun{\state, \action, \latentstate, \latentaction, \embed_{\encoderparameter}\fun{\state'}} - \expectedsymbol{\latentstate' \sim \latentprobtransitions_{\decoderparameter}\fun{\sampledot \mid \latentstate, \latentaction}} \transitionlossnetwork\fun{\state, \action, \latentstate, \latentaction, \latentstate'}}.
\end{equation*}

\subsection{The Latent Metric}\label{appendix:latent-metric}
In the following, we show that considering the Euclidean distance for $\tracedistance$ and $\distance_{\latentstates}$ in the latent space for optimizing  the regularizers $\steadystateregularizer{{\policy}}$ and $\localtransitionloss{\stationary{\policy}}$ is Lipschitz equivalent to considering a continuous $\temperature$-relaxation of the \emph{discrete metric} $\condition{\neq}\fun{\vx, \vy} = \condition{\vx \neq \vy}$.
Consequently, this also means it is consistently sufficient to enforce $1$-Lispchitzness via the gradient penalty approach of \citet{DBLP:conf/nips/GulrajaniAADC17} during training to maintain the guarantees linked to the regularizers in the zero-temperature limit, when the spaces are discrete.

\begin{lemma}
Let $\distance$ be the usual Euclidean distance and $\distance_{\temperature}\colon \mathopen[0, 1 \mathclose]^n \times \mathopen[0, 1 \mathclose]^n \to \mathopen[0, 1\mathclose[, \, \tuple{\vx, \vy} \mapsto \frac{\distance\fun{\vx, \vy}}{\temperature + \distance\fun{\vx, \vy}}$ for $\temperature \in \mathopen]0, 1\mathclose]$ and $n \in \N$, then $\distance_{\temperature}$ is a distance metric.
\end{lemma}
\begin{proof}
The function $\distance_{\temperature}$ is a metric iff it satisfies the following axioms:
\begin{enumerate}
    \item \emph{Identity of indiscernibles}:
    If $\vx=\vy$, then $\distance_{\lambda}\fun{\vx, \vy} = \frac{\distance\fun{\vx, \vy}}{\temperature + \distance\fun{\vx, \vy}} = \frac{0}{\temperature + 0} = 0$ since $\distance$ is a distance metric.
    Assume now that $\distance_\temperature\fun{\vx, \vy} = 0$ and take $\alpha = \distance\fun{\vx, \vy}$, for any $\vx, \vy$. Thus, $\alpha \in \left[0, +\infty\right[$ and $0 = \frac{\alpha}{\temperature + \alpha}$ is only achieved in $\alpha = 0$, which only occurs whenever $\vx = \vy$ since $\distance$ is a distance metric.
    \item \emph{Symmetry}:
    \begin{align*}
        \distance_\temperature\fun{\vx, \vy} &= \frac{\distance\fun{\vx, \vy}}{\temperature + \distance\fun{\vx, \vy}}\\
        &= \frac{\distance\fun{\vy, \vx}}{\temperature + \distance\fun{\vy, \vx}} \tag{$\distance$ is a distance metric} \\
        &= \distance_\temperature\fun{\vy, \vx}
    \end{align*}
    \item \emph{Triangle inequality}: 
    Let $\vx, \vy, \vz \in \mathopen[0, 1\mathclose]^n$, the triangle inequality holds iff
    \begin{align}
        &&\distance_\temperature\fun{\vx, \vy} + \distance_{\temperature}\fun{\vy, \vz} &\geq \distance_{\temperature}\fun{\vx, \vz}& \label{eq:triangle-ineq-d-temperature} \\
        &\equiv& \frac{\distance\fun{\vx, \vy}}{\temperature + \distance\fun{\vx, \vy}} + \frac{\distance\fun{\vy, \vz}}{\temperature + \distance\fun{\vy, \vz}} & \geq \frac{\distance\fun{\vx, \vz}}{\temperature + \distance\fun{\vx, \vz}}& \notag \\
        &\equiv& \frac{\temperature \distance\fun{\vx, \vy} + \temperature \distance\fun{\vy, \vz} + 2\distance\fun{\vx, \vy} \distance\fun{\vy, \vz}}{\temperature^2 + \temperature \distance\fun{\vx, \vy} + \temperature \distance\fun{\vy, \vz} + \distance\fun{\vx, \vy}\distance\fun{\vy, \vz}} & \geq \frac{\distance\fun{\vx, \vz}}{\temperature + \distance\fun{\vx, \vz}}& \notag
    \end{align}
    \begin{align}
        &\equiv& &\temperature^2 \distance\fun{\vx, \vy} + \temperature^2 \distance\fun{\vy, \vz} + 2 \temperature \distance\fun{\vx, \vy}\distance\fun{\vy, \vz} + \notag \\
        &&&\temperature\distance\fun{\vx, \vy} \distance\fun{\vx, \vz} + \temperature\distance\fun{\vy, \vz} \distance\fun{\vx, \vz} + 2 \distance\fun{\vx, \vy} \distance\fun{\vy, \vz}\distance\fun{\vx, \vz} & \notag \\
        && \geq& \temperature^2 \distance\fun{\vx, \vz} + \temperature\distance\fun{\vx, \vy} \distance\fun{\vx, \vz} + \temperature\distance\fun{\vy, \vz} \distance\fun{\vx, \vz} + \distance\fun{\vx, \vy} \distance\fun{\vy, \vz} \distance\fun{\vx, \vz} \tag{cross-product, with $\temperature > 0$ and $\images{\distance} \in \mathopen[0, \infty\mathclose[$} \\
        &\equiv& &\temperature^2 \distance\fun{\vx, \vy} + \temperature^2 \distance\fun{\vy, \vz} + 2 \temperature \distance\fun{\vx, \vy}\distance\fun{\vy, \vz} + 
        \distance\fun{\vx, \vy} \distance\fun{\vy, \vz}\distance\fun{\vx, \vz}
        \geq \temperature^2 \distance\fun{\vx, \vz} \label{eq:triangle-ineq-relaxation}
    \end{align}
    Since $\distance$ is a distance metric, we have 
    \begin{equation}
        \temperature^2 \distance\fun{\vx, \vy} + 
        \temperature^2 \distance\fun{\vy, \vz} \geq 
        \temperature^2 \distance\fun{\vx, \vz} \label{eq:triangle-ineq-scaled}
    \end{equation}
    and $\images{\distance}\in \mathopen[0, \infty\mathclose[$, meaning
    \begin{equation}
        2 \temperature \distance\fun{\vx, \vy}\distance\fun{\vy, \vz} + 
        \distance\fun{\vx, \vy} \distance\fun{\vy, \vz}\distance\fun{\vx, \vz} \geq 0 \label{eq:triangle-ineq-upper-bound}
    \end{equation}
\end{enumerate}
By Eq.~\ref{eq:triangle-ineq-scaled} and~\ref{eq:triangle-ineq-upper-bound}, the inequality of Eq.~\ref{eq:triangle-ineq-relaxation} holds.
Furthermore, the fact that Eq.~\ref{eq:triangle-ineq-d-temperature} and~\ref{eq:triangle-ineq-relaxation} are equivalent yields the result.
\end{proof}
\begin{lemma}
Let $\distance$, $\distance_{\temperature}$ as defined above, then
(i) $ \distance_{\temperature} \xrightarrow[\temperature \to 0]{} \condition{\neq}$
and (ii) $\distance, \distance_\temperature$ are Lipschitz-equivalent.
\end{lemma}
\begin{proof}
Part (i) is straightforward by definition of $\distance_{\temperature}$.
Distances $\distance$ and $\distance_{\temperature}$ are Lispchitz equivalent if and only if $\exists a, b > 0$ such that $\forall \vx, \vy \in \mathopen[ 0, 1 \mathclose]^n$,
\begin{alignat*}{3}
    &\,& a \cdot \distance\fun{\vx, \vy} &\leq &\distance_\temperature\fun{\vx, \vy} \,\,\,\,\,\, &\leq  b \cdot \distance\fun{\vx, \vy} \\
    & \equiv & a \cdot \distance\fun{\vx, \vy} &\leq &\frac{\distance\fun{\vx, \vy}}{\temperature + \distance\fun{\vx, \vy}} &\leq  b \cdot \distance\fun{\vx, \vy} \\
    & \equiv & a & \leq & \frac{1}{\temperature + \distance\fun{\vx, \vy}} &\leq b
\end{alignat*}
Taking $a = \frac{1}{\temperature + \sqrt{n}}$ and $b = \frac{1}{\temperature}$ yields the result.
\end{proof}
\begin{corollary}
For all $\beta \geq \nicefrac{1}{\temperature}$, $\state \in \states$, $\action \in \actions$, $\latentstate \in \latentstates$, and $\latentaction \in \latentactions$, we have
\begin{enumerate}
    \item $\wassersteindist{\distance_{\temperature}}{\originaltolatentstationary{}}{\latentstationaryprior} \leq \beta \cdot \wassersteindist{\distance}{\originaltolatentstationary{}}{\latentstationaryprior}$ \label{enum:steady-state-regularizer-distance}
    \item $\wassersteindist{\distance_{\temperature}}{\embed_{\encoderparameter}\probtransitions\fun{\sampledot\mid \state, \action}}{\latentprobtransitions_{\decoderparameter}\fun{\sampledot \mid \latentstate, \latentaction}} \leq \beta \cdot \wassersteindist{\distance}{\embed_{\encoderparameter}\probtransitions\fun{\sampledot\mid \state, \action}}{\latentprobtransitions_{\decoderparameter}\fun{\sampledot \mid \latentstate, \latentaction}}$\label{enum:transition-regularizer-distance}
\end{enumerate}
\end{corollary}
\begin{proof}
By Lipschitz equivalence, taking $\beta \geq \nicefrac{1}{\temperature}$ ensures that $\forall n \in \N$, $\forall \vx, \vy \in \mathopen[0, 1\mathclose]^n, \, \distance_{\temperature}\fun{\vx, \vy} \leq \beta \cdot \distance\fun{\vx, \vy}$.
Moreover, for any distributions $P, Q$, $\wassersteindist{\distance_{\temperature}}{P}{Q} \leq \beta \cdot \wassersteindist{\distance}{P}{Q}$
(cf., e.g., \citealt[Lemma~A.4]{DBLP:conf/icml/GeladaKBNB19} for details).
\end{proof}
In practice, taking the hyperparameter $\beta \geq \nicefrac{1}{\temperature}$ in the \waemdp ensures that minimizing the $\beta$-scaled regularizers w.r.t$.$ $\distance$ also minimizes the regularizers w.r.t$.$ the $\temperature$-relaxation $\distance_\temperature$, being the discrete distribution in the zero-temperature limit.
Note that optimizing over two different $\beta_1, \beta_2$ instead of a unique scale factor $\beta$ is also a good practice to interpolate between the two regularizers.

\section{Experiment Details} \label{appendix:experiments}
The code for conducting and replicating our experiments is available at \url{https://github.com/florentdelgrange/wae_mdp}.
\subsection{Setup}\label{appendix:setup}
We used \textsc{TensorFlow} \texttt{2.7.0} \citepAR{tensorflow2015-whitepaper} to implement the neural network architecture of our \waemdp, \textsc{TensorFlow Probability} \texttt{0.15.0} \citepAR{dillon2017tensorflow} to handle the probabilistic components of the latent model (e.g., latent distributions with reparameterization tricks, masked autoregressive flows, etc.), as well as \textsc{TF-Agents} \texttt{0.11.0} \citepAR{TFAgents} to handle the RL parts of the framework.

Models have been trained on a cluster running under \texttt{CentOS Linux 7 (Core)} composed of a mix of nodes containing Intel processors with the following CPU microarchitectures:
(i) \texttt{10-core INTEL E5-2680v2}, (ii) \texttt{14-core INTEL E5-2680v4}, and (iii) \texttt{20-core INTEL Xeon Gold 6148}.
We used $8$ cores and $32$ GB of memory for each run.

\subsection{Stationary Distribution}\label{appendix:experiments:stationary-distribution}
To sample from the stationary distribution $\stationary{\policy}$ of episodic learning environments operating under $\policy \in \mpolicies{\mdp}$, we implemented the \emph{recursive $\epsilon$-perturbation trick} of \cite{DBLP:conf/nips/Huang20}.
In a nutshell, the reset of the environment is explicitly added to the state space of $\mdp$, which is entered at the end of each episode and left with probability $1 - \epsilon$ to start a new one.
We also added a special atomic proposition $\labelset{reset}$ into $\atomicprops$ to label this reset state and reason about episodic behaviors.
For instance, this allows verifying whether the agent behaves safely during the entire episode, or if it is able to reach a goal before the end of the episode.

\subsection{Environments with initial distribution}
Many environments do not necessarily have a single initial state, but rather an initial distribution over states $d_I \in \distributions{\states}$.
In that case, the results presented in this paper remain unchanged: it suffices to add a dummy state $\state^{\star}$ to the state space $\states \cup \set{\state^{\star}}$ so that $\sinit = \state^{\star}$ with the transition dynamics $\probtransitions\fun{\state' \mid \state^{\star}, \action} = d_I\fun{\state'}$ for any action $\action \in \actions$.
Therefore, each time the reset of the environment is triggered, we make the MDP entering the initial state $\state^{\star}$, then transitioning to $\state'$ according to $d_I$.

\subsection{Latent space distribution}
\begin{figure}
    \centering
    \includegraphics[width=0.32\textwidth]{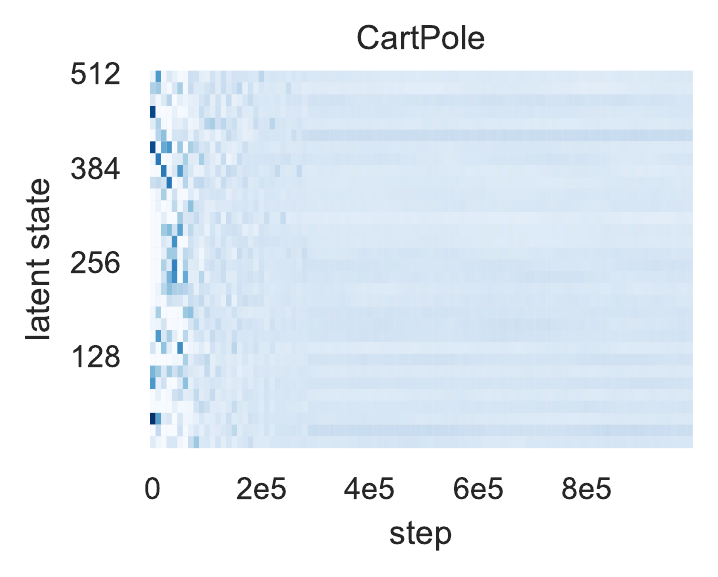}
    \includegraphics[width=0.32\textwidth]{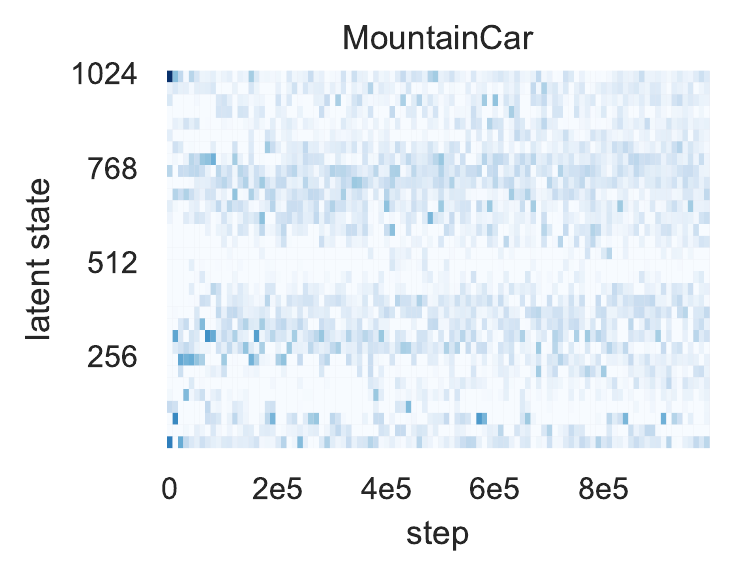}
    \includegraphics[width=0.32\textwidth]{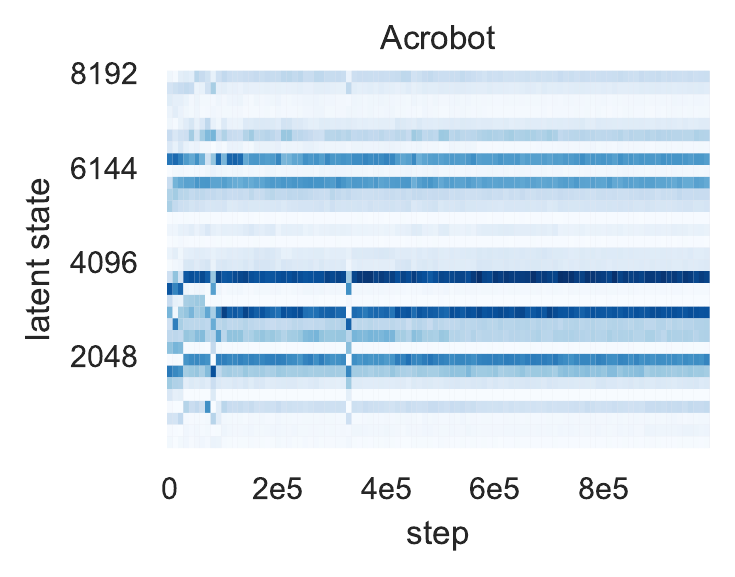}
    \\
    \includegraphics[width=0.32\textwidth]{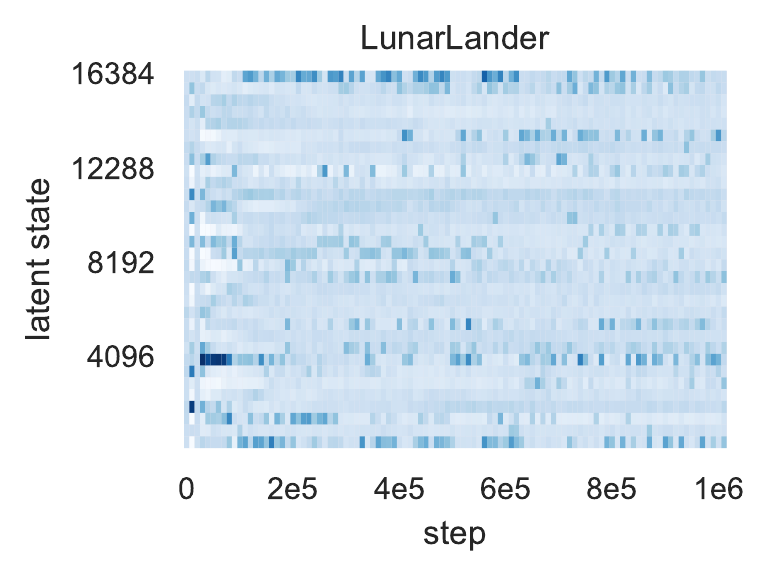}
    \includegraphics[width=0.31\textwidth]{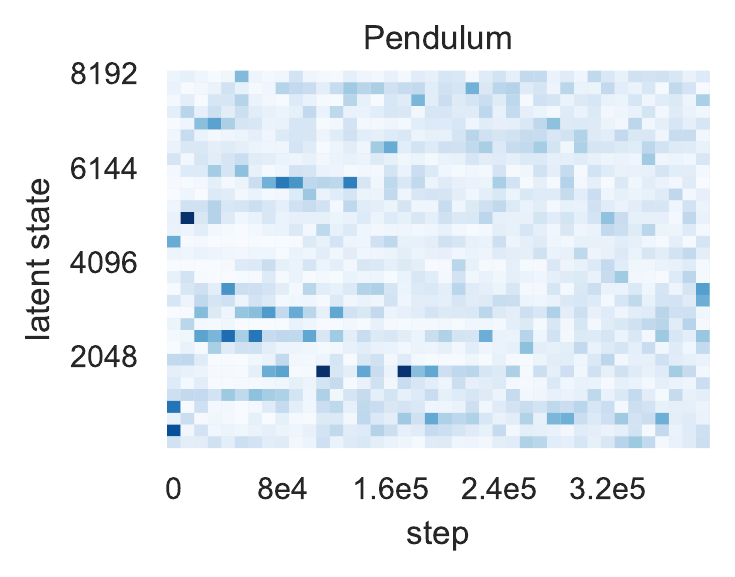}
    \caption{Latent space distribution along training steps. The intensity of the blue hue corresponds to the frequency of latent states produced by $\embed_{\encoderparameter}$ during training.}
    \label{fig:histograms}
\end{figure}
As pointed out in Sect.~\ref{sec:experiments}, posterior collapse is naturally avoided when optimizing \waemdp.
To illustrate that, we report the distribution of latent states produced by $\embed_{\encoderparameter}$ during training (Fig.~\ref{fig:histograms}).
The plots reveal that the latent space generated by mapping original states drawn from $\stationary{\policy}$ during training to $\latentstates$ via $\embed_{\encoderparameter}$ is fairly distributed, for each environment.

\subsection{Distance Metrics: state, action, and reward reconstruction}

The choice of the distance functions $\distance_{\states}$, $\distance_{\actions}$, and $\distance_{\rewards}$, plays a role in the success of our approach.
The usual Euclidean distance is often a good choice for all the transition components, but the scale, dimensionality, and nature of the inputs sometimes require using scaled, normalized, or other kinds of distances to allow the network to reconstruct each component. While we did not observe such requirements in our experiments (where we simply used the Euclidean distance), high dimensional observations (e.g., images) are an example of data which could require tuning the state-distance function in such a way, to make sure that the optimization of the reward or action reconstruction will not be disfavored compared to that of the states.

\subsection{Value difference}
In addition to reporting the quality guarantees of the model along training steps through local losses (cf. Figure~\ref{subfig:pac-losses}), our experiments revealed that the absolute value difference $\norm{V_{\latentpolicy_{\decoderparameter}}}$ between the original and latent models operating under the latent policy  quickly decreases and tends to converge to values in the same range (Figure \ref{fig:value-diff}). 
\begin{figure}
    \centering
    \includegraphics[width=\textwidth]{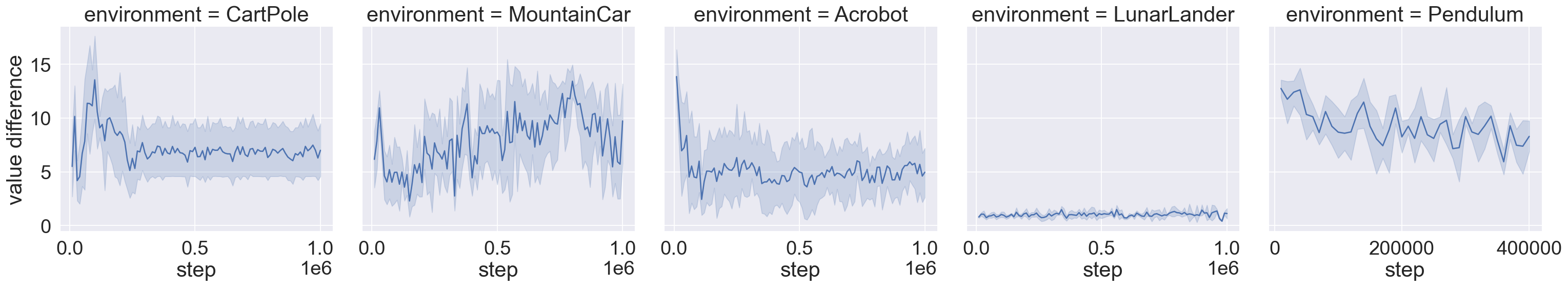}
    \caption{Absolute value difference $\norm{V_{\latentpolicy_{\decoderparameter}}}$ reported along training steps.}
    \label{fig:value-diff}
\end{figure}
This is consistent with the fact that minimizing local losses lead to close behaviors (cf. Eq.~\ref{eq:bidistance-bound}) and that the value function is Lipschitz-continuous w.r.t. $\bidistance_{\latentpolicy_{\decoderparameter}}$ (cf. Section~\ref{sec:background}).

\subsection{Remark on formal verification}\label{rmk:reachability}
Recall that \emph{our bisimulation guarantees come by construction of the latent space.}
Essentially, our learning algorithm spits out a distilled policy and a latent state space which already yields a guaranteed bisimulation distance between the original MDP and the latent MDP.
This is the crux of how we enable verification techniques like model checking. 
In particular, bisimulation guarantees mean that \emph{reachability probabilities in the latent MDP compared to those in the original one are close}.
Furthermore, the value difference of (omega-regular) properties (formulated through mu-calculus) obtained in the two models is bounded by this distance (cf. Sect.~\ref{sec:background} and \citealt{DBLP:conf/fsttcs/ChatterjeeAMR08}). 

\smallparagraph{Reachability is the key ingredient} to model-check MDPs.
Model-checking properties is in most cases performed by reduction to the reachability of components or regions of the MDP: it either consists of (i) iteratively checking the reachability of the parts of the state space satisfying path formulae that comprise the specification, through a tree-like decomposition of the latter (e.g., for (P,R-)CTL properties, cf. \citealt{DBLP:BK08}), or (ii) checking the reachability to the part of the state space of a product of the MDP with a memory structure or an automaton that embeds the omega-regular property --- e.g., for LTL \citep{DBLP:conf/cav/BaierK0K0W16,DBLP:conf/cav/SickertEJK16}, LTLf \citep{DBLP:journals/corr/abs-2009-10883}, or GLTL \citep{DBLP:journals/corr/LittmanTFIWM17}, among other specification formalisms. 
The choice of specification formalism is up to the user and depends on the case study. {The scope of this work is focusing on learning to distill RL policies with bisimulation guarantees \emph{so that model checking can be applied}, in order to reason about the behaviors of the agent}. That being said, \emph{reachability is all we need} to show that model checking can be applied.

\subsection{Hyperparameters}\label{appendix:hyperparams}

\smallparagraph{\waemdp parameters.}~All components (e.g., functions or distribution locations and scales, see Fig.~\ref{fig:wae-architecture}) are represented and inferred by neural networks (multilayer perceptrons).
All the networks share the same architecture (i.e., number of layers and neurons per layer).
We use a simple uniform experience replay of size $10^{6}$ to store the transitions and sample them.
The training starts when the agent has collected $10^{4}$ transitions in $\mdp$.
We used minibatches of size $128$ to optimize the objective and we applied a minibatch update every time the agent executing $\policy$ has performed $16$ steps in $\mdp$.
We use the recursive $\epsilon$-perturbation trick of \cite{DBLP:conf/nips/Huang20} with $\epsilon = \nicefrac{3}{4}$: when an episode ends, it restarts from the initial state with probability $\nicefrac{1}{4}$; before re-starting an episode, the time spent in the reset state labeled with $\labelset{reset}$ follows then the geometric distribution with expectation $\nicefrac{\epsilon}{1 - \epsilon} = 3$.
We chose the same latent state-action space size than \cite{DBLP:journals/corr/abs-2112-09655}, except for LunarLander that we decreased to $\log_2 \left| \latentstates \right| = 14$ and $\left|\latentactions\right| = 3$ to improve the scalability of the verification.

\smallparagraph{VAE-MDPs parameters.}~For the comparison of Sect.~\ref{sec:experiments}, we used the exact same VAE-MDP hyperparameter set as prescribed by \cite{DBLP:journals/corr/abs-2112-09655}, except for the state-action space of LunarLander that we also changed for scalability and fair comparison purpose.%
\footnote{The code for conducting the VAE-MDPs experiments is available at \url{https://github.com/florentdelgrange/vae_mdp} (GNU General Public License v3.0).}

\smallparagraph{Hyperparameter search.}~To evaluate our \waemdp, we realized a search in the parameter space defined in Table~\ref{table:hyperparameter-search}.
The best parameters found (in terms of trade-off between performance and latent quality) are reported in Table~\ref{table:hyperparameters}.
We used two different optimizers for minimizing the loss (referred to as the minimizer) and computing the Wasserstein terms (reffered to as the maximizer).
We used \textsc{Adam} \citepAR{DBLP:journals/corr/KingmaB14} for the two, but we allow for different learning rates $\textsc{Adam}_{\alpha}$ and exponential decays $\textsc{Adam}_{\beta_1}, \textsc{Adam}_{\beta_2}$.
We also found that polynomial decay for $\textsc{Adam}_{\alpha}$ (e.g., to $10^{-5}$ for $4\cdot 10^{5}$ steps) is a good practice to stabilize the experiment learning curves, but is not necessary to obtain high-quality and performing distillation.
Concerning the continuous relaxation of discrete distributions,
we used a different temperature for each distribution, as \cite{DBLP:conf/iclr/MaddisonMT17} pointed out that doing so is valuable to improve the results.
We further followed the guidelines of \citet{DBLP:conf/iclr/MaddisonMT17} to choose the interval of temperatures and did not schedule any annealing scheme (in contrast to VAE-MDPs). 
Essentially, the search reveals that the regularizer scale factors $\beta_{\scalebox{1.1}{$\cdot$}}$ (defining the optimization direction) as well as the encoder and latent transition temperatures are important to improve the performance of distilled policies.
For the encoder temperature, we found a nice spot in $\temperature_{\scriptscriptstyle \embed_{\encoderparameter}} = \nicefrac{2}{3}$, which provides the best performance in general, whereas the choice of $\temperature_{\scriptscriptstyle \latentprobtransitions_{\decoderparameter}}$ and $\beta_{\scalebox{1.1}{$\cdot$}}$ are (latent-) environment dependent.
The importance of the temperature parameters for the continuous relaxation of discrete distributions is consistent with the results of \citep{DBLP:conf/iclr/MaddisonMT17}, revealing that the success of the relaxation depends on the choice of the temperature for the different latent space sizes. 

\begin{table}
\centering
\caption{Hyperparameter search. $\temperature_X$ refers to the temperature used for \waemdp component $X$.}
\label{table:hyperparameter-search}
\resizebox{\columnwidth}{!}{%
\begin{tabular}{@{}ll@{}}
\toprule
Parameter &
  Range \\ \midrule
$\textsc{Adam}_{\alpha}$ (minimizer) &
  $\set{0.0001, 0.0002, 0.0003, 0.001}$ \\
$\textsc{Adam}_{\alpha}$ (maximizer) &
  $\set{0.0001, 0.0002, 0.0003, 0.001}$ \\
$\textsc{Adam}_{\beta_1}$ &
  $\set{0, 0.5, 0.9}$ \\
$\textsc{Adam}_{\beta_2}$ &
  $\set{0.9, 0.999}$ \\
neurons per layer &
  $\set{64, 128, 256, 512}$ \\
number of hidden layers &
  $\set{1, 2, 3}$ \\
activation &
  $\set{\text{ReLU}, \text{Leaky ReLU}, \text{tanh}, \frac{\text{softplus}\fun{2x + 2}}{2} - 1 \, \textit{(smooth ELU)}}$ \\
$\beta_{\steadystateregularizer{\policy}}$ &
  $\set{10, 25, 50, 75, 100}$ \\
$\beta_{\localtransitionloss{\stationary{\policy}}}$ &
  $\set{10, 25, 50, 75, 100}$ \\
$\ncritic$ &
  $\set{5, 10, 15, 20}$ \\
$\delta$ &
  $\set{10, 20}$ \\
use $\varepsilon$-mimic (cf. \citealt{DBLP:journals/corr/abs-2112-09655}) &
  $\set{\text{True}, \text{False}}$ (if True, a decay rate of $10^{-5}$ is used) \\
$\temperature_{\scriptscriptstyle\latentprobtransitions_{\decoderparameter}}$ &
  $\set{0.1, \nicefrac{1}{3}, \nicefrac{1}{2}, \nicefrac{2}{3}, \nicefrac{3}{5}, 0.99}$ \\
$\temperature_{\scriptscriptstyle\embed_{\encoderparameter}}$ &
  $\set{0.1, \nicefrac{1}{3}, \nicefrac{1}{2}, \nicefrac{2}{3}, \nicefrac{3}{5}, 0.99}$ \\
$\temperature_{\scriptscriptstyle\latentpolicy_{\decoderparameter}}$ &
  $\set{\nicefrac{1}{\left| \latentactions \right| - 1}, \nicefrac{1}{\fun{\left| \latentactions \right| - 1} \cdot 1.5}}$ \\
$\temperature_{\scriptscriptstyle \actionencoder}$ &
 $\set{\nicefrac{1}{\left| \latentactions \right| - 1}, \nicefrac{1}{\fun{\left| \latentactions \right| - 1} \cdot 1.5}}$ \\ \bottomrule
\end{tabular}%
}
\end{table}

\begin{table}
\centering
\caption{Final hyperparameters used to evaluate \waemdps in Sect.~\ref{sec:experiments}}
\label{table:hyperparameters}
\resizebox{\columnwidth}{!}{%
\begin{tabular}{llllll}
\toprule
{} &          CartPole &       MountainCar &           Acrobot &       LunarLander &          Pendulum \\
\midrule
$\log_2 \left| \latentstates \right|$                                         &  9 &  10 &  13 &  14 &  13 \\
$\left| \latentactions \right|$                                               &  2 $=\left|\actions\right|$ & $2=\left|\actions\right|$  & $3=\left|\actions\right|$ &  $3$ &  $3$ \\
activation                                                                    &  tanh &  ReLU &  Leaky Relu &  ReLU &  ReLU \\
layers                                                                        &  $[64, 64, 64]$ &  $[512, 512]$ &  $[512, 512]$ &  $[256]$ &  $[256, 256, 256]$ \\
$\textsc{Adam}_{\alpha}$ (minimizer)                                          & $0.0002$ & $0.0001$ & $0.0002$ & $0.0003$ & $0.0003$ \\
$\textsc{Adam}_{\alpha}$ (maximizer)                                          & $0.0002$ & $0.0001$ & $0.0001$ & $0.0003$ & $0.0003$ \\
$\textsc{Adam}_{\beta_1}$                                                     & $0.5$ & $0$ & $0$ & $0$ & $0.5$ \\
$\textsc{Adam}_{\beta_2}$                                                     & $0.999$ & $0.999$ & $0.999$ & $0.999$ & $0.999$ \\
$\beta_{\localtransitionloss{\stationary{\policy}}}$                          & $10$ & $25$ & $10$ & $50$ & $25$ \\
$\beta_{\steadystateregularizer{\policy}}$                                    & $75$ & $100$ & $10$ & $100$ & $25$ \\
$\ncritic$                                          &  5 &  20 &  20 &  15 &  5 \\
$\delta$                                                                      & $20$ & $10$ & $20$ & $20$ & $10$ \\
$\varepsilon$                                                                 & $0$ & $0$ & $0$ & $0$ & $0.5$ \\
$\temperature_{\scriptscriptstyle\latentprobtransitions_{\decoderparameter}}$ & $\nicefrac{1}{3}$ & $\nicefrac{1}{3}$ & $0.1$ & $0.75$ & $\nicefrac{2}{3}$ \\
$\temperature_{\scriptscriptstyle\embed_{\encoderparameter}}$                 & $\nicefrac{1}{3}$ & $\nicefrac{2}{3}$ & $\nicefrac{2}{3}$ & $\nicefrac{2}{3}$ & $\nicefrac{2}{3}$ \\
$\temperature_{\scriptscriptstyle\latentpolicy_{\decoderparameter}}$          & $\nicefrac{2}{3}$ & $\nicefrac{1}{3}$ & $0.5$ & $0.5$ & $0.5$ \\
$\temperature_{\scriptscriptstyle \actionencoder}$                            & / & / & / & $\nicefrac{1}{3}$ & $\nicefrac{1}{3}$ \\
\bottomrule
\end{tabular}
}
\end{table}

\smallparagraph{Labeling functions.}~We used the same labeling functions as those described by \citet{DBLP:journals/corr/abs-2112-09655}.
For completeness, we recall the labeling function used for each environment in Table~\ref{appendix:table:labels}.

\smallparagraph{Time to failure properties.}~%
Based on the labeling described in Table~\ref{appendix:table:labels}, we formally detail the time to failure properties checked in Sect.~\ref{sec:experiments} whose results are listed in Table~\ref{table:evaluation} for each environment.
Let $\labelset{Reset} = \set{\mathsf{reset}} = \tuple{0, \dots, 1}$ (we assume here that the last bit indicates whether the current state is a reset state or not)
and define $\state \models \labelset{L}_1 \wedge \labelset{L}_2$ iff $s \models \labelset{L}_1$ and $s \models \labelset{L}_2$ for any $\state \in \states$, then
\begin{itemize}
    \item \emph{CartPole}: $\varphi = \until{\neg \labelset{Reset}}{\labelset{Unsafe}}$, where $\labelset{Unsafe} = \tuple{1, 1, 0}$
    \item \emph{MountainCar}: $\varphi = \until{\neg \labelset{Goal}}{\labelset{Reset}}$, where $\labelset{Goal} = \tuple{1, 0, 0, 0}$
    \item \emph{Acrobot}: $\varphi = \until{\neg \labelset{Goal}}{\labelset{Reset}}$, where $\labelset{Goal} = \tuple{1, 0, \dots, 0}$
    \item \emph{LunarLander}: $\varphi = \until{\neg \labelset{SafeLanding}}{\labelset{Reset}}$, where $\labelset{SafeLanding} = \labelset{GroundContact} \wedge \labelset{MotorsOff}$, $\labelset{GroundContact} = \tuple{0, 1, 0, 0, 0, 0, 0}$, and $ \labelset{MotorsOff} = \tuple{0, 0, 0, 0, 0, 1, 0}$
    \item \emph{Pendulum}: $\varphi = \eventually\fun{\neg\mathsf{Safe} \wedge \ltlnext \mathsf{Reset}}$,
    where $\labelset{Safe} = \tuple{1, 0, 0, 0, 0}$,
    $\eventually \labelset{T} = \neg \until{\emptyset}{\labelset{T}}$, and
    $\state_{i} \models \ltlnext \labelset{T}$ iff $\state_{i + 1} \models \labelset{T}$,
    for any $\labelset{T} \subseteq \atomicprops, {{\state}_{i : \infty}, {\action}_{i : \infty} }\in \inftrajectories{\mdp}$.
    Intuitively, $\varphi$ denotes the event of ending an episode in an unsafe state, just before resetting the environment, which means that either the agent never reached the safe region or it reached and left it at some point. 
    Formally, $\varphi =
	\set{\seq{\state}{\infty}, \seq{\action}{\infty} \, | \, \exists i \in \N, \state_i \not\models \labelset{Safe}\, \wedge \, \state_{i + 1} \models \labelset{Reset} } \subseteq \inftrajectories{\mdp}$.
\end{itemize}

\begin{table}
\centering
\resizebox{\columnwidth}{!}{%
\begin{tabular}{lllll}
\toprule
Environment &
  $\states \subseteq$ &
  Description, for $\vect{\state} \in \states$ &
  $\labels\fun{\vect{s}} = \tuple{p_1, \dots, p_n, p_{\mathsf{reset}}}$ \\
 \midrule
CartPole &
  $\R^4$ &
  \begin{tabular}[c]{@{}l@{}}\tabitem $\vect{\state}_1$: cart position\\ \tabitem $\vect{\state}_2$: cart velocity\\ \tabitem $\vect{\state}_3$: pole angle (rad)\\ \tabitem $\vect{\state}_4$: pole velocity at tip\end{tabular} &
  \begin{tabular}[c]{@{}l@{}}\tabitem $p_1 = \condition{\vect{\state}_1 \geq 1.5}$: unsafe cart position\\
  \tabitem $p_2 = \condition{\vect{\state}_3 \geq  0.15} $: unsafe pole angle\end{tabular} \\ \midrule
MountainCar &
  $\R^2$ &
  \begin{tabular}[c]{@{}l@{}}\tabitem $\vect{\state}_1$: position\\ \tabitem $\vect{\state}_2$: velocity\end{tabular} &
  \begin{tabular}[c]{@{}l@{}}\tabitem $p_1 = \condition{\vect{\state}_1 > 1.5}$: target position\\ \tabitem $p_2 = \condition{\vect{\state}_1 \geq \nicefrac{-1}{2}}$: right-hand side of the mountain\\ \tabitem $p_3 = \condition{\vect{\state}_2 \geq 0}$: car going forward\end{tabular} \\ \midrule
Acrobot &
  $\R^6$ &
  \begin{tabular}[c]{@{}l@{}}Let $\theta_1, \theta_2 \in \mathopen[0, 2\pi \mathclose]$ be the angles\\ of the two rotational joints,\\ \tabitem $\vect{\state}_1 = \cos\fun{\theta_1}$\\ \tabitem $\vect{\state}_2 = \sin\fun{\theta_1}$\\ \tabitem $\vect{\state}_3 = \cos\fun{\theta_2}$\\ \tabitem $\vect{\state}_4 = \sin\fun{\theta_2}$\\ \tabitem $\vect{\state}_5$: angular velocity 1\\ \tabitem $\vect{\state}_6$: angular velocity 2\end{tabular} &
  \begin{tabular}[c]{@{}l@{}}\tabitem $p_1= \condition{-\vect{\state}_1 -\vect{\state}_3 \cdot \vect{\state}_1 + \vect{\state}_4 \cdot \vect{\state}_2 > 1}$: RL agent target \\\tabitem $p_2 = \condition{\vect{\state}_1 \geq 0}$: $\theta_1 \in [0, \nicefrac{\pi}{2}] \cup [\nicefrac{3\pi}{2}, 2\pi]$ \\
  \tabitem $p_3 = \condition{\vect{\state}_2 \geq 0} $: $\theta_1 \in [0, \pi]$\\
  \tabitem $p_4 = \condition{\vect{\state}_3 \geq 0}$: $\theta_2 \in [0, \nicefrac{\pi}{2}] \cup [\nicefrac{3\pi}{2}, 2\pi]$\\
  \tabitem $p_5 = \condition{\vect{\state}_4 \geq 0}$: $\theta_2 \in [0, \pi]$\\
  \tabitem $p_6 = \condition{\vect{\state}_5 \geq 0}$: positive angular velocity (1)\\
  \tabitem $p_7 = \condition{\vect{\state}_6 \geq 0}$: positive angular velocity (2)\\
  \end{tabular} \\\midrule
Pendulum &
  $\R^3$ &
  \begin{tabular}[c]{@{}l@{}}Let $\theta \in \mathopen[0, 2\pi \mathclose]$ be the joint angle\\ \tabitem $\vect{\state}_1 = \cos\fun{\theta}$\\ \tabitem $\vect{\state}_2 = \sin\fun{\theta}$\\ \tabitem $\vect{\state}_3$: angular velocity\end{tabular} &
    \begin{tabular}[c]{@{}l@{}}
    \tabitem $p_1 = \condition{\vect{\state}_1 \geq \cos\fun{\nicefrac{\pi}{3}}}$: safe joint angle \\
    \tabitem $p_2 = \condition{\vect{\state}_1 \geq 0}$: $\theta \in [0, \nicefrac{\pi}{2}] \cup [\nicefrac{3\pi}{2}, 2\pi]$ \\
  \tabitem $p_3 = \condition{\vect{\state}_2 \geq 0}$: $\theta \in [0, \pi]$\\
  \tabitem $p_4 = \condition{\vect{\state}_3 \geq 0}$: positive angular velocity\\
  \end{tabular} \\\midrule
LunarLander &
  $\R^8$ &
  \begin{tabular}[c]{@{}l@{}}\tabitem $\vect{\state}_1$: horizontal coordinates\\ \tabitem $\vect{\state}_2$: vertical coordinates\\ \tabitem $\vect{\state}_3$: horizontal speed\\ \tabitem $\vect{\state}_4$: vertical speed\\ \tabitem $\vect{\state}_5$: ship angle\\ \tabitem $\vect{\state}_6$: angular speed\\ \tabitem $\vect{\state}_7$: left leg contact\\ \tabitem $\vect{\state}_8$: right leg contact\end{tabular} &
  \begin{tabular}[c]{@{}l@{}}
  \tabitem $p_1$: unsafe angle\\
  \tabitem $p_2$: leg ground contact\\
  \tabitem $p_3$: lands rapidly\\
  \tabitem $p_4$: left inclination\\
  \tabitem $p_5$: right inclination\\
  \tabitem $p_6$: motors shut down 
  \end{tabular}\\
\bottomrule
\end{tabular}
}
\caption{
Labeling functions for the OpenAI environments considered in our experiments \citep{DBLP:journals/corr/abs-2112-09655}.
We provide a short description of the state space and the meaning of each atomic proposition.
Recall that labels are binary encoded, for
$n = |\atomicprops| - 1$ (one bit is reserved for $\mathsf{reset}$) and $p_{\mathsf{reset}} = 1$ iff $\vect{s}$ is a reset state (cf. Appendix~\ref{appendix:experiments:stationary-distribution}). 
}
\label{appendix:table:labels}
\end{table}

\section{On the curse of Variational Modeling}\label{appendix:posterior-collapse}

\emph{Posterior collapse} is a well known issue occurring in variational models (see, e.g., \citealtAR{DBLP:conf/icml/AlemiPFDS018,DBLP:conf/iclr/TolstikhinBGS18,DBLP:conf/iclr/HeSNB19,DBLP:conf/icml/DongS0B20}) which intuitively results in a degenerate local optimum where the model learns to ignore the latent space and use only the reconstruction functions (i.e., the decoding distribution) to optimize the objective.
VAE-MDPs are no exception, as pointed out in the original paper (\citealp[Section 4.3 and Appendix C.2]{DBLP:journals/corr/abs-2112-09655}).

Formally, VAE- and WAE-MDPs optimize their objective by minimizing two losses: a \emph{reconstruction cost} plus a \emph{regularizer term} which penalizes a discrepancy between the encoding distribution and the dynamics of the latent space model. 
In VAE-MDPs, the former corresponds to the the \emph{distortion}, and the later to the \emph{rate} of the variational model (further details are given in \citealt{DBLP:conf/icml/AlemiPFDS018,DBLP:journals/corr/abs-2112-09655}), while in our WAE-MDPs, the former corresponds to the raw transition distance and the later to both the steady-state and transition regularizers.
Notably, the rate minimization of VAE-MDPs involves regularizing a \emph{stochastic} {embedding function} $\embed_{\encoderparameter}\fun{\sampledot \mid \state}$ \emph{point-wise}, i.e., for all different input states $\state \in \states$ drawn from the interaction with the original environment.
In contrast, the latent space regularization of the WAE-MDP involves the marginal embedding distribution $\encodersymbol_{\encoderparameter}$ where the embedding function $\embed_{\encoderparameter}$ is not required to be stochastic.
\cite{DBLP:conf/icml/AlemiPFDS018} showed that \emph{posterior collapse occurs in VAEs when the rate of the variational model is close to zero,} leading to low-quality representation.

\begin{figure*}
    \begin{subfigure}{.6\textwidth}
        \includegraphics[width=.495\textwidth]{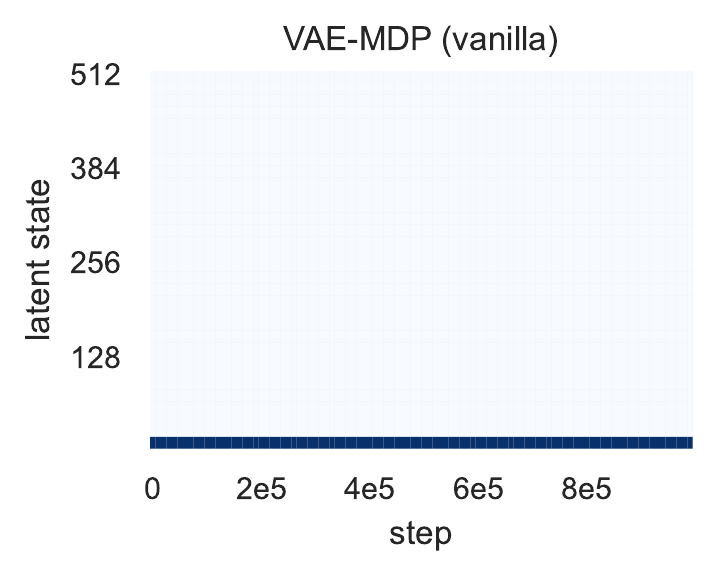}
        \hfill
        \includegraphics[width=.495\textwidth]{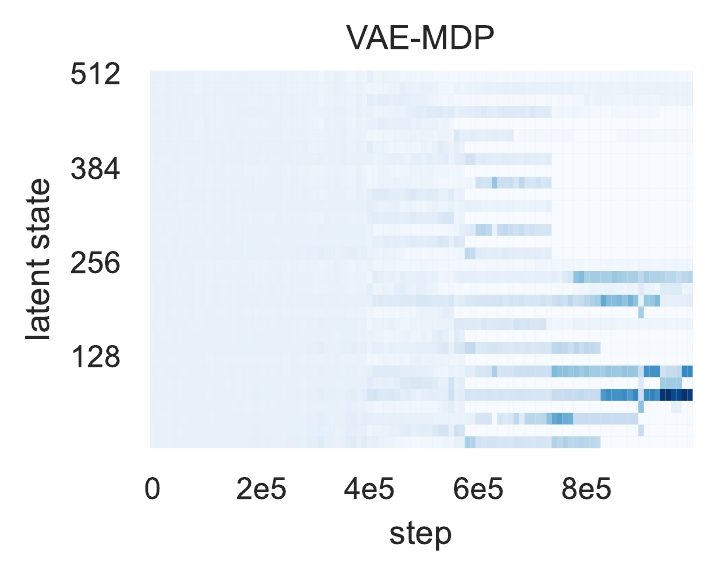}
        \caption{Latent space distribution along training steps. The intensity of the blue hue corresponds to the frequency of latent states produced from $\embed_{\encoderparameter}$ during training.
        The vanilla model collapses to a single state.}
        \label{subfig:mode-collapse-state-frequency}
    \end{subfigure}
    \begin{subfigure}{.4\textwidth}
        \includegraphics[width=\textwidth]{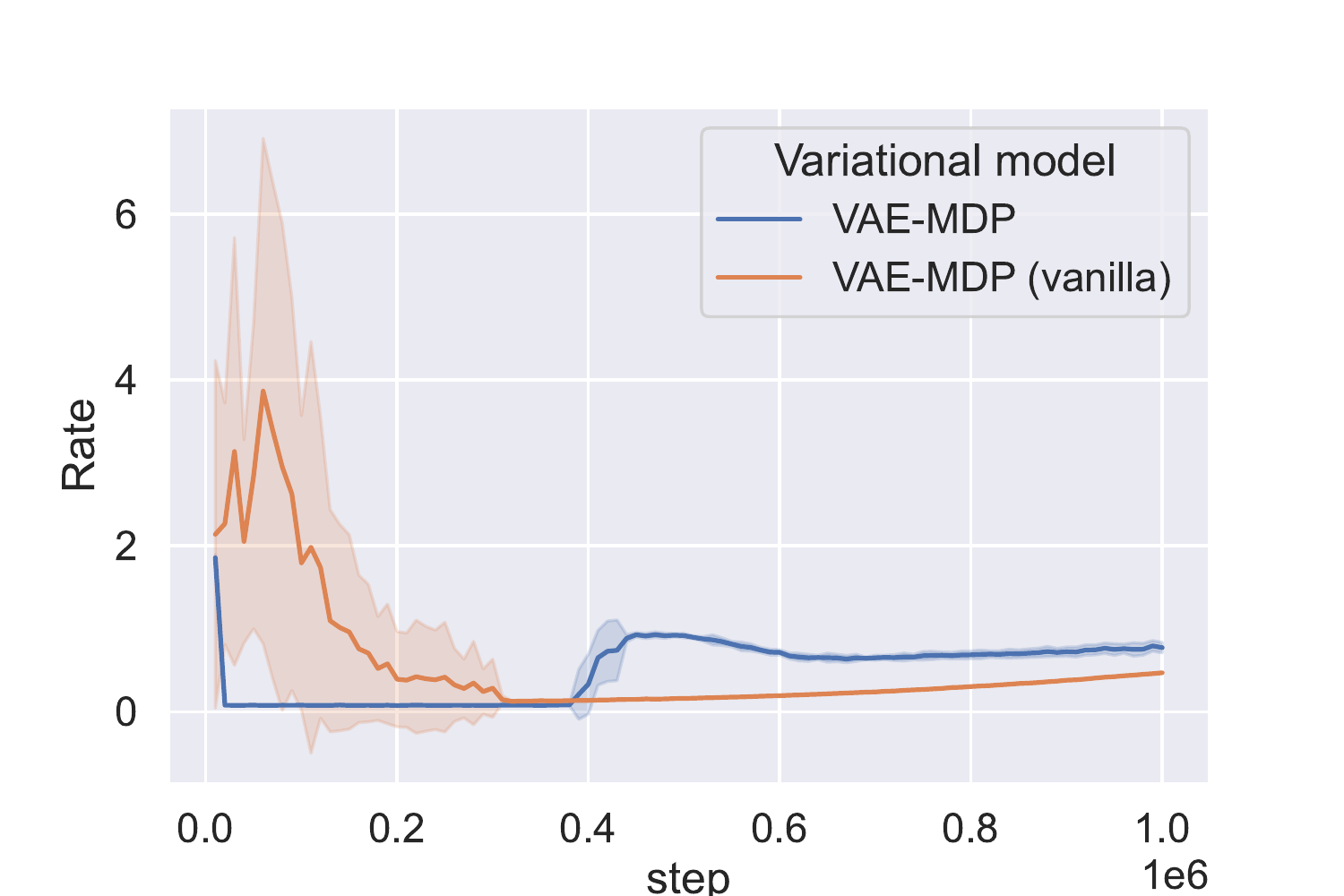}
        \caption{Rate of the variational model.}
        \label{subfig:rate}
    \end{subfigure}
    \begin{subfigure}{0.32\textwidth}
        \includegraphics[width=\textwidth]{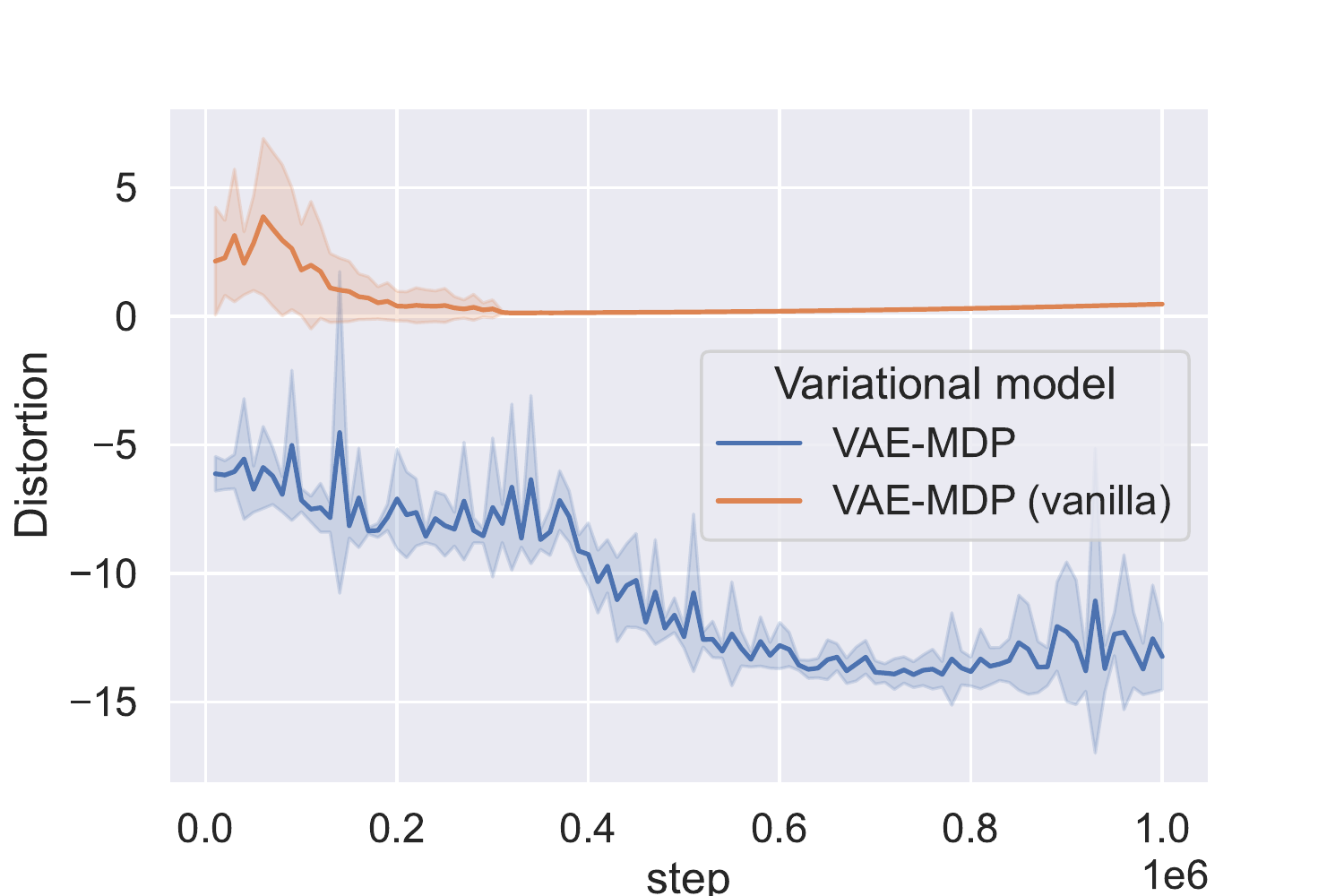}
        \caption{Distortion of the variational model.\\ $\,$\\ $\;$}
        \label{subfig:distortion}
    \end{subfigure}
    \hspace{.01333\textwidth}
    \begin{subfigure}{0.32\textwidth}
        \includegraphics[width=\textwidth]{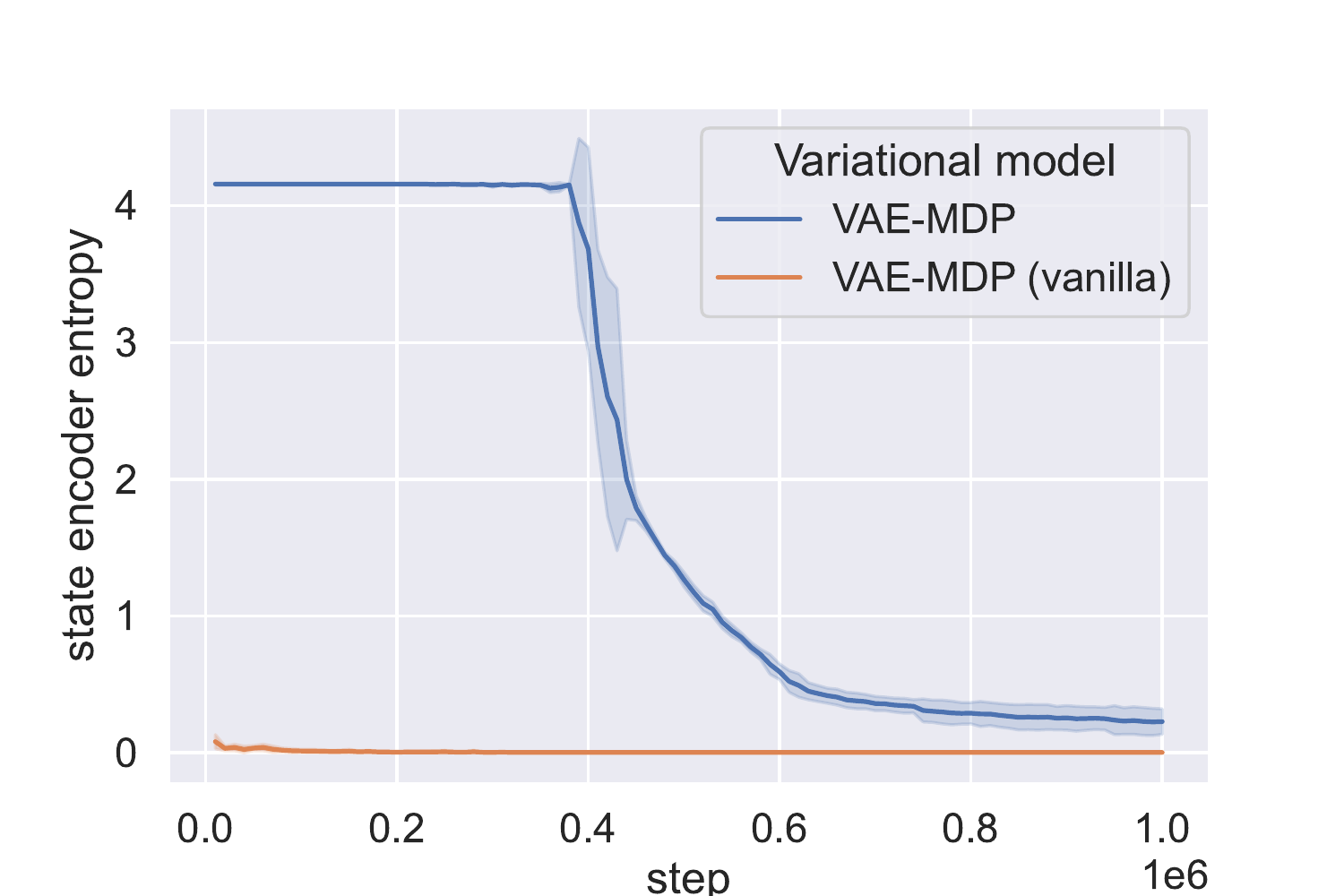}
        \caption{Average point-wise entropy of $\embed_{\encoderparameter}\fun{\sampledot \mid \state}$, for $\state \in \states$ drawn from the interaction with the original environment.}
        \label{subfig:encoder-entropy}
    \end{subfigure}
    \hspace{.01333\textwidth}
    \begin{subfigure}{0.31\textwidth}
        \includegraphics[width=\textwidth]{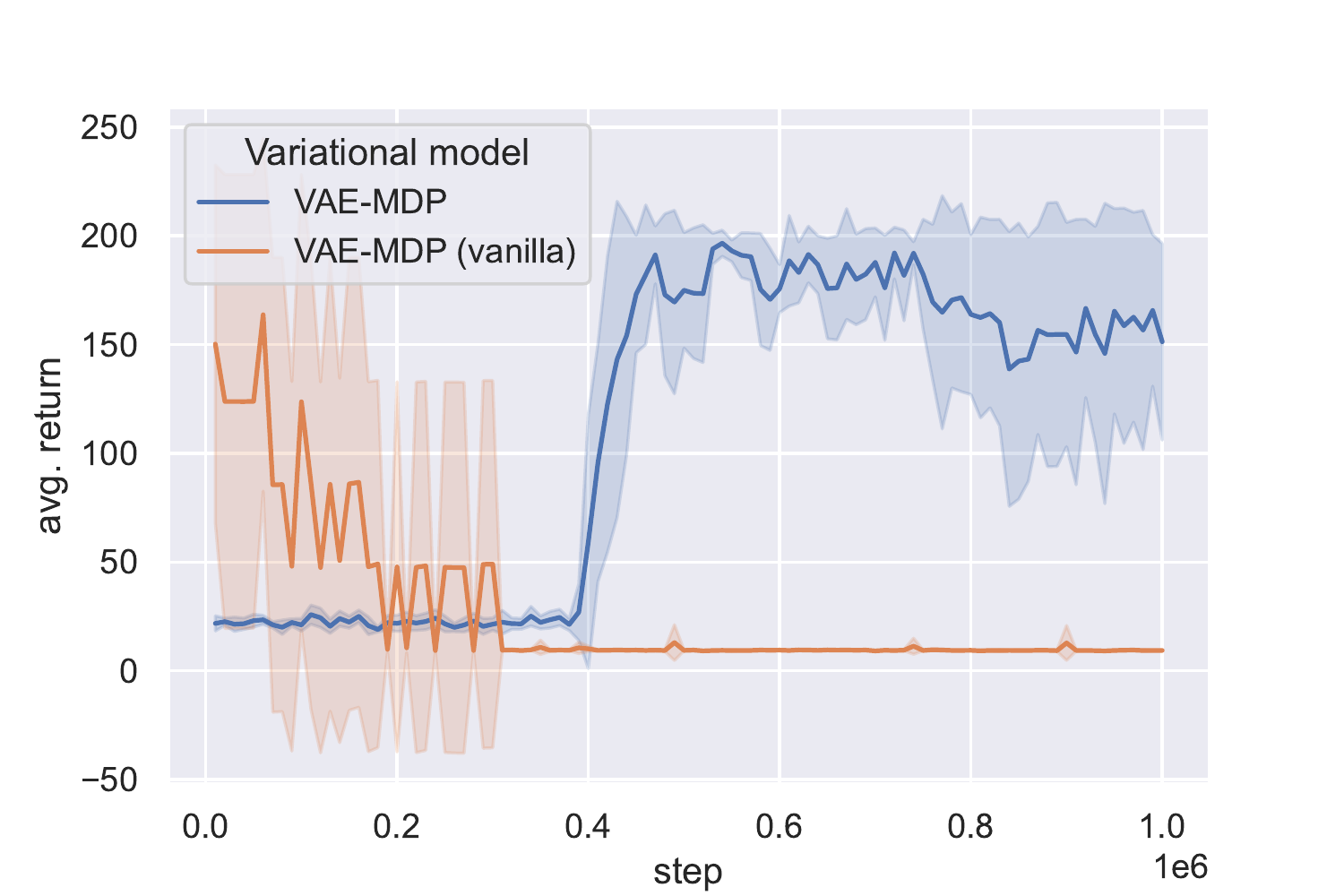}
        \caption{Performance of the resulting distilled policy $\latentpolicy_{\decoderparameter}$ when deployed in the original environment (averaged over 30 episodes).}
        \label{subfig:vae-distillation}
    \end{subfigure}
    \caption{
        Comparison of the VAE-MDP in the CartPole environment (i) when the distortion and the rate are minimized as is (\emph{vanilla model}) and (ii) when it makes use of annealing schemes, entropy regularization, and prioritized experience replay to avoid posterior collapse (cf.~\citealt{DBLP:journals/corr/abs-2112-09655}).
        While the former clearly fails to learn a useful latent representation, the later does so meticulously and smoothly in two distinguishable phases:
        first, $\embed_{\encoderparameter}$ focuses on fairly distributing the latent space, setting up the stage to the concrete optimization occurring from step $4 \cdot 10^5$, where the entropy of $\embed_{\encoderparameter}$ is lowered, which allows to get the rate of the variational model away from zero.
        Five instances of the models are trained with different random seeds, with the same hyperparameters than in Sect.~\ref{sec:experiments}.
    }
    \label{fig:vae-mode-collapse}
\end{figure*}

\smallparagraph{Posterior collapse in VAE-MDPs.}~We illustrate the sensitivity of VAE-MDPs to the posterior collapse problem in Fig.~\ref{fig:vae-mode-collapse}, through the CartPole environment\footnote{
    In fact, the phenomenon of collapsing to few state occurs for all the environments considered in this paper when their prioritized experience replay is not used, as illustrated in \citealp[Appendix~C.2]{DBLP:journals/corr/abs-2112-09655}.
}: minimizing the distortion and the rate as is yields an embedding function which maps deterministically every input state to the same \emph{sink} latent state (cf. Fig.~\ref{subfig:mode-collapse-state-frequency}).
Precisely, there is a latent state $\latentstate \in \latentstates$ so that $\embed_{\encoderparameter}\fun{\latentstate \mid \state} \approx 1$ and $\latentprobtransitions_{\decoderparameter}\fun{\latentstate \mid \latentstate, \latentaction} \approx 1$ whatever the state $\state \in \states$ and action $\latentaction \in \latentactions$.
This is a form of posterior collapse, the resulting rate quickly drops to zero (cf. Fig~\ref{subfig:rate}), and the resulting latent representation yields no information at all.
This phenomenon is handled in VAE-MDPs by using (i) prioritized replay buffers that allow to focus on inputs that led to bad representation, and (ii) modifying the objective function for learning the latent space model --- the so-called evidence lower bound \citepAR{DBLP:journals/jmlr/HoffmanBWP13,DBLP:journals/corr/KingmaW13}, or \textsc{ELBO} for short --- and set up annealing schemes to eventually recover the \textsc{ELBO} at the end of the training process.
Consequently, the resulting learning procedure focuses primarily on fairly distributing the latent space, to avoid it to collapse to a single latent state,
to the detriment of learning the dynamics of the environment and the distillation of the RL policy.
Then, the annealing scheme allows to make the model learn to finally smoothly use the latent space to maximize the \textsc{ELBO}, and achieve consequently a lower distortion at the ``price'' of a higher rate.  

\smallparagraph{Impact of the resulting learning procedure.}~The aforementioned annealing process, used to avoid that every state collapses to the same representation, possibly induces a high entropy embedding function (Fig.~\ref{subfig:encoder-entropy}), which further complicates the learning of the model dynamics and the distillation in the first stage of the training process.
In fact, in this particular case, one can observe that the entropy reaches its maximal value, which yields a fully random state embedding function.
Recall that the VAE-MDP latent space is learned through \emph{independent} Bernoulli distributions.
Fig.~\ref{subfig:encoder-entropy} reports values centered around $4.188$ in the first training phase, which corresponds to the entropy of the state embedding function when $\embed_{\encoderparameter}\fun{\sampledot \mid \state}$ is uniformly distributed over $\latentstates$ for any state $\state \in \states$: $H\fun{\embed_{\encoderparameter}\fun{\sampledot \mid \state}} = \sum_{i=0}^{\log_2 \left| \latentstates \right| - \left| \atomicprops \right| = 6} - p_i\log~p_i - \fun{1 - p_i} \log\fun{1 - p_i} = 4.188 $, where $p_i = \nicefrac{1}{2}$ for all $i$.
The rate (Fig.~\ref{subfig:rate}) drops to zero since the divergence pulls the latent dynamics towards this high entropy (yet another form of posterior collapse), which hinders the latent space model to learn a useful representation.
However, the annealing scheme increases the rate importance along training steps, which enables the optimization to eventually leave this local optimum (here around $4\cdot10^5$ training steps).
This allows the learning procedure to leave the zero-rate spot, reduce the distortion (Fig.~\ref{subfig:distortion}), and finally distill the original policy (Fig.~\ref{subfig:vae-distillation}).

As a result, the whole engineering required to mitigate posterior collapse slows down the training procedure.
This phenomenon is reflected in Fig.~\ref{fig:eval-plots}: VAE-MDPs need several steps to stabilize and set up the stage to the concrete optimization, whereas WAE-MDPs have no such requirements since they naturally do not suffer from collapsing issues (cf. Fig.~\ref{fig:histograms}), and are consequently faster to train.

\smallparagraph{Lack of representation guarantees.}~On the theoretical side, since VAE-MDPs are optimized via the ELBO and the local losses via the related variational proxies, VAE-MDPs \emph{do not leverage the representation quality guarantees} induced by local losses (Eq.~\ref{eq:bidistance-bound}) during the learning procedure (as explicitly pointed out by \citealp[Sect.~4.1]{DBLP:journals/corr/abs-2112-09655}.): in contrast to WAE-MDPs, when two original states are embedded to the same latent, abstract state, the former are not guaranteed to be bisimilarly close (i.e., the agent is not guaranteed to behave the same way from those two states by executing the policy), meaning those proxies do not prevent original states having distant values collapsing together to the same latent representation.

    \nomenclature[m0]{$\mdp = \mdptuple$}{MDP $\mdp$ with state space $\states$, action space $\actions$, transition function $\probtransitions$, labeling function $\labels$, atomic proposition space $\atomicprops$, and initial state $\sinit$.}%
    \nomenclature[msinit]{$\mdp_\state$}{MDP obtained by replacing the initial state of $\mdp$ by $\state \in \states$}
    \nomenclature[mtrajectory]{$\trajectory = \trajectorytuple{\state}{\action}{T}$}{Trajectory}
    \nomenclature[mtrajectories]{$\inftrajectories{\mdp}$}{Set of infinite trajectories of $\mdp$}
    \nomenclature[mpolicy]{$\policy$}{Memoryless policy $\policy \colon \states \to \distributions{\actions} $}
    \nomenclature[pdist]{$\distributions{\measurableset}$}{Set of measures over a complete, separable metric space $\measurableset$}
    \nomenclature{$\condition{[\textit{cond}]}$}{indicator function: $1$ if the statement [\textit{cond}] is true, and $0$ otherwise}
    \nomenclature{$f_{\decoderparameter}$}{A function $f_{\decoderparameter} \colon \measurableset \to \R$ modeled by a neural network, parameterized by $\decoderparameter$, where $\measurableset$ is any measurable set}
    \nomenclature[mprobmeasure]{$\Prob_{\policy}^{\mdp}$}{Unique probability measure induced by the policy $\policy$ in $\mdp$ on the Borel $\sigma$-algebra over measurable subsets of $\inftrajectories{\mdp}$}
    \nomenclature[mpolicies]{$\mpolicies{\mdp}$}{Set of memoryless policies of $\mdp$}
    \nomenclature{}{}
    \nomenclature[mstationary]{$\stationary{\policy}$}{Stationary distribution of $\mdp$ induced by the policy $\policy$}
    \nomenclature[mdiscount]{$\discount$}{Discount factor in $\mathopen[0, 1\mathclose]$}
    \nomenclature[mvalues]{}{}
    \nomenclature[mstate]{$\state$}{State in $\states$}
    \nomenclature[maction]{$\action$}{Action in $\actions$}
    \nomenclature[mlimitingdistr]{$\stationary{\policy}^{t}$}{Limiting distribution of the MDP defined as  $\stationary{\policy}^{t}\fun{\state' \mid \state} = \Prob_{\policy}^{\mdp_{\state}}\fun{\set{\seq{\state}{\infty}, \seq{\action}{\infty} \mid \state_t = \state'}}$, for any source state $\state \in \states$}
    \nomenclature[mvalues]{$\valuessymbol{\policy}{\sampledot}$}{Value function for the policy $\policy$}
    \nomenclature[mreachability]{$\until{\labelset{C}}{\labelset{T}}$}{Constrained reachability event}
    \nomenclature[l1]{$\latentmdp = \latentmdptuple$}{Latent MDP with state space $\latentstates$, action space $\latentactions$, reward function $\latentrewards$, labeling function $\latentlabels$, atomic proposition space $\atomicprops$, and initial state $\zinit$.}
    \nomenclature[lstate]{$\latentstate$}{Latent state in $\latentstates$}
    \nomenclature[laction]{$\latentaction$}{Latent action in $\latentactions$}
    \nomenclature[lembedding]{$\embed$}{State embedding function, from $\states$ to $\latentstates$}
    \nomenclature[lembeddingaction]{$\embeda$}{Action embedding function, from $\latentstates \times \latentactions$ to $\actions$}
    \nomenclature[l2latentspacemodel]{$\tuple{\latentmdp, \embed, \embeda}$}{Latent space model of $\mdp$}
    \nomenclature[ldistance]{$\distance_{\latentstates}$}{Distance metric over $\latentstates$}
    \nomenclature[lantentpolicy]{$\latentpolicy$}{Latent policy $\latentpolicy\colon \latentstates \to \actions$; can be executed in $\mdp$ via $\embed$: $\latentpolicy\fun{\sampledot \mid \embed\fun{\state}}$}
    \nomenclature[localtransitionloss]{$\localtransitionloss{\stationary{}}$}{Local transition loss under distribution $\stationary{}$}
    \nomenclature[localrewardloss]{$\localrewardloss{\stationary{}}$}{Local reward loss under distribution $\stationary{}$}
    \nomenclature[lembedprob]{$\embed\probtransitions$}{Distribution of drawing $\state' \sim \probtransitions\fun{\sampledot \mid \state, \action}$, then embedding $\latentstate' = \embed\fun{\state'}$, for any state $\state \in \states$ and action $\action \in \actions$}
    \nomenclature[mbidistance]{$\bidistance_{\policy}$}{Bisimulation pseudometric}
    \nomenclature[pdiscrepancy]{$D$}{Discrepancy measure; $D\fun{P, Q}$ is the discrepancy between distributions $P, Q \in \distributions{\measurableset}$}
    \nomenclature[pwasserstein]{$\wassersteinsymbol{\distance}$}{Wasserstein distance w.r.t. the metric $\distance$; $\wassersteindist{\distance}{P}{Q}$ is the Wasserstein distance between distributions $P, Q \in \distributions{\measurableset}$}
    \nomenclature{$\Lipschf{\distance}$}{Set of $1$-Lipschitz functions w.r.t. the distance metric $\distance$}
    \nomenclature[wbehavioral]{$\stationary{\decoderparameter}$}{Behavioral model: distribution over $\states \times \actions \times \images{\rewards} \times \states$}
    \nomenclature[mtracedistance]{$\tracedistance$}{Raw transition distance, i.e., metric over $\states \times \actions \times \images{\rewards} \times \states$}
    \nomenclature[wstationary]{$\latentstationaryprior$}{Stationary distribution of the latent model $\latentmdp_{\decoderparameter}$, parameterized by $\decoderparameter$}
    \nomenclature[wmarginalencoder]{$Q_{\encoderparameter}$}{Marginal encoding distribution over $\latentstates \times \latentactions \times \latentstates: \expectedsymbol{\state, \action, \state' \sim \stationary{\policy}} \embed_{\encoderparameter}\fun{\sampledot \mid \state, \action, \state'}$}
    \nomenclature[wembeddinga]{$\embed^{\actions}_{\encoderparameter}$}{Action encoder mapping $\latentstates \times \actions$ to $\distributions{\latentactions}$}
    \nomenclature[wtransition]{$\originaltolatentstationary{}$}{Distribution of drawing state-action pairs from interacting with $\mdp$, embedding them to the latent spaces, and finally letting them transition to their successor state in $\latentmdp_{\decoderparameter}$, in $\distributions{\latentstates \times \latentactions \times \latentstates}$}
    \nomenclature[mdistancestate]{$\distance_{\states}$}{Metric over the state space}
    \nomenclature[mdistanceaction]{$\distance_{\actions}$}{Metric over the action space}
    \nomenclature[mdistancerewards]{$\distance_{\rewards}$}{Metric over $\images{\rewards}$}
    \nomenclature[wgenerate]{$\generative_{\decoderparameter}$}{State-wise decoder, from $\latentstates$ to $\states$}
    \nomenclature[wdirac]{$G_{\decoderparameter}$}{Mapping $\tuple{\latentstate, \latentaction, \latentstate'} \mapsto \tuple{\generative_{\decoderparameter}\fun{\latentstate}, \embeda_{\decoderparameter}\fun{\latentstate, \latentaction}, \latentrewards_{\decoderparameter}\fun{\latentstate, \latentaction}, \generative_{\decoderparameter}\fun{\latentstate'}}$}
    \nomenclature[wsteadystate]{\steadystateregularizer{\latentpolicy}}{Steady-state regularizer}
    \nomenclature[w]{}{}
    \nomenclature[wtransitionlossnet]{\transitionlossnetwork}{Transition Lipschitz network}
    \nomenclature[wsteadystatenet]{\steadystatenetwork}{Steady-state Lipschitz network}
    \nomenclature{$\sigmoid$}{Sigmoid function, with $\sigmoid\fun{x} = \nicefrac{1}{1 + \exp\fun{-x}}$}
    \nomenclature[lvalue]{$\latentvaluessymbol{\latentpolicy}{\sampledot}$}{Latent value function}
    \nomenclature[lpolicies]{$\latentpolicies$}{Set of (memoryless) latent policies}
    \nomenclature[wtemperarue]{$\temperature$}{Temperature parameter}
    \nomenclature[pLogistic]{$\logistic{\mu}{s}$}{Logistic distribution with location parameter $\mu$ and scale parameter $s$}
    %
\printnomenclature
%
%
\bibliographyAR{references}
\bibliographystyleAR{iclr2023_conference}

\end{document}